\documentclass{article}

\PassOptionsToPackage{numbers, compress}{natbib}
\usepackage[final]{neurips_2021}

\usepackage[utf8]{inputenc} 
\usepackage[T1]{fontenc}    
\usepackage{hyperref}       
\usepackage{url}            
\usepackage{booktabs}       
\usepackage{amsfonts}       
\usepackage{nicefrac}       
\usepackage{microtype}      
\usepackage{xcolor}         

\usepackage{makecell}
\usepackage{enumitem}

\usepackage{bbm}

\usepackage{caption, subcaption}

\usepackage[ruled,vlined, linesnumbered]{algorithm2e}

\usepackage{tikz}
\usetikzlibrary{automata,decorations.markings,positioning,arrows}
\usepackage{wrapfig, amsmath, amsthm, mathtools}
\usepackage{amssymb}

\usetikzlibrary{shapes,decorations,arrows,calc,arrows.meta,fit,positioning}
\tikzset{
    -Latex,auto,node distance =1 cm and 1 cm,semithick,
    state/.style ={ellipse, draw, minimum width = 0.7 cm},
    point/.style = {circle, draw, inner sep=0.04cm,fill,node contents={}},
    bidirected/.style={Latex-Latex,dashed},
    el/.style = {inner sep=2pt, align=left, sloped}
}

\def\E{{\bf E}}

\def\R{{\bf R}}

\def\P{{\bf P}}

\def\0{{\bf 0}}
\def\1{{\bf 1}}

\def\OM{{\mathcal O}}

\newcommand{\Var}{\text{Var}}
\newcommand{\Cov}{\text{Cov}}

\newcommand{\Amse}{\text{AMSE}}

\newcommand{\indep}{\perp \!\!\! \perp}
\newcommand{\notindep}{\not\!\perp\!\!\!\perp}

\newcommand{\CardC}{\left|\psi\right|}
\newcommand{\ChoiceSimplex}{\Delta_{\psi}}
\newcommand{\ctrSim}{\text{ctr}\left(\ChoiceSimplex\right)}

\newcommand{\ConvProb}{\overset{p}{\rightarrow}}

\newtheorem{remark}{Remark}
\newtheorem*{remark*}{Remark}
\newtheorem{theorem}{Theorem}
\newtheorem*{theorem*}{Theorem}
\newtheorem{lemma}{Lemma}
\newtheorem*{lemma*}{Lemma}
\newtheorem{definition}{Definition}

\newtheorem{proposition}{Proposition}
\newtheorem*{proposition*}{Proposition}

\newtheorem{assumption}{Assumption}

\newtheorem{property}{Property}
\newtheorem{example}{Example}

\DeclarePairedDelimiter\floor{\lfloor}{\rfloor}

\SetKwInput{KwInput}{Input}
\SetKwInput{KwOutput}{Output}

\title{Efficient Online Estimation of Causal Effects \\
by Deciding What to Observe}

\PassOptionsToPackage{numbers, sort&compress}{natbib}

\author{Shantanu Gupta, Zachary C. Lipton, David Childers \\ \\
Carnegie Mellon University\\
\small{\texttt{\{\href{mailto:shantang@cmu.edu}{shantang},\href{mailto:zlipton@cmu.edu}{zlipton},\href{mailto:dchilders@cmu.edu}{dchilders}\}@cmu.edu}}
}

\begin{document}

\maketitle

\begin{abstract}
Researchers often face data fusion problems, 
where multiple data sources are available,
each capturing a distinct subset of variables.
While problem formulations typically
take the data as given, in practice,
data acquisition can be an ongoing process. 
In this paper, we aim to estimate 
any functional of a probabilistic model
(e.g., a causal effect)
as efficiently as possible,
by deciding, at each time,
which data source to query.
We propose \emph{online moment selection} (OMS),
a framework in which structural assumptions 
are encoded as moment conditions.
The optimal action at each step
depends, in part, on the very moments
that identify the functional of interest.
Our algorithms balance exploration 
with choosing the best action
as suggested by current estimates of the moments.
We propose two selection strategies:
(1) explore-then-commit (OMS-ETC) 
and (2) explore-then-greedy (OMS-ETG),
proving that both achieve 
zero asymptotic regret
as assessed by MSE. 
We instantiate our setup for 
average treatment effect estimation,
where structural assumptions
are given by a causal graph 
and data sources may include subsets
of mediators, confounders, 
and instrumental variables.

\end{abstract}

\section{Introduction}
\label{sec:introduction}
Statistical and causal modeling 
typically proceed from the assumption
that we already know which
variables are (and are not) observed.
However, this perspective fails 
to address the difficult 
data collection decisions
that precede such modeling efforts.
Doctors must select a set of tests to run.
Survey designers must select
a slate of questions to ask.
Companies must select which 
datasets to purchase.
Whether or not we model these decisions,
they pervade the practice of data science,
influencing both what questions we can ask
and how accurately we can answer them.

One might ask,
\emph{why not collect everything?} 
The answers are two-fold:
First, data acquisition can be expensive. 
In a medical setting, 
blood tests can cost anywhere
from tens to thousands of dollars. 
Running every test for every patient is infeasible. 
Likewise, space on surveys is limited, 
and asking every conceivable question
of every respondent is infeasible.
Second, in many settings,
we lack complete control 
over the set of variables observed. 
Instead, we might have access 
to multiple data sources, 
each capturing a different subset of variables.
Such \emph{data fusion problems}
pervade economic modeling and public health,
and present interesting challenges:
(i) efficiently estimating
(or even identifying)
a population parameter of interest
often requires intelligently combining 
data from multiple sources;
(ii) data collection is often iterative,
with tentative conclusions 
at each stage
informing choices 
about what data to collect next.

In this paper, we formalize the sequential problem 
of deciding, at each time,
which data source to query
(i.e., \emph{what to observe})
in order to efficiently estimate
a target parameter.
We propose online moment selection (OMS), 
a framework that applies 
the generalized method of moments (GMM)~\citep{hansen1982large}
both to estimate the parameter
and to decide which data sources to query.
This framework can be applied 
to estimate any statistical parameter 
that can be identified by a set of moment conditions.
For example, OMS can address
(i) any (regular) maximum likelihood estimation problem \citep[Page~109]{hall2005generalized};
and (ii) estimating average treatment effects (ATEs)
using instrumental variables (IVs),
backdoor adjustment sets, mediators,
and/or other identification strategies.

Our strategy requires only that the agent 
has sufficient structural knowledge 
to formulate the set of moment conditions
and that each moment can be estimated
using the variables returned
by at least one of the data sources.
Interestingly, the optimal decisions
which lead to estimates with the lowest mean squared error (MSE)
depend on the (unknown) model parameters.
This motivates our adaptive strategy: 
as we collect more data, 
we better estimate the underlying parameters,
improving our strategy
for allocating our remaining budget
among the available data sources.

We first address the setting 
where the cost per instance 
is equal across data sources 
(Section~\ref{sec:adaptive-data-collection}).
First, we show that any fixed policy
that differs from the oracle 
suffers constant asymptotic regret,
as assessed by MSE.
We then overcome this limitation 
by proposing two adaptive strategies---explore-then-commit (OMS-ETC) 
and explore-then-greedy (OMS-ETG)---both 
of which choose data sources 
based on the \emph{estimated} asymptotic variance 
of the target parameter.

Under OMS-ETC,
we use some fraction of the sample budget
to explore randomly,
using the collected data 
to estimate the model parameters.
We then exploit the current estimated model,
collecting the remaining samples 
according to the fraction expected to minimize 
our estimator's asymptotic variance.
In OMS-ETG,
we continue to update our parameter estimates
after every step as we collect new data.
We prove that both policies 
achieve zero asymptotic regret.
To overcome the non-i.i.d.~nature of the sample moments, 
we draw upon martingale theory.
To derive zero asymptotic regret,
we show uniform concentration of sample moments
and a finite-sample concentration inequality 
for the GMM estimator with dependent data.
Next, we adapt OMS-ETC and OMS-ETG to handle
heterogeneous costs over the data sources 
(Section~\ref{sec:cost-structure})
and prove that they still have zero asymptotic regret.

Finally, we validate our findings experimentally \footnote{The code 
and data are available at
\url{https://www.github.com/acmi-lab/online-moment-selection}.}
(Section~\ref{sec:experiments}).
Motivated by ATE estimation
in causal models encoded as
directed acyclic graphs,
we generate synthetic data 
from a variety of causal graphs 
and show that the regret 
of our proposed methods 
converges to zero.
Furthermore, we see that despite 
being asymptotically equivalent,
OMS-ETG outperforms OMS-ETC in finite samples.
Finally, we demonstrate 
the effectiveness of our methods 
on two semi-synthetic datasets: 
the Infant Health Development Program (IHDP) dataset~\citep{hill2011bayesian} 
and a Vietnam era draft lottery dataset~\citep{angrist1990lifetime}.

\section{Related Work}
\label{sec:related-work}
Many works attempt to
identify and estimate causal effects
from multiple datasets. 
\cite{bareinboim2016causal, hunermund2019causal} study the problem 
of combining multiple heterogeneous datasets 
and propose methods for dealing with various biases.
Other works study causal identification 
when observational and interventional distributions involving
different sets of variables are available \cite{lee2020causal, tikka2019causal}.
\cite{gupta2020estimating} 
introduce estimators of the ATE
that efficiently combine two datasets,
one where confounders are observed 
(enabling the backdoor adjustment)
and another where mediators are observed
(enabling the frontdoor adjustment).
\cite{li2020causal} derive an estimator for the ATE 
in linear causal models with multiple confounders,
where the confounders are observed in different datasets.

Another related line of work addresses finding optimal adjustment
sets for covariate adjustment 
\cite{henckel2019graphical, rotnitzky2019efficient, witte2020efficient, smucler2020efficient}.
While these works take for granted 
the collection of available datasets,
we focus on the problem of 
deciding which data to collect.
Our work shares motivation with active learning,
where the learner strategically chooses 
which (unlabeled) samples to label
in order to learn most efficiently
\cite{settles2009active, kumar2020active}.
\cite{cohn1996active} design algorithms 
for actively collecting samples 
in a manner that minimizes 
the learner's variance.
In settings where there is a cost associated 
with collecting each feature,
active feature acquisition methods 
incrementally query feature values 
to improve a predictive model 
\cite{ji2007cost, saar2009active, huang2018active}.
\cite{zhang2020active} propose an active learning criterion 
to find the most informative questions 
to ask each respondent in a survey.
In the context of causal inference,
\cite{squires2020} study active structure learning 
of causal DAGs by finding cost-optimal interventions,
\cite{gan2020causal} demonstrate
that for ATE estimation,
actively deconfounding data
can improve sample complexity.
\cite{wang2020confounding} propose strategies
for acquiring missing confounders
to efficiently estimate the ATE.

Others have studied moment selection 
and IV selection from batch data.
\cite{andrews1999consistent} introduce
consistent moment selection procedures 
for the GMM setting with some incorrect moments.
\cite{cheng2015select} propose 
an information-based lasso method 
for excluding invalid or 
redundant moment conditions.
\cite{urminsky2016using} propose 
a variable selection framework 
that uses lasso regression to decide 
which covariates to include.
\cite{hall2003consistent} propose 
four statistical criteria---including 
estimation efficiency and non-redundancy---for 
selecting among a set of candidate IVs.
\cite{donald2009choosing} develop 
an IV selection criteria 
based on asymptotic MSE 
and develop it for the GMM
and generalized empirical likelihood estimators.
\cite{cocci2019standard} develop conservative confidence intervals for structural parameters 
when the off-diagonal entries of the covariance matrix of the empirical moments 
are unknown (e.g., when the moments are obtained from different datasets).
By contrast, we are interested 
in selecting the data sources 
for the moments in an online setting.

Previous works address learning 
from adaptively collected data.
\cite{kato2020efficient}
propose methods
for adaptive experimental design, 
where at each step, 
the experimenter must decide 
the treatment probability using past data 
in order to efficiently estimate the ATE.
\cite{kato2021adaptive, kato2021adaptive2} 
propose a doubly-robust estimator 
for off-policy evaluation with dependent samples.
\cite{zhan2021policy} provide regret bounds
for learning an optimal policy
using adaptively collected data,
where the probability of selecting an action
is a function of past data.
\cite{zhang2021inference, zhang2021statistical} 
study statistical inference 
for OLS and M-estimation
with non-i.i.d. bandit data.
While these settings are different from ours,
some of the theoretical tools 
(e.g., martingale asymptotics and 
uniform martingale concentration bounds) 
are similar.

\section{Preliminaries}
\label{sec:preliminaries}
In the GMM framework \citep{hansen1982large, newey1994large}, 
we estimate model parameters 
by leveraging moment conditions 
that are satisfied at the true parameters $\theta^*$. 
A moment condition is a vector $g(X_t, \theta)$ 
such that $\E[g(X_t, \theta^*)] = 0$.
We estimate $\theta^*$ 
by minimizing the objective $\widehat{Q}_T$:
\begin{align*}
    \widehat{\theta}_T = \arg\min_{\theta \in \Theta} \widehat{Q}_T(\theta), \,\,\, \text{where} \,\,\, \widehat{Q}_T(\theta) = \left[ \frac{1}{T} \sum_{t=1}^T g_t(\theta) \right]^\top \widehat{W} \left[ \frac{1}{T} \sum_{t=1}^T g_t(\theta) \right],
\end{align*}
$\Theta$ is the parameter space,
$g_t(\theta) := g(X_t, \theta)$,
and $\widehat{W}$ is some
(possibly data dependent)
positive definite matrix. 
In this work, we use the \textit{two-step GMM estimator},
where the \textit{one-step estimator} $\widehat{\theta}^{(\text{os})}_T$
is computed with $\widehat{W} := I$ (identity) 
and the two-step estimator
with $\widehat{W} := \left[\widehat{\Omega}_T(\widehat{\theta}^{(\text{os})}_T)\right]^{-1}$, where $ \widehat{\Omega}_T(\theta) = \left[ \frac{1}{T} \sum_{t=1}^T g_t(\theta) g_t(\theta)^\top \right]$.

Let $\textbf{V}$ be the set of variables of interest
and $\psi$ a collection of subsets of $\textbf{V}$,
each corresponding to the specific variables observable 
via one of the available data sources.
Our methods are applicable 
whenever the target parameter can be identified 
by a set of moment conditions
such that each moment relies on variables
simultaneously observable in at least one data source.
The \textit{selection vector}, 
denoted by $s_t \in \left\{ 0, 1 \right\}^{\CardC}$, 
is the binary vector indicating 
the data source selected at time $t$.
\begin{assumption}
The agent queries one data source at each step: 
$\sum_{i=1}^{\CardC} s_{t, i} = 1$, i.e., $s_t$ is one-hot.
\end{assumption}
We can handle the querying of multiple sources 
by adding the union of their variables to $\psi$.
In our setup,
the moment conditions can be written as
$g_t(\theta) = m(s_t) \otimes \Tilde{g}_t(\theta) \in \R^M$,
where $\otimes$ is the Hadamard product, 
$m : \left\{ 0, 1 \right\}^{\CardC} \rightarrow \left\{ 0, 1 \right\}^M$ 
is a fixed known function 
such that
$m(s_t)$ determines which moments get selected, 
and $\Tilde{g}_t(\theta)$ are i.i.d. across $t$.
For concreteness, we instantiate 
our setup with a simple example:
\begin{example}[Instrumental Variable (IV) graph]\label{example:iv-graph}
Consider a linear IV causal model 
(Figure~\ref{fig:disjoint-iv-graph}) 
with instrument $Z$, treatment $X$, outcome $Y$,
and the following data-generating process:
\begin{align*}
    & X := \alpha Z + \eta, \,\,\, Y := \beta X + \epsilon, \,\, \epsilon \notindep \eta,\, \epsilon \indep Z,\, \eta \indep Z.
\end{align*}
The target parameter is the ATE $\beta$.
For $\psi = \left\{ \{ Z, X \}, \{ Z, Y \} \right\}$,
the moment conditions are
\begin{align*}
    g_t(\theta) = \underbrace{\begin{bmatrix} 
        s_{t, 1} \\
        s_{t, 2}
    \end{bmatrix}}_{=m(s_t)} \otimes \underbrace{\begin{bmatrix}
        Z_t (X_t - \alpha Z_t) \\
        Z_t (Y_t - \alpha \beta Z_t)
    \end{bmatrix}}_{=\Tilde{g}_t(\theta)} = \begin{bmatrix}
        s_{t, 1} Z_t (X_t - \alpha Z_t) \\
        (1 - s_{t, 1}) Z_t (Y_t - \alpha \beta Z_t)
    \end{bmatrix},
\end{align*}
where $\theta = [\beta, \alpha]^\top$ and $\{ Z_t, X_t, Y_t \}$ are i.i.d.
\end{example}
For some known function $f_{\text{tar}} : \Theta \rightarrow \R$, let $\beta^* := f_{\text{tar}}(\theta^*)$ be the target parameter (e.g., the ATE).
In practice, we estimate the target parameter by plugging-in the GMM estimate: $\widehat{\beta} = f_{\text{tar}}(\widehat{\theta})$.
Let $H_{t}$ represent the \textit{history} or the data collected until time $t$ with $H_0 = \{\}$ and space $\mathcal{H}_t$.
A data collection policy $\pi$ consists of a sequence of functions $\pi_t : \mathcal{H}_{t-1} \rightarrow \left\{ 0, 1 \right\}^{\CardC}$ with $s_t = \pi_t(H_{t-1})$. Thus $s_t$ can be dependent on data collected until time $(t - 1)$ and
so the sample moments $g_t(\theta)$ are \textit{not} i.i.d.
\begin{definition}[Selection ratio]
The selection ratio,
denoted by $\kappa^{(\pi)}_t$,
encodes the fraction of samples collected 
from each data source until time $t$: 
$\kappa^{(\pi)}_t = \frac{1}{t} \sum_{i=1}^t s_t \in \Delta^{\CardC - 1}$ (standard simplex).
\end{definition}
We use $\widehat{\theta}^{(\pi)}_t$ and $\widehat{\theta}^{(\text{os})}_t$ to denote 
the two-step and one-step GMM estimators, respectively,
that use the data $H_t$.
To reduce clutter, we use
$\ChoiceSimplex := \Delta^{\CardC - 1}$,
$\ctrSim = \left[ \frac{1}{\CardC}, \frac{1}{\CardC}, \hdots, \frac{1}{\CardC} \right]$ (center of the simplex),
and might drop 
the superscript
$\pi$ from 
$\kappa_t, \widehat{\theta}_t$, 
and $\widehat{\beta}_t$.
$\|.\|$ denotes the spectral and $l_2$ norms 
for matrices and vectors, respectively,
and $N_{\epsilon}(\theta) := \left\{\theta' : \left\| \theta' - \theta \right\| \leq \epsilon \right\}$ ($\epsilon$-ball around $\theta$).

\section{Adaptive Data Collection}
\label{sec:adaptive-data-collection}
The central challenge in this work
is to make strategic decisions 
about which data to observe 
so that we most efficiently 
estimate the target functional.
In this section, we present three policies:
(i) \emph{fixed}: query 
the data sources according 
to a pre-specified ratio;
(ii) OMS-ETC: query uniformly
for a specified exploration period,
estimate the optimal ratio
based on the inferred parameters
and thereafter aim for that ratio;
and (iii) OMS-ETG: same as OMS-ETC
but continue to update parameter 
(and thus oracle ratio) estimates
after the exploration period. 
For now, we analyze these policies
for the case when the cost to query
is identical across data sources.
By $T \in \mathbb{N}$,
we denote the (known) \textit{horizon},
which can be thought of as the agent's 
data acquisition budget. 
We defer all proofs to Appendix~\ref{sec:appendix-omitted-proofs}. 

We now present
sufficient conditions 
for consistency and asymptotic normality 
of the GMM estimator 
under adaptively collected data.
We later use these results to
derive the regret of our policies.

\begin{assumption}\label{assump:standard-gmm}
(a)
(Identification)
$\forall \theta \neq \theta^*, \P\left( \liminf_{T \to \infty} \Bar{Q}(\theta) > 0 \right) = 1$,
where $\Bar{Q}(\theta) = g^{*}_T(\theta) \widehat{W} g^{*}_T(\theta)^\top$ and $g^{*}_T(\theta) = \left[ \frac{1}{T} \sum_{t=1}^T m(s_t) \otimes \E\left[ \Tilde{g}_t(\theta) \right]  \right]$;
(b) $\Theta \subset \R^D$ is compact;
(c) $\forall \theta,
\Tilde{g}_i(\theta)$ is twice continuously differentiable (c.d.);
(d) $\forall \theta, \E[\Tilde{g}_i(\theta)]$ is continuous; and
(e) $f_{\text{tar}}$ is c.d. at $\theta^*$.
\end{assumption}

By Assumption~\ref{assump:standard-gmm}(a),
the
GMM objective is uniquely minimized at $\theta^*$.
Informally, this means that each moment is (asymptotically) collected enough times to allow identification.
If $M = D$ (just-identified case),
this holds when
(i) an asymptotically non-negligible fraction of every moment is collected: $\forall j \in [M], \,\, \P\left( \liminf_{T \to \infty} \frac{1}{T} \sum_{t=1}^T m_j(s_t) \neq 0  \right) = 1$; and (ii) $\forall \theta \neq \theta^*, \, \E\left[ \Tilde{g}_t(\theta) \right] \neq 0$.

\begin{property}[ULLN]\label{property:mod-of-continuity}
Let $a_i(\theta) := a(X_i; \theta) \in \R$ be a continuous function
with $X_i$ sampled i.i.d. We say that $a_i(\theta)$ satisfies the ULLN property if 
(i) $\forall \theta, \, \E\left[ a_i(\theta)^2 \right] < \infty$;
(ii) $a_i(\theta)$ is dominated by a function $A(X_i)$: $\forall \theta, |a_i(\theta)| \leq A(X_i)$; and 
(iii) $\E[A(X_i)] < \infty$.
\end{property}

\begin{proposition}[Consistency]\label{prop:consistency}
Suppose that (i) Assumption~\ref{assump:standard-gmm} holds, (ii) $\forall j \in [M], \, \Tilde{g}_{t, j}(\theta)$ satisfies Property~\ref{property:mod-of-continuity}, and
(iii) $\forall (i, j) \in [M]^2, \, \left[ \Tilde{g}_t(\theta) \Tilde{g}_t(\theta)^\top \right]_{i, j}$ satisfies Property~\ref{property:mod-of-continuity};
Then, $\widehat{\theta}^{(\pi)}_T \xrightarrow[T \to \infty]{p} \theta^*$.
\end{proposition}

\begin{proposition}[Asymptotic normality]\label{prop:asymptotic-normality}
Suppose that
(i) $\widehat{\theta}^{(\pi)}_T \xrightarrow{p} \theta^*$;
(ii) $\forall (i, j) \in [M] \times [D], \, \left[ \frac{\partial \Tilde{g}_t}{\partial \theta} (\theta) \right]_{i, j}$ satisfies Property~\ref{property:mod-of-continuity};
(iii) $\exists \delta > 0 \, \text{such that} \, \E\left[ \left\| \Tilde{g}_i(\theta^*) \right\|^{2 + \delta} \right] < \infty$, and
(iv) (Selection ratio convergence) $\kappa^{(\pi)}_T \overset{p}{\rightarrow} k$ for some constant $k \in \ChoiceSimplex$. Then
$\widehat{\theta}_T$ is asymptotically normal:
\begin{align*}
    \sqrt{T} (\widehat{\theta}^{(\pi)}_T - \theta^*) &\overset{d}{\rightarrow} \mathcal{N}\left( 0, \Sigma(\theta^*, k) \right),
\end{align*}
where $\Sigma(\theta^*, k)$ is a constant matrix that depends only on $\theta^*$ and $k$  (see Appendix~\ref{sec:appendix-asymptotic-normality-proof} for the complete expression).
By Assumption~\ref{assump:standard-gmm}(e) and the Delta method, $\widehat{\beta}_T$ is asymptotically normal:
\begin{align*}
    \sqrt{T} (\widehat{\beta}_T - \beta^*) &\overset{d}{\rightarrow} \mathcal{N}\left( 0, V(\theta^*, k) \right), \, \text{where} \, V(\theta^*, k) = \nabla_{\theta} f_{\text{tar}}(\theta^*)^\top [\Sigma(\theta^*, k)] \nabla_{\theta} f_{\text{tar}}(\theta^*).
\end{align*}
\end{proposition}
Proposition~\ref{prop:asymptotic-normality} shows that for a policy under which the selection ratio $\kappa_T$ converges in probability to a constant, the GMM estimator $\widehat{\theta}_T$ can be asymptotically normal.
The specific order in which the data sources are queried
does not affect asymptotic normality as long the selection ratio $\kappa^{(\pi)}_T$ converges.

A \textit{fixed (or non-adaptive) policy}, denoted by $\pi_{k}$, has $\kappa_T = k$ for some constant $k \in \ChoiceSimplex$.
Here, the collection decisions 
do not depend on the data 
and each data source is queried 
a fixed fraction of the time.
By Proposition~\ref{prop:asymptotic-normality}, 
$\widehat{\theta}^{(\pi_k)}_T$ 
is asymptotically normal. 
The \textit{oracle policy}, denoted by $\pi^*$, 
is the fixed policy 
with the lowest asymptotic variance.
Thus for $\pi^*$, we have $\kappa^{(\pi^*)}_T = \kappa^*$, 
where $\kappa^* = \arg\min_{\kappa} V(\theta^*, \kappa)$.
We call $\kappa^*$ the \textit{oracle selection ratio}.
For the oracle policy,
we have $\sqrt{T} (\widehat{\beta}^{(\pi^*)}_T - \beta) \overset{d}{\rightarrow} \mathcal{N}\left( 0, V(\theta^*, \kappa^*) \right)$.
The following assumption
ensures that
$\kappa^*$ is unique and consequently the data collection decisions are meaningful.

\begin{assumption}\label{assump:kappa-star-identify}
$\kappa^*$ uniquely minimizes $V(\theta^*, \kappa)$: $\forall \kappa \in \ChoiceSimplex \, \text{s.t.} \, \kappa \neq \kappa^*, \,\, V(\theta^*, \kappa) > V(\theta^*, \kappa^*)$.
\end{assumption}

\begin{definition}[Asymptotic regret]
The asymptotic regret captures how close 
the scaled asymptotic error of
a given policy is to the oracle policy.
We define the asymptotic regret of a policy $\pi$ as
\begin{align}\label{eq:asymptotic-regret}
    R_\infty(\pi) = \Amse\left( \sqrt{T} \left(\widehat{\beta}^{(\pi)}_T - \beta^* \right) \right) - V(\theta^*, \kappa^*),
\end{align}
where $\Amse$ is the asymptotic MSE 
(i.e., the MSE of the limiting distribution).
\end{definition}
For any fixed policy $\pi_k$ such that 
$\kappa^{(\pi_k)}_T = k$ for some constant $k \neq \kappa^*$, 
we have $R_\infty(\pi_k) = \left[ V(\theta^*, k) - V(\theta^*, \kappa^*) \right] > 0$ (by Assumption~\ref{assump:kappa-star-identify}).
This shows that a fixed policy suffers constant regret.
This motivates the design of adaptive policies,
where the regret asymptotically converges to zero.

\subsection{Online Moment Selection via Explore-then-Commit (OMS-ETC)}

OMS-ETC is inspired by the ETC strategy for
multi-armed bandits (MABs) \citep[Chapter~6]{lattimore2020bandit}.
Under OMS-ETC, we first explore by collecting 
a fixed number of samples for each choice.
Then, we use these samples 
to estimate the oracle selection ratio $\kappa^*$.
Finally, we commit to this ratio 
for the remaining time steps. 
We denote the OMS-ETC policy by $\pi_{\text{ETC}}$ (Figure~\ref{fig:policy-algorithm-etc}).

The policy $\pi_{\text{ETC}}$ is characterized 
by an exploration fraction $e \in (0, 1)$.
We first collect $Te$ samples 
by querying each data source equally
so that $\kappa_{Te} = \ctrSim$.
We then estimate $\widehat{\theta}_{Te}$
and obtain the plugin estimate of $\kappa^*$
as $\widehat{k} = \arg\min_{\kappa \in \ChoiceSimplex} V(\widehat{\theta}_{Te}, \kappa)$.
The \emph{feasible region} for $\kappa_T$
is defined as the set of values that $\kappa_T$ can take 
after we have devoted $Te$ samples to exploration
and is given by
$\Tilde{\Delta} = \left\{ e \kappa_{Te} + (1 - e) \kappa : \kappa \in \ChoiceSimplex \right\}$
(proof in Appendix~\ref{sec:appendix-feasibility-region}).
We collect the remaining $T(1 - e)$ samples 
such that $\kappa_T$ is as close to $\widehat{k}$ as possible:
$\kappa_T = \text{proj}(\widehat{k}, \Tilde{\Delta})$, 
where $\text{proj}\left( k, \Tilde{\Delta} \right) = \arg\min_{k' \in \Tilde{\Delta}} \| k - k' \|$.

\begin{remark*}
The feasible region shrinks as $e$ increases because, as $e$ increases, the $T(1-e)$ samples that remain after exploration decrease thereby shrinking the possible values that $\kappa_T$ can take.
\end{remark*}

\begin{figure}[t]
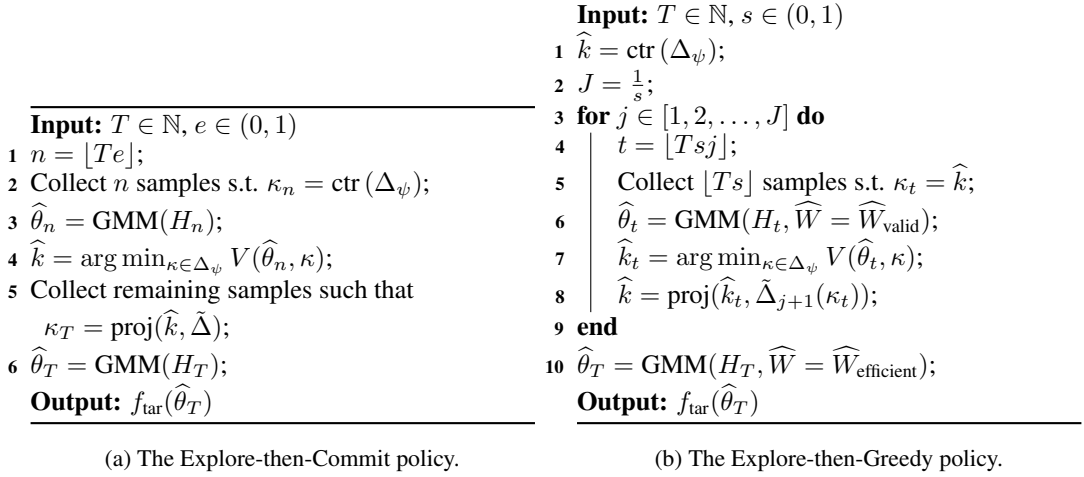

\centering
\begin{subfigure}[b]{0.48\textwidth}
{
    \setlength{\interspacetitleruled}{0pt}%
    \setlength{\algotitleheightrule}{0pt}%
    \begin{algorithm}[H]
    \SetAlgoLined
    \KwInput{$T \in \mathbb{N}$, $e \in (0, 1)$}
    $n = \floor{Te}$\;
    Collect $n$ samples s.t. $\kappa_{n} = \ctrSim$\;
    $\widehat{\theta}_{n} = \text{GMM}(H_{n})$\;
    $\widehat{k} = \arg\min_{\kappa \in \ChoiceSimplex} V(\widehat{\theta}_{n}, \kappa)$\;
    Collect remaining samples such that $\kappa_T = \text{proj}(\widehat{k}, \Tilde{\Delta})$\;
    $\widehat{\theta}_T = \text{GMM}(H_T)$\;
    \KwOutput{$f_{\text{tar}}(\widehat{\theta}_T)$}
    \end{algorithm}
}
\caption{The Explore-then-Commit policy.}
\label{fig:policy-algorithm-etc}
\end{subfigure}
\hfill
\begin{subfigure}[b]{0.48\textwidth}
{
    \setlength{\interspacetitleruled}{0pt}%
    \setlength{\algotitleheightrule}{0pt}%
    \begin{algorithm}[H]
    \SetAlgoLined
    \KwInput{$T \in \mathbb{N}$, $s \in (0, 1)$}
    $\widehat{k} = \ctrSim$\;
    $J = \frac{1}{s}$\;
    \For{$j \in [1, 2, \hdots, J]$}{
        $t = \floor{Tsj}$\;
        Collect $\floor{Ts}$ samples s.t. $\kappa_{t} = \widehat{k}$\;
        $\widehat{\theta}_{t} = \text{GMM}(H_t, \widehat{W}=\widehat{W}_{\text{valid}})$\;
        $\widehat{k}_{t} = \arg\min_{\kappa \in \ChoiceSimplex} V(\widehat{\theta}_{t}, \kappa)$\;
        $\widehat{k} = \text{proj}(\widehat{k}_{t}, \Tilde{\Delta}_{j+1}(\kappa_t))$\;
    }
    $\widehat{\theta}_T = \text{GMM}(H_T, \widehat{W}=\widehat{W}_{\text{efficient}})$\;
    \KwOutput{$f_{\text{tar}}(\widehat{\theta}_T)$}
    \end{algorithm}
}
\caption{The Explore-then-Greedy policy.}
\label{fig:policy-algorithm-greedy}
\end{subfigure}
\caption{Algorithms for OMS-ETC and OMS-ETG.}
\end{figure}

\begin{theorem}[Regret of OMS-ETC]\label{thm:etc-regret}
Suppose that (i) Conditions (i)-(iii) of Proposition~\ref{prop:asymptotic-normality} hold
and
(ii) Assumption~\ref{assump:kappa-star-identify} holds.
Case (a): For a fixed $e \in (0, 1)$,
if
$\kappa^* \in \Tilde{\Delta}$, then the regret converges to zero: $R_\infty(\pi_{\text{ETC}}) = 0$.
If $\kappa^* \notin \Tilde{\Delta}$, then $\pi_{\text{ETC}}$ suffers constant regret: $R_\infty(\pi_{\text{ETC}}) = r$ for some constant $r > 0$.
Case (b): If $e$ depends on $T$ such that $e = o(1)$ 
and $Te \rightarrow \infty$ as $T \to \infty$ (e.g. $e = \frac{1}{\sqrt{T}}$),
then $\forall \theta^* \in \Theta$,
we have $R_\infty(\pi_{\text{ETC}}) = 0$.
\end{theorem}

The theorem provides sufficient conditions for when the regret converges to zero.
Case~(a) of the theorem shows that if we explore 
for a fixed fraction of the horizon $T$,
the regret will only converge to zero if $\kappa^{*}$
is inside the feasible region.
Thus the regret will \textit{not} converge to zero
over the entire parameter space $\Theta$
as there would be certain parameter values 
for which $\kappa^*$ would be 
outside the feasible region.
Case~(b) shows that
we \textit{can} achieve 
zero asymptotic regret
for every $\theta^{*} \in \Theta$
by setting $e$ such that it 
becomes asymptotically negligible ($e \in o(1)$).
The main idea in the proof (see Appendix~\ref{sec:apdx-proof-etc-regret}) is to show 
that $\kappa_T \ConvProb \kappa^*$ 
and apply Proposition~\ref{prop:asymptotic-normality}.
In Case~(b), the feasible region $\Tilde{\Delta}$
asymptotically covers the entire simplex
($\Tilde{\Delta} \xrightarrow[]{T \to \infty} \ChoiceSimplex$)
and this is sufficient to show that $\kappa_T \ConvProb \kappa^*$.

\subsection{Online Moment Selection via Explore-then-Greedy (OMS-ETG)}

We extend OMS-ETC 
by periodically updating our estimate of $\kappa^*$
as we collect additional samples
instead of committing to a value after exploration.
The data is collected in batches.
OMS-ETG (Figure~\ref{fig:policy-algorithm-greedy}) is characterized 
by a batch fraction $s \in (0, 1)$.
The algorithm runs for $J = \frac{1}{s}$ rounds 
and we collect $b = Ts$ samples in each round.
In the first round, we explore 
and thus $\kappa_{b} = \ctrSim$.
After every round $j \in [J - 1]$, 
we estimate $\widehat{\theta}_{bj}$ 
and the oracle selection ratio: 
$\widehat{k}_{bj} = \arg\min_{\kappa \in \ChoiceSimplex} V(\widehat{\theta}_{bj}, \kappa)$.
The feasible region for round $j+1$ 
(the set of values of $\kappa_{b(j+1)}$ can take) 
is $\Tilde{\Delta}_{j+1}(\kappa_{bj}) = \left\{ \frac{j \kappa_{bj} + \kappa}{j + 1} : \kappa \in \ChoiceSimplex  \right\}$ 
(proof in Appendix~\ref{sec:appendix-feasibility-region}).
In round $(j+1)$, we (greedily) collect samples 
such that $\kappa_{b(j+1)}$ 
is as close to $\widehat{k}_{j+1}$ as possible:
$\kappa_{b(j+1)} = \text{proj}\left( \widehat{k}_{j}, \Tilde{\Delta}_{j+1}(\kappa_{bj}) \right)$.

Theorem~\ref{thm:regret-etg} states sufficient conditions for when OMS-ETG has zero regret.
We first state a finite-sample tail bound for the two-step GMM estimator under adaptively collected
(non-i.i.d.)
data in Lemma~\ref{lemma:gmm-conc-inequality} (which might be of independent interest) and use it to prove the theorem.

\begin{property}[Concentration]\label{property:concentration}
Let $\Tilde{a}_i(\theta) := \Tilde{a}(X_i; \theta) \in \R$ with $X_i$ sampled i.i.d., $\Tilde{a}_{*}(\theta) = \E\left[ \Tilde{a}(X_i; \theta) \right]$,
$A(X_i, \theta) = \frac{\partial \Tilde{a}(X_i; \theta)}{\partial \theta}$,
$u_i(\eta) = \sup_{\theta, \theta' \in \Theta, \left\| \theta - \theta' \right\| \leq \eta} \left| \Tilde{a}_i(\theta) - \Tilde{a}_i(\theta') \right|$, and 
$u_{*}(\eta) = \E\left[ u_i(\eta) \right]$.
Let $L_1, \eta_0$, and $A_0$ be some positive constants.
We say that $\Tilde{a}_i(\theta)$ satisfies the Concentration property if
(i) $\Tilde{a}_{*}(\theta)$ is $L_1$-Lipschitz,
(ii) $\forall \theta \in \Theta, \,\, \left[\Tilde{a}_i(\theta) - \Tilde{a}_{*}(\theta)\right]$ is sub-Exponential,
(iii) $\E\left[ \left\| A(X_i, \theta) \right\| \right] < A_0 < \infty$, and
(iv) one of the following two conditions hold: (a) $\forall \eta \in (0, \eta_0), \, \left[u_i(\eta) - u_{*}(\eta)\right]$ is sub-Exponential,
or (b) $\sup_{\theta \in \Theta} \left\| A(X_i, \theta) \right\|$ is sub-Exponential.
\end{property}

\begin{remark*}
Property~\ref{property:concentration} is used to derive a uniform law (see Lemma~\ref{lemma:apdx-uniform-law-dependent-data}) that is used to prove Lemma~\ref{lemma:gmm-conc-inequality}.
Property~\ref{property:concentration}(iv) might be hard to check but (iv)(a) is satisfied for bounded function classes, i.e., when $\|\Tilde{a}_i\|_\infty < A < \infty$ (see Proposition~\ref{prop:apdx-boundedness-and-condition-conc}) and (iv)(b) for linear function classes with sub-Exponential data (see Proposition~\ref{prop:apdx-ulln-linearity}).
For the linear IV model in Example~\ref{example:iv-graph}, Property~\ref{property:concentration}(iv) would hold if $Z^2_i$ is sub-Exponential (e.g., when $Z_i$ is sub-Gaussian).
\end{remark*}

\begin{lemma}[GMM concentration inequality]\label{lemma:gmm-conc-inequality}
Let $\lambda_{*}, C_0, \eta_1, \eta_2$, and $\delta_0$ be some positive constants.
Suppose that (i) Assumption~\ref{assump:standard-gmm} holds;
(ii) $\forall j, \, \Tilde{g}_{i, j}(\theta)$ satisfies Property~\ref{property:concentration};
(iii) The spectral norm of the GMM weight matrix $\widehat{W}$ is upper bounded with high probability: $\forall \delta \in \left(0, C_0 \right), \,\, \P \left( \|\widehat{W} \| \leq \lambda_* \right) \geq 1 - \frac{1}{\delta^{D}} \exp\left\{ - \OM\left( T \delta^2 \right) \right\}$ (see Remark~\ref{remark:weight-matrix-concentration-condition});
(iv) (Local strict convexity) 
$ \forall \theta \in N_{\eta_1}(\theta^*), \,\, \P \left( \left\| \frac{\partial^2 \Bar{Q}}{\partial \theta^2} (\theta)^{-1} \right\| \leq h \right) = 1$
($\Bar{Q}(\theta)$ is defined in Assumption~\ref{assump:standard-gmm}a);
(v) (Strict minimization)
$\forall \theta \in N_{\eta_2}(\theta^*)$, there is a unique minimizer $\kappa(\theta) = \arg\min_{\kappa} V(\theta, \kappa)$ s.t. $V(\theta, \kappa) - V(\theta, \kappa(\theta)) \leq c \delta^2 \implies \|\kappa - \kappa(\theta)\| \leq \delta$; and
(vi) $\sup_{\kappa} |V(\theta, \kappa) - V(\theta', \kappa)| \leq L \| \theta - \theta' \|$.
Then, for $\widehat{k}_t = \arg\min_{\kappa \in \ChoiceSimplex} V(\widehat{\theta}^{(\pi)}_T, \kappa)$,
any policy $\pi$, and $\forall \delta \in (0, \delta_0)$,
\begin{align*}
    \P\left( \left\| \widehat{\theta}^{(\pi)}_T - \theta^* \right\| > \delta \right) < \frac{1}{\delta^{2D}} \exp\left\{ - \OM\left( T \delta^4 \right) \right\} \,\, \text{and} \,\, \P\left( \left\| \widehat{k}_T - \kappa^* \right\| > \delta \right) < \frac{1}{\delta^{4D}} \exp\left\{ - \OM \left( T \delta^8 \right) \right\}.
\end{align*}
Better rates for $\widehat{k}_T$ are applicable under additional restrictions on $\theta^*$ (see Lemma~\ref{lemma:apdx-k-hat-concentration-faster}).
\end{lemma}

To derive the tail bound for $\widehat{\theta}_T$, we first prove that the minimized empirical GMM objective is close to $\Bar{Q}(\theta^*)$ with high probability (w.h.p.) (using Conditions~(i)-(iii)).
Then we show that $\widehat{\theta}_T$ is close to $\theta^*$ w.h.p. (using Condition~(iv)).
Next, we use the inequality for $\widehat{\theta}_T$ to derive the tail bound for $\widehat{k}$.
For this, we show that $V(\theta^*, \widehat{k})$ is close to $V(\theta^*, \kappa^*)$ (using Condition~(vi)) and then show that $\widehat{k}$ is close to $\kappa^*$ w.h.p. (using Condition~(v)).

\begin{theorem}[Regret of OMS-ETG]\label{thm:regret-etg}
Suppose that Conditions (i)-(iv) of Proposition~\ref{prop:asymptotic-normality} hold.
Let $\Tilde{\Delta}(s) = \{ s \kappa_b + (1-s)\kappa : \kappa \in \ChoiceSimplex \}$.
Case (a): For a fixed $s \in (0, 1)$, if the oracle selection ratio $\kappa^* \in \Tilde{\Delta}(s)$, then the regret converges to zero: $R_\infty(\pi_{\text{ETG}}) = 0$.
If $\kappa^* \notin \Tilde{\Delta}(s)$, then $R_\infty(\pi_{\text{ETG}}) > 0$ (non-zero regret).
Case (b): Now also suppose that the conditions for Lemma~\ref{lemma:gmm-conc-inequality} hold.
If $s = C T^{\eta - 1}$ for some constant $C$ and any $\eta \in [0, 1)$, then $\forall \theta^* \in \Theta, \,\, R_\infty(\pi_{\text{ETG}}) = 0$.
\end{theorem}

Similar to OMS-ETC, Case~(a) of the theorem shows that if the batch size is a constant fixed fraction, some values of $\kappa^*$ will be outside the feasible region and thus we do not get zero regret over the entire parameter space.
Case~(b) shows that to get zero asymptotic regret for every $\theta^* \in \Theta$, $s$ must depend on $T$ and be asymptotically negligible.
But unlike OMS-ETC, the estimate of $\kappa^*$ is updated after every round.
The batch fraction $s$ can be as small as $C T^{-1}$ allowing the agent to collect a constant number of samples in each batch.

We prove the theorem by showing that $\kappa_T \ConvProb \kappa^*$ and applying Proposition~\ref{prop:asymptotic-normality}.
To do so for Case~(b),
we show that
the estimated oracle ratio after \textit{every} round is
close to $\kappa^*$, i.e.,
$\forall \epsilon > 0, \,\, \P\left( \forall j \in [J - 1], \,\, \widehat{k}_{bj} \in N_{\epsilon}(\kappa^*) \right) \xrightarrow[T \to \infty]{} 1$ (by using Lemma~\ref{lemma:gmm-conc-inequality}).
Since we move as close as possible to $\widehat{k}_{bj}$ after every round,
this ensures that $\forall \epsilon >0, \, \P \left( \kappa_T \in N_{\epsilon}(\kappa^*) \right) \rightarrow 1$ (and thus $\kappa_T \ConvProb \kappa^*$).

Both OMS-ETC and OMS-ETG are asymptotically equivalent as they can achieve zero regret for every $\theta^* \in \Theta$.
But our experiments show that OMS-ETG outperforms OMS-ETC
(in terms of regret)
for small sample sizes.
This may be because with OMS-ETG, the estimate of $\kappa^*$ keeps improving as more samples are collected instead of being fixed after exploration. 
This suggests that better estimates of $\kappa^*$ may lead
to higher-order reductions in MSE.

\begin{remark}[Weight matrix $\widehat{W}$]\label{remark:weight-matrix-concentration-condition}
For OMS-ETG, the GMM weight matrix $\widehat{W}$ needs to satisfy Condition~(iii) of Lemma~\ref{lemma:gmm-conc-inequality} till round $(J-1)$. We denote this matrix by $\widehat{W}_{\text{valid}}$ in Figure~\ref{fig:policy-algorithm-greedy}.
For the final step of OMS-ETG, we use the standard efficient two-step GMM weight matrix (denoted by $\widehat{W}_{\text{efficient}}$ in Figure~\ref{fig:policy-algorithm-greedy}) and thus the final estimate of $\theta^*$ is still asymptotically efficient.
Condition~(iii) of Lemma~\ref{lemma:gmm-conc-inequality} would hold
for the efficient two-step GMM weight matrix (i.e., for $\widehat{W} := \left[\widehat{\Omega}_T(\widehat{\theta}^{(\text{os})}_T)\right]^{-1}$)
if $\,\forall (j, k), \, [\Tilde{g}_{i, j}(\theta) \Tilde{g}_{i, k}(\theta)]$ satisfies Property~\ref{property:concentration} (see Lemma~\ref{lemma:apdx-sufficient-weight-matrix-bounded-condition} for proof).
This would hold if the moments $\Tilde{g}_{i,j}(\theta)$ are uniformly bounded.
Condition~(iii) can also be satisfied with a regularized weight matrix: $\widehat{W} := \left[\widehat{\Omega}_T(\widehat{\theta}^{(\text{os})}_T) + \lambda_W I \right]^{-1}$ for some
$\lambda_W > 0$
as this ensures that $\| \widehat{W} \| \leq \lambda_W^{-1}$.
\end{remark}

\paragraph{Additional exploration.}
MAB algorithms usually require additional exploration to perform well (e.g., $\epsilon$-greedy and upper confidence bound strategies \citep[Chapter~7]{lattimore2020bandit}). However, we empirically noticed that additional exploration hurts performance in our setup.
This might be
because unlike 
typical bandit setups
where pulling one arm does not improve the estimates of another arm,
querying any data source can improve the estimates of the model parameter $\theta^*$ in our setup.

\section{Incorporating a Cost Structure}
\label{sec:cost-structure}
In many real-world settings, the agent has a budget constraint and 
must pay a different cost to query each data source.  
We adapt OMS-ETC and OMS-ETG to this setting where a cost structure is associated with the data sources in $\psi$. We prove that these policies still have zero asymptotic regret for every $\theta^* \in \Theta$.
Let $\left(\psi_i\right)_{i \in [\CardC]}$ be an indexed family.
We denote the (known) budget by $B \in \mathbb{N}$ and by $c \in \R^{\CardC}_{>0}$, a cost vector such that $c_i$ is the cost of querying data source $\psi_i$.
Due to the cost structure, the horizon $T$ is a random variable dependent on $\pi$ with $T =
\floor*{\frac{B}{\kappa^\top_T c}}$.
The setting in Section~\ref{sec:adaptive-data-collection} is a special case of this formulation when $\forall i, c_i = 1$ and $B = T$. We defer proofs to Appendix~\ref{sec:appendix-cost-structure}.

For a fixed policy $\pi_k$, we have $\kappa^{(\pi_k)}_T = k$, for some constant $k \in \ChoiceSimplex$. By Proposition~\ref{prop:asymptotic-normality}, as $B \to \infty$, we have
$\sqrt{B} \left(\widehat{\beta}^{(\pi_k)}_T - \beta^* \right) \xrightarrow{d} \mathcal{N}\left(0, V(\theta^*, k) \left( k^\top c \right) \right)$.
Here we scale by $\sqrt{B}$ instead of $\sqrt{T}$ to make comparisons across policies meaningful.
The \textit{oracle selection ratio} is now defined as $\kappa^* = \arg\min_{\kappa \in \ChoiceSimplex} \left[V(\theta^*, \kappa) \left( \kappa^\top c \right)\right]$ and the \textit{asymptotic regret} of policy $\pi$ now is
\begin{align*}
    R_\infty(\pi) = \Amse\left( \sqrt{B} \left(\widehat{\beta}^{(\pi)}_T - \beta^* \right) \right) - V(\theta^*, \kappa^*) \left((\kappa^*)^\top c \right).
\end{align*}
OMS-ETC-CS (\textit{OMS-ETC with cost structure}) is an adaptation of OMS-ETC for this setting.
We use $Be$ budget to explore and estimate $\kappa^*$ by $\widehat{k} = \arg\min_{\kappa \in \ChoiceSimplex} \left[V(\widehat{\theta}_{T_e}, \kappa) \left( \kappa^\top c \right)\right]$, where $T_e = \floor*{\frac{eB}{\kappa^\top_{T_e} c }}$ and $\kappa_{T_e} = \ctrSim$.
Exploration strategies that utilize the cost structure 
can also be used 
(e.g., evenly dividing the budget across data sources while exploring).
With the remaining budget, we collect samples such that $\kappa_T = \text{proj}\left( \widehat{k}, \Tilde{\Delta} \right)$, where $\Tilde{\Delta}$ is the feasibility region (expression given in Appendix~\ref{sec:appendix-feasibility-region}). The following proposition shows that OMS-ETC-CS can achieve zero regret.
\begin{proposition}[Regret of OMS-ETC-CS]\label{prop:regret-etc-cost-structure}
Suppose that the conditions of Theorem~\ref{thm:etc-regret} hold. 
If $e = o(1)$ such that $Be \rightarrow \infty$ as $B \to \infty$, then
$\forall \theta^* \in \Theta, \, R_{\infty}(\pi_{\text{ETC-CS}}) = 0$.
\end{proposition}
We propose two ways of adapting OMS-ETG to this setting: (i) OMS-ETG-FS (\textit{OMS-ETG with fixed samples}) where we collect a fixed number of samples in every round and (ii) OMS-ETG-FB (\textit{OMS-ETG with fixed budget}) where we spend a fixed fraction of the budget in every round.

Let $c_{\max} = \max_{i \in [\CardC]} c_i$ and $c_{\min} = \min_{i \in [\CardC]} c_i$. In OMS-ETC-FS (Figure~\ref{fig:policy-algorithm-etg-fs}), we collect $b = \floor*{\frac{Bs}{c_{\max}}}$ samples in every round except the last one (the last batch can be smaller as we may not have enough budget left for a full batch).
The number of rounds $J$ is random since, in every round, depending on what we collect, we use up a different amount of the budget.
Like OMS-ETG, after each round, we estimate $\kappa^*$ and greedily collect samples to get as close to it as possible. 

In OMS-ETG-FB (Figure~\ref{fig:policy-algorithm-etg-fb}), we spend $Bs$ budget in each round. Thus the number of rounds $J$ is fixed with $J = \frac{1}{s}$
but the number of samples collected per round is now random.
Like OMS-ETG, after each round, we estimate $\kappa^*$ and collect samples to get as close to the estimate as possible.
The next two Propositions show that both OMS-ETG-FS and OMS-ETC-FB have zero asymptotic regret.

\begin{proposition}[Regret of OMS-ETG-FS]\label{prop:regret-etg-fs}
Suppose that the conditions of Theorem~\ref{thm:regret-etg}b hold.
If $s = B^{\eta - 1}$ for any $\eta \in \left[0, 1\right)$, then
$\forall \theta^* \in \Theta, \, R_{\infty}\left(\pi_{\text{ETG-FS}}\right) = 0$.
\end{proposition}

\begin{proposition}[Regret of OMS-ETG-FB]\label{prop:regret-etg-fb}
Suppose that the conditions of Theorem~\ref{thm:regret-etg}b hold.
If $s = B^{\eta - 1}$ for any $\eta \in \left[0, 1\right)$,
then
$\forall \theta^* \in \Theta, \, R_{\infty}\left(\pi_{\text{ETG-FB}}\right) = 0$.
\end{proposition}

\section{Experiments}
\label{sec:experiments}
\subsection{Synthetic data}\label{sec:experiments-synthetic-data}
We validate our methods on synthetic data generated from known causal graphs (see Appendix~\ref{sec:appendix-experiments} for parameter values and  moment conditions used).
In our experiments (including the ones in Section~\ref{sec:experiments-semi-synthetic-data}),
for OMS-ETG, we use the regularized GMM weight matrix with $\lambda_W := 0.01$ (see Remark~\ref{remark:weight-matrix-concentration-condition}).
The regret is only minimally affected
with an unregularized matrix ($\lambda_W := 0$)
despite theoretical guarantees only holding for the regularized case 
(the maximum change in regret was $1.98\%$)
and thus our conclusions do not change.
We first simulate data from a linear IV graph (Figure~\ref{fig:disjoint-iv-graph})
and compare the performance of different polices (Figure~\ref{fig:disjoint-iv-regret-curve}).
We use $\psi = \{ \{Z, X\}, \{Z, Y\} \}$ and assume that both sources cost the same.
We set parameter values such that $\kappa^* \approx [0.36, 0.64]^\top$.
We compare policies based on \textit{relative regret (RR)} with respect to
the oracle policy:
\begin{align*}
    \text{Relative regret} = \text{RR}(\pi) = \frac{\text{MSE}^{(\pi)} - \text{MSE}^{(\text{oracle})} }{\text{MSE}^{(\text{oracle})}} \times 100\%.
\end{align*}
The MSE values are computed across $12,000$ runs.
The label \textit{etc}\_$\{x\}$ in the plot refers to OMS-ETC with exploration fraction $e = x$ (e.g., \textit{etc\_0.1} means $e = 10\%$).
Similarly, \textit{etg}\_$\{x\}$ refers to OMS-ETG with $s = x$.
We see that as the horizon increases, the RR of all policies converges to zero.
This supports the claim that both OMS-ETC and OMS-ETG have zero asymptotic regret.
However, when the horizon is small ($T = 300$), both \textit{etc}\_\textit{0.1} and \textit{etc}\_\textit{0.2}
perform poorly 
due to insufficient exploration.
In contrast, OMS-ETG has close to zero RR even for small horizons which shows that repeatedly improving the parameter estimates can lead to faster convergence of regret.

\begin{figure}
\centering
\begin{subfigure}[b]{0.3\textwidth}
\begin{tikzpicture}
    \node[state] (1) {$X$};
    \node[state] (2) [left =of 1] {$Z$};
    \node[state] (3) [right =of 1] {$Y$};
    
    \path (2) edge node[]{} (1);
    \path (1) edge node[]{} (3);
    \path[bidirected] (1) edge[bend left=60] node[el,above]{} (3);
\end{tikzpicture}
\caption{Instrumental variable graph.}
\label{fig:disjoint-iv-graph}
\end{subfigure}
\hfill
\begin{subfigure}[b]{0.3\textwidth}
\begin{tikzpicture}
    \node[state] (1) {$X$};
    \node[state] (2) [left =of 1,yshift=-.7cm] {$Z_1$};
    \node[state] (3) [left =of 1,yshift=.7cm] {$Z_2$};
    \node[state] (4) [right =of 1] {$Y$};
    
    \path (2) edge node[]{} (1);
    \path (3) edge node[]{} (1);
    \path (1) edge node[]{} (4);
    \path[bidirected] (1) edge[bend left=60] node[el,above]{} (4);
\end{tikzpicture}
\caption{Two IVs graph.}
\label{fig:multiple-iv-graph}
\end{subfigure}
\hfill
\begin{subfigure}[b]{0.3\textwidth}
\begin{tikzpicture}
    \node[state] (1) {$W$};
    \node[state] (2) [left =of 1,yshift=-1.5cm] {$X$};
    \node[state] (3) [right =of 2] {$M$};
    \node[state] (4) [right =of 3] {$Y$};
    
    \path (1) edge node[above]{} (2);
    \path (1) edge node[above]{} (4);
    \path (2) edge node[above]{} (3);
    \path (3) edge node[above]{} (4);
\end{tikzpicture}
\caption{Confounder-mediator graph.}
\label{fig:confounder-mediator-graph}
\end{subfigure}
\caption{Examples of causal models---with treatment $X$ and outcome $Y$---where the ATE can be identified by different data sources returning different subsets of variables.}
\end{figure}
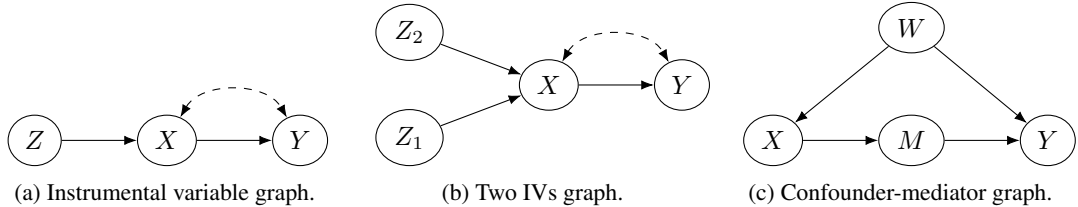

\begin{figure}
\centering
\begin{subfigure}[b]{0.32\textwidth}
\includegraphics[scale=0.33]{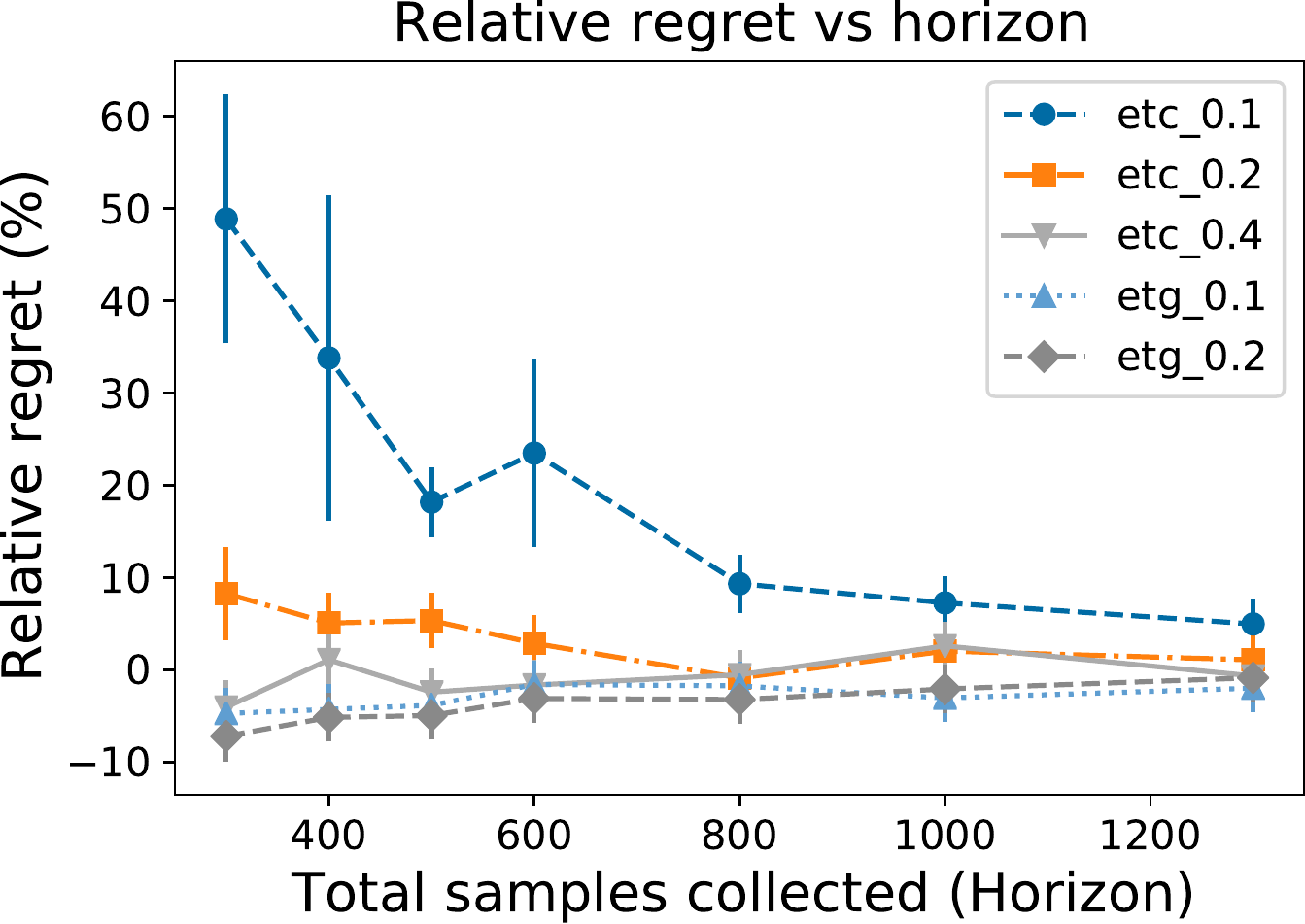}
\caption{IV model.}
\label{fig:disjoint-iv-regret-curve}
\end{subfigure}
\hfill
\begin{subfigure}[b]{0.32\textwidth}
\includegraphics[scale=0.33]{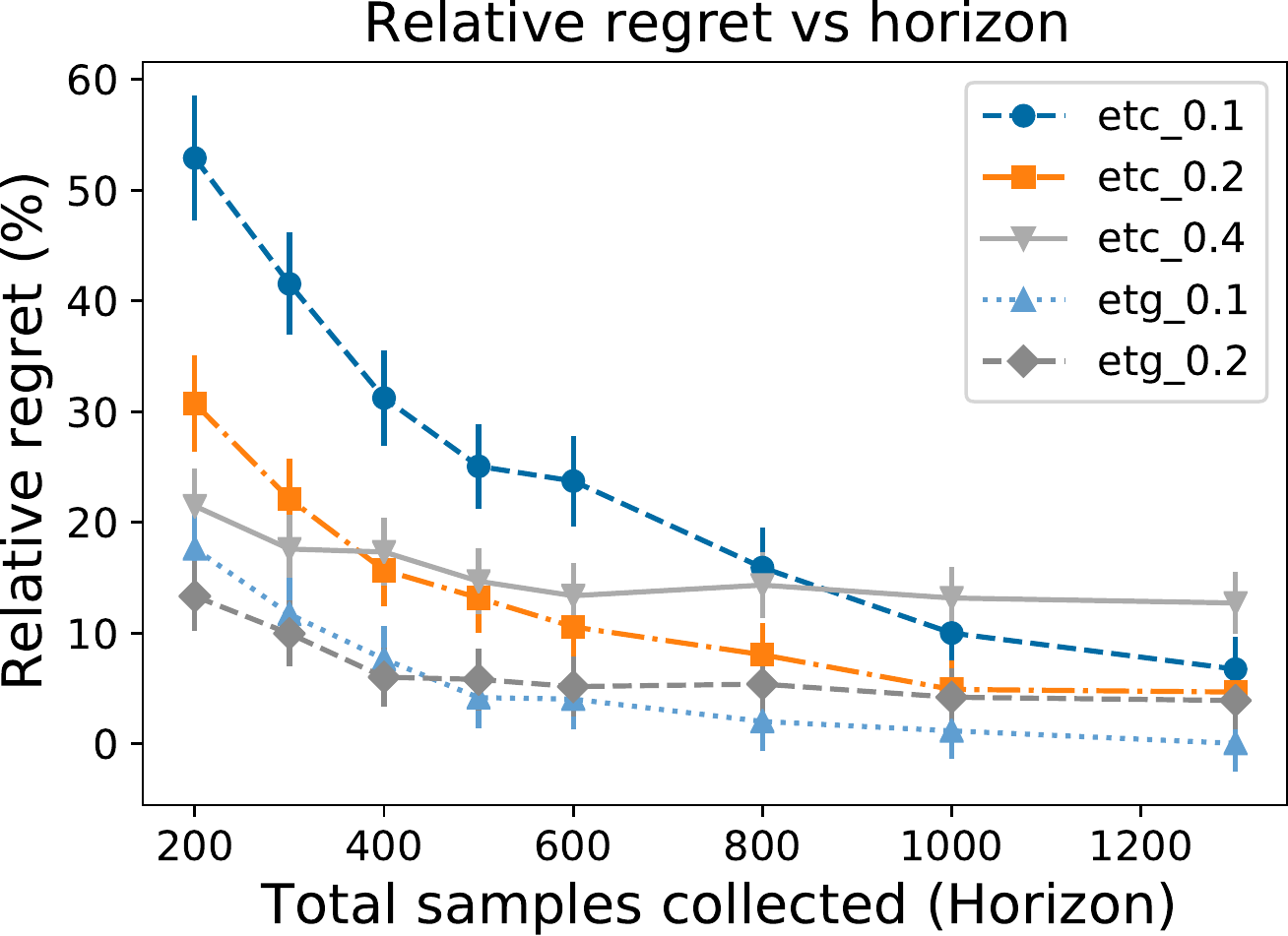}
\caption{Two IVs model.}
\label{fig:multiple-iv-corner-regret-curve}
\end{subfigure}
\hfill
\begin{subfigure}[b]{0.32\textwidth}
\includegraphics[scale=0.33]{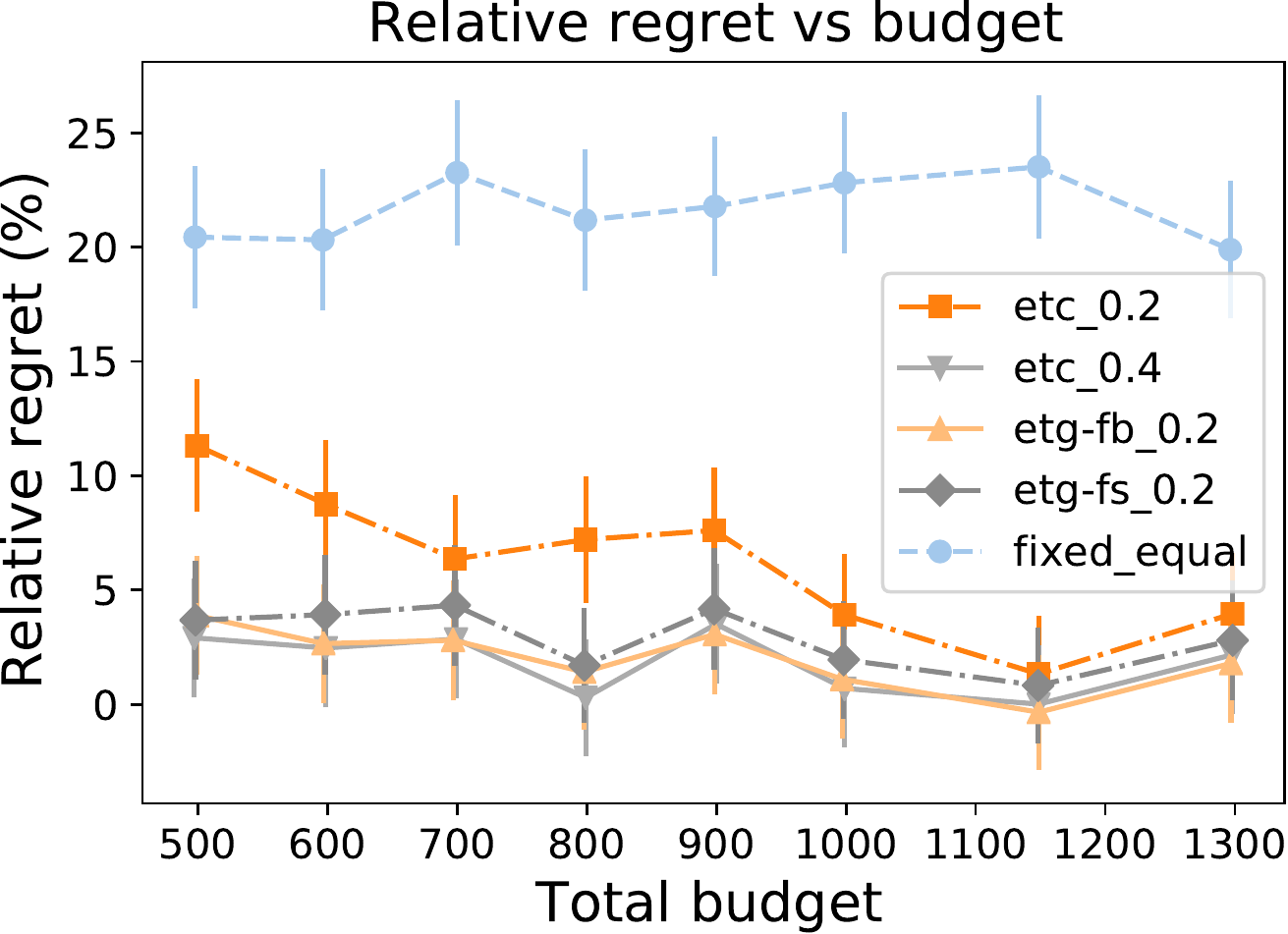}
\caption{Confounder-mediator model.}
\label{fig:frontdoor-backdoor-cost-structure}
\end{subfigure}
\caption{Relative regret (RR) across horizons/budgets for different policies (error bars denote $95\%$ CIs).
(a) RR for the IV model (Figure~\ref{fig:disjoint-iv-graph}) with $\psi = \{\{ Z, X \}, \{ Z, Y \}\}$. (b) RR for the two IVs model (Figure~\ref{fig:multiple-iv-graph}) with $\psi = \{\{ X, Y, Z_1 \}, \{ X, Y, Z_2 \}\}$. (c) RR for the confounder-mediator model (Figure~\ref{fig:confounder-mediator-graph}) with $\psi = \{\{ X, Y, W \}, \{ X, Y, M \}\}$ and a cost structure.
}
\end{figure}

Next, we simulate data from a linear graph with two IVs (Figure~\ref{fig:multiple-iv-graph}).
Here $\psi = \{ \{ X, Y, Z_1 \}, \{X, Y, Z_2 \} \}$ and
both choices cost the same.
We set the parameters such that $\kappa^* = [0, 1]^\top$ ($\kappa^*$ is on the corner of the simplex).
We compare the RR across various horizons (Figure~\ref{fig:multiple-iv-corner-regret-curve}).
We see the OMS-ETC performs worse than OMS-ETG for small horizons.
One difference from the previous case (Figure~\ref{fig:disjoint-iv-regret-curve}) is that \textit{etc}\_\textit{0.4} performs poorly even for large horizons.
This is because after using $40\%$ of the samples for exploration, the feasibility region is not large enough to get close to the corner of the simplex.
This demonstrates another benefit of OMS-ETG over OMS-ETC:
OMS-ETG can achieve close to zero regret in finite samples when the oracle ratio $\kappa^*$ is either on the boundary or in the interior of the simplex.

Finally, we simulate data from a linear confounder-mediator graph
(Figure~\ref{fig:confounder-mediator-graph}).
Here, both the backdoor (using $\{X, Y, W\}$) and frontdoor (using $\{X, Y, M\}$) adjustments are applicable \citep[Section~3.3]{pearl2009causality}.
We use $\psi = \{ \{X, Y, W\}, \{X, Y, M\} \}$ with cost structure $c = [1.8, 1]^\top$ (confounders $W$ cost more than the mediators $M$).
We set the parameters such that  $\kappa^* \approx [0.15, 0.85]^\top$.
We see similar conclusions as the previous cases.
OMS-ETC with low exploration performs poorly but converges for large horizons.
Both OMS-ETG variants---OMS-ETG-FS and OMS-ETG-FB---have close to zero RR for all horizons.
We see no significant difference between the regret of OMS-ETG-FS and OMS-ETG-FB.
The policy
\textit{fixed}\_\textit{equal} is a fixed policy that collects an equal fraction of both subsets.
Its RR does not converge and is substantially higher than the oracle ($\approx 20\%$). 
This demonstrates that adaptive policies can lead to significant gains in MSE over fixed policies and that our methods remain applicable even with an associated cost structure on the data sources.

\subsection{Semi-synthetic data}\label{sec:experiments-semi-synthetic-data}

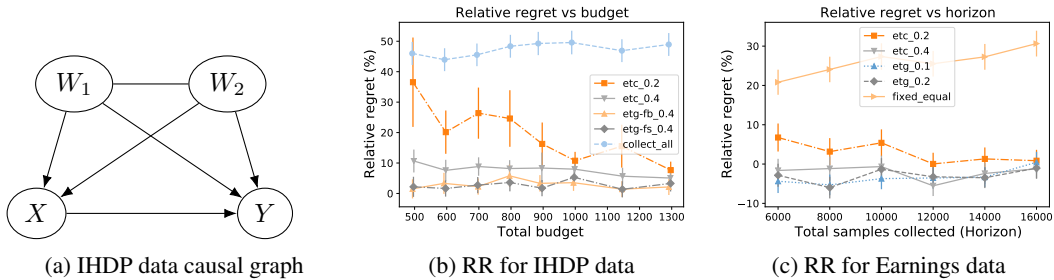
\begin{figure}
\centering
\begin{subfigure}[b]{0.32\textwidth}
\begin{tikzpicture}
    \node[state] (1) {$W_1$};
    \node[state] (2) [right =of 1] {$W_2$};
    \node[state] (3) [below =of 1, xshift=-.5cm] {$X$};
    \node[state] (4) [below =of 2, xshift=.5cm] {$Y$};
    
    \path (3) edge (4);
    \draw[-] (1) -- (2); 
    \path (1) edge (3);
    \path (2) edge (3);
    \path (1) edge (4);
    \path (2) edge (4);
\end{tikzpicture}
\caption{IHDP data causal graph}
\label{fig:ihdp-graph}
\end{subfigure}
\hfill
\begin{subfigure}[b]{0.32\textwidth}
\includegraphics[scale=0.32]{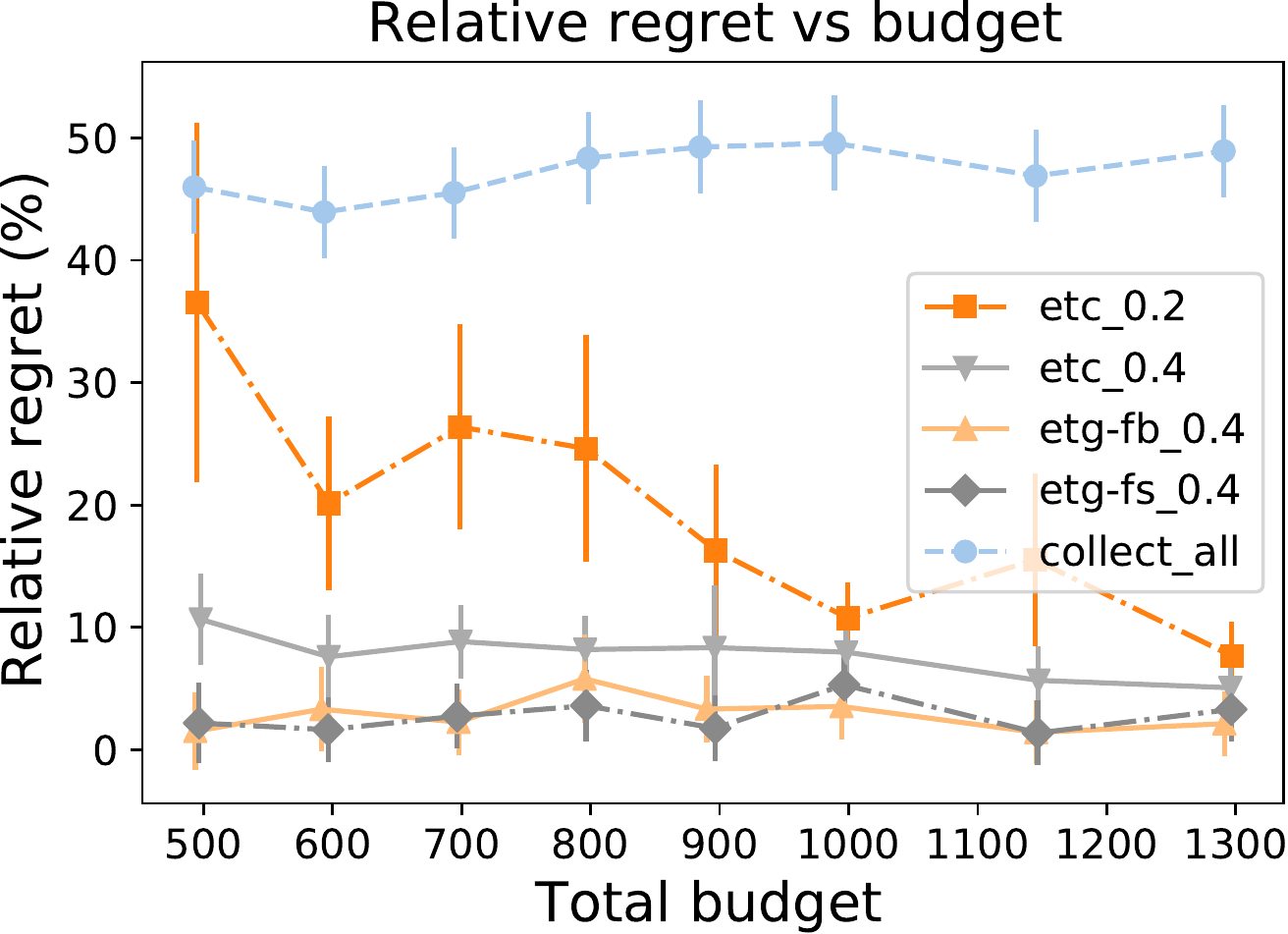}
\caption{RR for IHDP data}
\label{fig:ihdp-regret-curve}
\end{subfigure}
\hfill
\begin{subfigure}[b]{0.32\textwidth}
\includegraphics[scale=0.32]{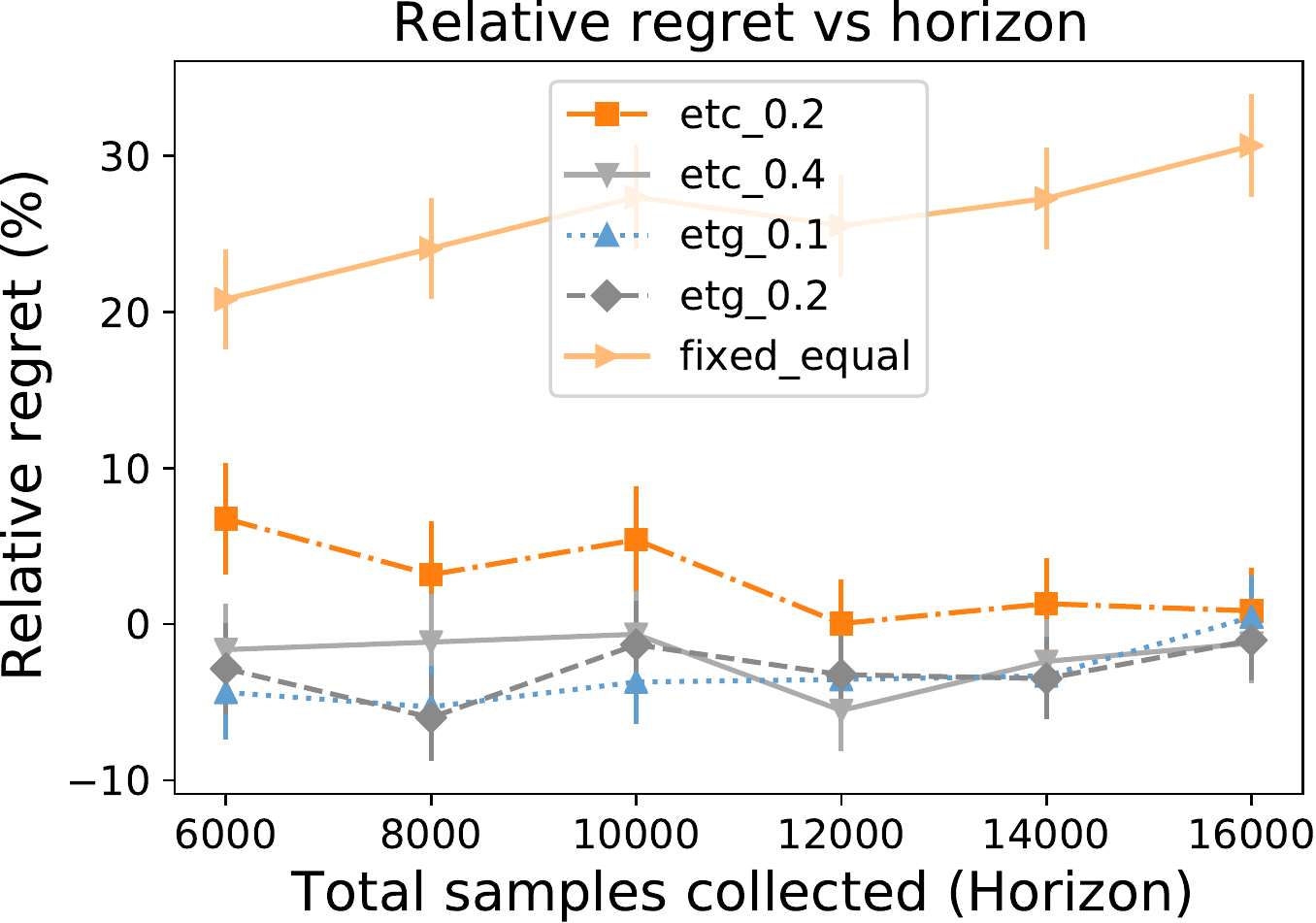}
\caption{RR for Earnings data}
\label{fig:vietnam-regret-curve}
\end{subfigure}
\caption{Relative regret (RR) on semi-synthetic data (error bars denote $95\%$ CIs).
For both IHDP (b) and Earnings data (c), adaptive policies converge to zero RR but fixed policies (\textit{collect}\_\textit{all}, \textit{fixed}\_\textit{equal}) suffer constant regret. OMS-ETG outperforms OMS-ETC for small horizons.
}
\end{figure}

\paragraph{IHDP.}
\citet{hill2011bayesian} constructed a dataset based on the Infant Health and Development Program (IHDP).
The data \citep{Dorie2016} is from a randomized experiment studying the effect of home visits by a trained provider on future cognitive test scores of children.
Following \citet{hill2011bayesian},
we create an observational dataset by removing a non-random subset of the data.
The treatment $X$ is binary.
The dataset contains pre-treatment covariates which are measurements on the mother and the child.
For simplicity, we only use two covariates: birth weight  (continuous) ($W_1$) and whether the mother smoked (binary) ($W_2$).
For each sample of the generated semi-synthetic data, $(X, W_1, W_2)$ are sampled uniformly at random from the real data. The outcome $Y$ (continuous) is simulated: $Y := \alpha_1 W_1 + \alpha_2 W_2 + \beta X + \epsilon_y$, where $\epsilon_y \sim \mathcal{N}(0, \sigma^2_y)$, $\alpha_1, \alpha_2, \beta \in \R$, and $\sigma^2_y \in \R^{+}$ (see Figure~\ref{fig:ihdp-graph}). Here $\alpha_1, \alpha_2, \beta$, and $\sigma_y$ are model parameters with $\beta$ being the ATE (target parameter).

For this experiment, 
we use $\psi = \{ \{X, Y, W_1\}, \{X, Y, W_2\}, \{X, Y, W_1, W_2\} \}$
with cost structure
$c = [1, 3, 3.5]^\top$. 
Thus, at each step, 
the agent can collect either
one of the covariates
or both
of them,
and each choice has a distinct cost. 
Setting model parameters such that 
$\kappa^* \approx [0.59, 0, 0.41]^\top$,
we compare performance across policies
for various total budgets 
(Figure~\ref{fig:ihdp-regret-curve}).
The policy \textit{collect}\_\textit{all} is a fixed policy 
that collects $\{X, Y, W_1, W_2\}$ at every step.
This policy has higher RR 
($\approx 50\%$ higher MSE than the oracle) 
for all budgets demonstrating 
that \textit{collecting all covariates for every sample can be sub-optimal}.
The policy \textit{etc}\_\textit{0.2} does poorly with a small budget whereas 
\textit{etc}\_\textit{0.4} has close to zero RR.
Both OMS-ETG-FB and OMS-ETG-FS have close 
to zero RR for all horizons.

\paragraph{The Vietnam draft and future earnings.} \citet{angrist1990lifetime} computed the effect 
of veteran status on future earnings 
from the Vietnam draft lottery data \citep{angrist1990data}
using an IV (Figure~\ref{fig:disjoint-iv-graph}).
The IV $Z$ (binary) indicates 
whether an individual 
was eligible for the draft 
based on a random lottery.
The treatment $X$ (binary) indicates 
whether they actually served.
The outcome $Y$ (continuous) 
represents their future earnings.
The IV removes bias caused 
by certain types of men 
being more likely to serve.
In this dataset,
$\{Z, X\}$ and $\{Z, Y\}$
were collected using different data sources
(thus $\{ Z, X, Y \}$ are not observed simultaneously)
,
which suits our framework.
We construct a semi-synthetic dataset 
that closely matches the real data 
so that we know the ground-truth causal effect
(needed to compute the MSE). 
For each instance,
we sample $Z$ uniformly at random 
from the empirical distribution.
$X$ is sampled from
the Bernoulli distribution
$\widehat{\P}(X|Z)$ with conditional probabilities given by the empirical distribution (values taken from \citep[Table~2]{angrist1990lifetime}).
We generate the outcome as $Y := \beta X + \gamma + \epsilon$,
where $\epsilon \sim \mathcal{N}(0, \sigma^2_{\epsilon})$ and $\epsilon \notindep X$.
The parameters $\beta, \gamma$, and $\sigma^2_{\epsilon}$
are set such that the distribution of $(Z, Y)$ is close to the real data (see Appendix \ref{sec:apdx-vietnam-earnings-experiment} for details).
We compare the RR of our
policies on this dataset (Figure~\ref{fig:vietnam-regret-curve}).
Most adaptive policies converge 
to near zero RR as the horizon gets large. 
OMS-ETC does poorly with low exploration
while OMS-ETG policies have 
significantly lower RR 
for smaller horizons.
By contrast, \textit{fixed}\_\textit{equal}
(the fixed policy that queries both data sources equally) 
suffers constant regret and has 
$\approx 25\%$ higher regret 
than the oracle even for large horizons.
This demonstrates that
adaptive policies
can lead to substantial MSE gains in a real-world setting.

\section{Conclusion}
\label{sec:conclusion}
This paper takes some initial strides 
towards endogenizing decisions 
about which variables to solicit
into the modeling process. 
Addressing the problem of deciding,
sequentially, which data sources to query
in order to efficiently estimate a parameter,
we developed the online moment selection (OMS) framework
and two instantiations: OMS-ETC and OMS-ETG.
We prove that over the entire parameter space,
adaptive data collection with either method
can provide substantial MSE gains.
While our work focuses on ATE estimation,
our framework is more broadly applicable 
to any parameter identified by moment conditions.
In future work,
we hope to apply our framework
to more general prediction problems,
addressing practical considerations
including high-dimensional data
and complex model classes
(e.g., neural networks).
Moreover, in real-world settings,
common assumptions like ignorability rarely hold.
We hope to extend our framework 
to overcome issues such as model misspecification,
or to overcome biases present in some, 
but not all data sources.

\bibliographystyle{abbrvnat}
\bibliography{refs}

\clearpage

\appendix
\section{Omitted Proofs for Section~\ref{sec:adaptive-data-collection}}\label{sec:appendix-omitted-proofs}

\subsection{Proof of Proposition~\ref{prop:consistency} (Consistency)}

\begin{proposition}[MDS LLN {\citep[Example~7.11]{hamilton1994time}}]\label{prop:apdx-mds-lln}
Let $\Bar{Y}_T$ be the sample mean from a martingale difference sequence (MDS), $\Bar{Y}_T = \frac{1}{T} \sum_{t=1}^T Y_i$, with $\E\left[ |Y_t|^r \right] < \infty$ for some $r > 1$. Then $\Bar{Y}_T \ConvProb 0$.
\end{proposition}

\begin{lemma}[Uniform convergence]\label{lemma:apdx-uniform-convergence-dependent-data}
Let $a_i(\theta) := S_i \Tilde{a}(\theta, X_i)$ be a real-valued function where $S_i \in \{0, 1\}$ is $H_{i-1}$-measurable and $X_i$ are i.i.d. Suppose that (i) $\Theta$ is compact and (ii) $\Tilde{a}(\theta, X_i)$ satisfies Property~\ref{property:mod-of-continuity}. Then
\begin{align*}
    \sup_{\theta \in \Theta} \left| \frac{1}{T} \sum_{i=1}^T \left [ a_i(\theta) - S_i a_{*}(\theta) \right] \right| \ConvProb 0,
\end{align*}
where $a_{*}(\theta) = \E[\Tilde{a}(\theta; X_i)]$.
\end{lemma}
\begin{proof}
We follow a standard uniform law of large numbers proof (e.g. \citet[Lemma~1]{tauchen1985diagnostic}) 
and modify it to work for dependent data.
The key modification is replacing the law of large numbers (LLN) in that proof with a MDS LLN.

Let $\left( \theta_1, \theta_2, \hdots, \theta_K \right)$ be a minimal $\delta$-cover of $\Theta$ and
$N_{\delta}(\theta_k)$ denote the $\delta$-ball around $\theta_k$. 
By compactness of $\Theta$, $K$ is finite.
For $k \in [K]$ and $\theta \in N_{\delta}(\theta_k)$, we have
\begin{align*}
    & \left| \frac{1}{T} \sum_{i=1}^{T} \left[ a_i(\theta) - S_i a_{*}(\theta) \right] \right| \\
     \,\,&= \left| \frac{1}{T} \sum_{i=1}^{T} \left[ a_i(\theta) - a_i(\theta_k) + a_i(\theta_k) - S_i a_{*}(\theta_k) + S_i a_{*}(\theta_k) - S_i a_{*}(\theta) \right] \right| \\
        &\leq \frac{1}{T} \sum_{i=1}^{T} \left| a_i(\theta) - a_i(\theta_k) \right| + \left| \frac{1}{T} \sum_{i=1}^{T} \left[ a_i(\theta_k) - S_i a_{*}(\theta_k) \right] \right| + \frac{1}{T} \sum_{i=1}^{T} \left| S_i a_{*}(\theta_k) - S_i a_{*}(\theta) \right| \\
        &= \frac{1}{T} \sum_{i=1}^{T} \left| S_i \left( \Tilde{a}(\theta; X_i) - \Tilde{a}(\theta_k; X_i) \right) \right| + \left| \frac{1}{T} \sum_{i=1}^{T} \left[ a_i(\theta_k) - S_i a_{*}(\theta_k) \right] \right| + \frac{1}{T} \sum_{i=1}^{T} \left| S_i \left( a_{*}(\theta_k) - a_{*}(\theta) \right) \right| \\
        &\leq \frac{1}{T} \sum_{i=1}^{T}  \left| \Tilde{a}(\theta; X_i) - \Tilde{a}(\theta_k; X_i) \right| + \left| \frac{1}{T} \sum_{i=1}^{T} \left[ a_i(\theta_k) - S_i a_{*}(\theta_k) \right] \right| + \left| a_{*}(\theta_k) - a_{*}(\theta) \right|.
\end{align*}
We now show that each of the three terms on the RHS above is small.
In the third term, by continuity of $a_{*}(\theta)$, $\forall \epsilon > 0, \exists \delta > 0$ s.t. $\left|a_{*}(\theta_k) - a_{*}(\theta)\right| < \epsilon$.

In the second term, $\left[a_i(\theta_k) - S_i a_{*}(\theta_k; S_i)\right]$ is a MDS. 
By Property~\ref{property:mod-of-continuity}(i) and Proposition~\ref{prop:apdx-mds-lln}, we have
$ \\ \left| \frac{1}{T} \sum_{i=1}^{T} \left[ a_i(\theta_k) - S_i a_{*}(\theta_k) \right] \right| \ConvProb 0$.

Next, we examine first term on the RHS.
Let $u_i(\delta) = \sup_{\theta, \theta' \in \Theta, \left\| \theta - \theta' \right\| \leq \delta} \left| \Tilde{a}(\theta, X_i) - \Tilde{a}(\theta', X_i) \right|$.
By continuity of $\Tilde{a}(\theta, X_i)$, compactness of $\Theta$, and the Heine-Cantor theorem, $\Tilde{a}(\theta, X_i)$ is uniformly continuous in $\theta$.
This ensures that $u_i(\delta)$ is continuous in $\delta$ and thus $u_i(\delta) \downarrow 0$ as $\delta \downarrow 0$.
Since $u_i(\delta) \leq 2 A(X_i)$ (by Property~\ref{property:mod-of-continuity}(iii)), using dominated convergence, we have
$\E[u_i(\delta)] \downarrow 0$ as $\delta \downarrow 0$.
Therefore, $\forall \epsilon > 0, \exists \delta > 0$ s.t. $\E[u_i(\delta)] < \epsilon$.
Thus we can write the first term as
\begin{align*}
    \frac{1}{T} \sum_{i=1}^{T} \left| \Tilde{a}(\theta; X_i) - \Tilde{a}_i(\theta_k; X_i) \right| &\leq \frac{1}{T} \sum_{i=1}^{T} u_i(\delta) \\
        &= \frac{1}{T} \sum_{i=1}^{T} u_i(\delta) - \E[u_i(\delta)] + \E[u_i(\delta)] \\
        &\leq \frac{1}{T} \sum_{i=1}^{T} u_i(\delta) - \E[u_i(\delta)] + \epsilon \\
        &\overset{(a)}{=} o_p(1) + \epsilon,
\end{align*}
where (a) follows by the weak law of large numbers which applies because $E[u_i(\delta)] \leq \E[A(X_i)] < \infty$ (by Property~\ref{property:mod-of-continuity}(iii)).
\end{proof}

\begin{proposition*}[Consistency]
Suppose that (i) Assumption~\ref{assump:standard-gmm} holds, (ii) $\forall j \in [M], \, \Tilde{g}_{t, j}(\theta)$ satisfies Property~\ref{property:mod-of-continuity}, and
(iii) $\forall (i, j) \in [M]^2, \, \left[ \Tilde{g}_t(\theta) \Tilde{g}_t(\theta)^\top \right]_{i, j}$ satisfies Property~\ref{property:mod-of-continuity};
Then, for any policy $\pi$, $\widehat{\theta}^{(\pi)}_T \xrightarrow[T \to \infty]{p} \theta^*$.
\end{proposition*}
\begin{proof}
We begin by defining the empirical and population analogues of the two-step GMM objective for a given policy $\pi$:
\begin{align*}
    \text{Empirical objective:} \,\,\, \widehat{Q}^{(\pi)}_T(\theta) &= \left[ \frac{1}{T} \sum_{t=1}^{T} g_t(\theta) \right]^\top \widehat{W} \left[ \frac{1}{T} \sum_{t=1}^{T} g_t(\theta) \right], \\
    \text{Population objective:} \,\,\, \Bar{Q}^{(\pi)}_T(\theta) &= \left[ \frac{1}{T} \sum_{t=1}^{T} \E \left[ g_t(\theta) | H_{t-1} \right] \right]^\top W \left[ \frac{1}{T} \sum_{t=1}^{T} \E \left[ g_t(\theta) | H_{t-1} \right] \right] \\
        &=  \left[ \left( \frac{1}{T} \sum_{t=1}^{T} m(s_t) \right) \otimes g_{*}(\theta) \right]^\top W \left[ \left( \frac{1}{T} \sum_{t=1}^{T} m(s_t) \right) \otimes g_{*}(\theta) \right] \\
        &=  \left[ m_T \otimes g_{*}(\theta) \right]^\top W \left[ m_T \otimes g_{*}(\theta) \right],
\end{align*}
where $g_{*}(\theta) = \E\left[ \Tilde{g}_i(\theta) \right]$ and $m_T = \frac{1}{T} \sum_{t=1}^{T} m(s_t)$. We have $\widehat{W} = \left[\widehat{\Omega}_T(\widehat{\theta}^{(\text{os})}_T)\right]^{-1}$, where $\widehat{\theta}^{(\text{os})}_T$ is the one-step GMM estimate and
\begin{align*}
    \widehat{\Omega}_T(\theta) &= \frac{1}{T} \sum_{t=1}^T \left[ g_t(\theta) g_t(\theta)^\top \right] \\
        &= \frac{1}{T} \sum_{t=1}^T \left( \left[ m(s_t) m(s_t)^\top \right] \otimes \left[ \Tilde{g}_t(\theta) \Tilde{g}_t(\theta)^\top \right] \right).
\end{align*}
Furthermore, we have $W = \left[ m_{\Omega}(\kappa_T) \otimes \Omega(\theta^*) \right]^{-1}$, where
\begin{align*}
    m_{\Omega}(\kappa_T) = \sum_{t=1}^T \left( m(s_t) m(s_t)^\top \right), \\
    \Omega(\theta) = \E\left[ \Tilde{g}_t(\theta) \Tilde{g}_t(\theta)^\top \right].
\end{align*}
The two-step GMM estimator is obtained by minimizing the empirical objective: $\widehat{\theta}_T = \arg\min_{\theta \in \Theta} \widehat{Q}^{(\pi)}_T(\theta)$.
At the true parameter $\theta^*$, $\Bar{Q}^{(\pi)}_T(\theta^*) = 0$ and  
by Assumption~\ref{assump:standard-gmm}(a), 
$\theta^*$ uniquely minimizes $\Bar{Q}^{(\pi)}_T(\theta)$.
By \citet[Theorem~2.1]{newey1994large}, $\sup_{\theta \in \Theta} \left| \widehat{Q}^{(\pi)}_T(\theta) - \Bar{Q}^{(\pi)}_T(\theta) \right| \overset{p}{\rightarrow} 0 \implies \widehat{\theta}_T \overset{p}{\rightarrow} \theta^*$.

\paragraph{Uniform convergence of $\widehat{Q}^{(\pi)}_T(\theta)$.} We now prove that $\sup_{\theta \in \Theta} \left| \widehat{Q}^{(\pi)}_T(\theta) - \Bar{Q}^{(\pi)}_T(\theta) \right| \overset{p}{\rightarrow} 0$. Following the proof of \citet[Theorem~2.6]{newey1994large}, we have
\begin{align}
    & \left| \widehat{Q}^{(\pi)}_T(\theta) - \Bar{Q}^{(\pi)}_T(\theta) \right| \nonumber \\
    \,\, & \leq \begin{aligned}[t]
    & \left\| \frac{1}{T} \sum_{t=1}^{T} \left[ g_t(\theta) - m(s_t) \otimes g_{*}(\theta) \right] \right\|^2 \left\| \widehat{W} \right\|^2 + 2 \left\|g_{*}(\theta) \right\| \left\| \frac{1}{T} \sum_{t=1}^{T} \left[ g_t(\theta) - m(s_t) \otimes g_{*}(\theta) \right] \right\| \left\| \widehat{W} \right\| + \\
    & \left\| g_{*}(\theta) \right\|^2 \left\| \widehat{W} - W \right\|.
    \end{aligned} \label{eq:apdx-Q-hat-diff-expansion}
\end{align}

We first prove that $\left\|\widehat{W} - W \right\| \ConvProb 0$. Due to Condition~(iii) of the theorem, we can apply Lemma~\ref{lemma:apdx-uniform-convergence-dependent-data} to get
\begin{align*}
    \forall \,\, (i, j) \in [M]^2, \,\, \forall \, \epsilon > 0, \,\, & \P\left( \sup_{\theta \in \Theta} \left| \widehat{\Omega}_T(\theta)_{i, j} - \left[m_{\Omega}(\kappa_T) \otimes \Omega(\theta) \right] \right| > \epsilon \right) \rightarrow 0, \\
    \therefore \,\, \forall \,\, (i, j) \in [M]^2, \,\, \forall \, \epsilon > 0, \,\, & \P\left( \left| \widehat{\Omega}_T(\widehat{\theta}^{(\text{os})}_T)_{i, j} - \left[ m_{\Omega}(\kappa_T) \otimes \Omega(\widehat{\theta}^{(\text{os})}_T) \right] \right| > \epsilon \right) \rightarrow 0, \\
    \therefore \,\, \forall \,\, (i, j) \in [M]^2, \,\, \forall \, \epsilon > 0, \,\, & \P\left( \left| \widehat{\Omega}_T(\widehat{\theta}^{(\text{os})}_T)_{i, j} - \left[ m_{\Omega}(\kappa_T) \otimes \Omega(\theta^*) \right] \right| > \epsilon \right) \xrightarrow[]{(a)} 0, \\
    \therefore \,\, \forall \,\, (i, j) \in [M]^2, \,\, \forall \, \epsilon > 0, \,\, & \P\left( \left| \widehat{W}_{i,j} - W_{i,j} \right| > \epsilon \right) \xrightarrow[]{(b)} 0, \\
    \therefore  & \left\| \widehat{W} - W \right\| \ConvProb 0,
\end{align*}
where (a) follows because $\widehat{\theta}^{(\text{os})}_T \ConvProb \theta^*$ (by Proposition~\ref{prop:consistency}) and (b) by the continuous mapping theorem.
Therefore, we have
\begin{align*}
    \left\| \widehat{W} \right\| &\leq \left\| W \right\| + o_p(1) \\
        &\leq \underbrace{\limsup_{T \to \infty} \left\| \left[ m_{\Omega}(\kappa_T) \otimes \Omega(\theta^*) \right]^{-1} \right\|}_{:= \lambda_0} + o_p(1).
\end{align*}
Substituting these results in Eq.~\ref{eq:apdx-Q-hat-diff-expansion}, we get
\begin{align*}
    \left| \widehat{Q}^{(\pi)}_T(\theta) - \Bar{Q}^{(\pi)}_T(\theta) \right| \leq
    \left\| \frac{1}{T} \sum_{t=1}^{T} \left[ g_t(\theta) - m(s_t) \otimes g_{*}(\theta) \right] \right\|^2 \lambda_0^2 + 2 \left\|g_{*}(\theta) \right\| \left\| \frac{1}{T} \sum_{t=1}^{T} \left[ g_t(\theta) - m(s_t) \otimes g_{*}(\theta) \right] \right\| \lambda_0 + o_p(1).
\end{align*}
Thus, to show uniform convergence of $\widehat{Q}^{(\pi)}_T(\theta)$, we need to show that $\sup_{\theta \in \Theta} \left\| \frac{1}{T} \sum_{t=1}^{T} \left[ g_t(\theta) - m(s_t) \otimes g_{*}(\theta) \right] \right\| \ConvProb 0$. For any $\epsilon > 0$, we have
\begin{align*}
    \P\left( \sup_{\theta \in \Theta} \left\| \frac{1}{T} \sum_{t=1}^{T} \left[ g_t(\theta) - m(s_t) \otimes g_{*}(\theta) \right] \right\| < \epsilon \right) &\geq  \P\left( \sup_{\theta \in \Theta} \sum_{j=1}^M \left| \frac{1}{T} \sum_{t=1}^{T} \left[ g_{t,j}(\theta) - m_j(s_t) g_{*}(\theta)_j \right] \right| < \epsilon \right) \\
    &\overset{(a)}{\geq}  1 - \sum_{j=1}^M \P\left( \sup_{\theta \in \Theta} \left| \frac{1}{T} \sum_{t=1}^{T} \left[ g_{t,j}(\theta) - m_j(s_t) g_{*}(\theta)_j \right] \right| \geq \frac{\epsilon}{M} \right) \\
    &\overset{(b)}{\geq} 1 - o_p(1), \\
    \therefore \,\, \sup_{\theta \in \Theta} \left\| \frac{1}{T} \sum_{t=1}^{T} \left[ g_t(\theta) - m(s_t) \otimes g_{*}(\theta) \right] \right\| &\ConvProb 0,
\end{align*}
where (a) follows by the union bound and (b) by applying Lemma~\ref{lemma:apdx-uniform-convergence-dependent-data} for every $j \in [M]$ (using Condition~(ii)).
\end{proof}

\subsection{Proof of Proposition~\ref{prop:asymptotic-normality} (Asymptotic normality)}\label{sec:appendix-asymptotic-normality-proof}

\begin{proposition}[Martingale CLT {\citep[Corollary~3.1]{hall1980martingale}}]\label{prop:apdx-martingale-clt}
Let $M_i$ with $1 \leq i \leq n$ be a martingale adapted to the filtration $\mathcal{F}_i$ with differences $X_i = M_i - M_{i-1}$ and $M_0 = 0$. Suppose that the following two conditions hold: (i) (Conditional Lindeberg) $\forall \epsilon > 0$, \, $\sum_{i=1}^n \E\left[ X^2_i I\left( |X_i| > \epsilon \right) | \mathcal{F}_{i-1} \right] \ConvProb 0$, and (ii) (Convergence of conditional variance) For some constant $\sigma > 0$, $\sum_{i=1}^n \E\left[ X^2_i | \mathcal{F}_{i-1} \right] \ConvProb \sigma^2$. Then
$\sum_{i=1}^n X_i \xrightarrow{d} \mathcal{N}(0, \sigma^2)$.
\end{proposition}

\begin{proposition*}[Asymptotic normality]
Suppose that
(i) $\widehat{\theta}^{(\pi)}_T \xrightarrow{p} \theta^*$;
(ii) $\forall (i, j) \in [M] \times [D], \, \left[ \frac{\partial \Tilde{g}_t}{\partial \theta} (\theta) \right]_{i, j}$ satisfies Property~\ref{property:mod-of-continuity};
(iii) $\exists \delta > 0 \, \text{such that} \, \E\left[ \left\| \Tilde{g}_i(\theta^*) \right\|^{2 + \delta} \right] < \infty$, and
(iv) (Selection ratio convergence) $\kappa^{(\pi)}_T \overset{p}{\rightarrow} k$ for some constant $k \in \ChoiceSimplex$. Then
$\widehat{\theta}_T$ is asymptotically normal:
\begin{align*}
    \sqrt{T} (\widehat{\theta}^{(\pi)}_T - \theta^*) &\overset{d}{\rightarrow} \mathcal{N}\left( 0, \Sigma(\theta^*, k) \right),
\end{align*}
where $\Sigma(\theta^*, k)$ is a constant matrix that depends only on $\theta^*$ and $k$.
By Assumption~\ref{assump:standard-gmm}(e) and the Delta method, $\widehat{\beta}_T$ is asymptotically normal:
\begin{align*}
    \sqrt{T} (\widehat{\beta}_T - \beta^*) &\overset{d}{\rightarrow} \mathcal{N}\left( 0, V(\theta^*, k) \right), \, \text{where} \, V(\theta^*, k) = \nabla_{\theta} f_{\text{tar}}(\theta^*)^\top [\Sigma(\theta^*, k)] \nabla_{\theta} f_{\text{tar}}(\theta^*).
\end{align*}
\end{proposition*}
\begin{proof}
We follow a standard GMM asymptotic normality proof (e.g. \citet[Theorem~3.4]{newey1994large}) and modify it to work for dependent data. Applying the GMM first-order condition to the two-step GMM estimator, we get
\begin{align*}
    \sqrt{T} (\widehat{\theta}_T - \theta^*) = \left[ \widehat{G}^\top(\widehat{\theta}_T) \widehat{\Omega}(\widehat{\theta}^{(\text{os})}_T)^{-1} \widehat{G}(\Tilde{\theta})  \right]^{-1} \widehat{G}^\top(\widehat{\theta}_T) \widehat{\Omega}(\widehat{\theta}^{(\text{os})}_T)^{-1} \frac{1}{\sqrt{T}} \sum_{t=1}^{T} g_i(\theta^*),
\end{align*}
where $\widehat{\theta}^{(\text{os})}_T$ is the one-step GMM estimator, $\Tilde{\theta}$ is a point on the line-segment joining $\widehat{\theta}_T$ and $\theta^*$,
\begin{align*}
    \widehat{G}(\theta) &= \frac{1}{T} \sum_{t=1}^T \frac{\partial g_t(\theta)}{\partial \theta} \\
        &= \frac{1}{T} \sum_{t=1}^T \frac{\partial m(s_t) \otimes \Tilde{g}_t(\theta)}{\partial \theta} \\
        &= \frac{1}{T} \sum_{t=1}^T \left( \underbrace{\left[ m(s_t), m(s_t), \hdots, m(s_t) \right]}_{D \, \text{times}}  \otimes \left[\frac{\partial \Tilde{g}_t(\theta)}{\partial \theta} \right] \right), \\
        &= \frac{1}{T} \sum_{t=1}^T \left( m_G(s_t)  \otimes \left[\frac{\partial \Tilde{g}_t(\theta)}{\partial \theta} \right] \right), \, \text{and} \\
    \widehat{\Omega}(\theta) &= \frac{1}{T} \sum_{t=1}^T \left[ g_t(\theta) g_t(\theta)^\top \right] \\
        &= \frac{1}{T} \sum_{t=1}^T \left( \left[ m(s_t) m(s_t)^\top \right] \otimes \left[ \Tilde{g}_t(\theta) \Tilde{g}_t(\theta)^\top \right] \right), \\
        &= \frac{1}{T} \sum_{t=1}^T \left( m_{\Omega}(s_t) \otimes \left[ \Tilde{g}_t(\theta) \Tilde{g}_t(\theta)^\top \right] \right),
\end{align*}
where $m_G(s_t) = \underbrace{\left[ m(s_t), m(s_t), \hdots, m(s_t) \right]}_{D \, \text{times}}$ is a $M \times D$ matrix and $m_{\Omega}(s_t) = m(s_t) m(s_t)^\top$.

\paragraph{Convergence of $\widehat{G}(\widehat{\theta}_T)$.}

Let $G(\theta) = \E\left[\frac{\partial \Tilde{g}_t(\theta)}{\partial \theta} \right]$. Applying Lemma~\ref{lemma:apdx-uniform-convergence-dependent-data} to every element of $\widehat{G}$ (using Condition~(ii)) and using the union bound, we get
\begin{align}
    \sup_{\theta \in \Theta} & \left\| \widehat{G}(\theta) - \left(\frac{1}{T} \sum_{t=1}^T m_G(s_t)\right) \otimes G(\theta)  \right\| \ConvProb 0, \nonumber \\
    \therefore\,\, \forall \epsilon > 0, \,\, & \P\left( \left\| \widehat{G}(\widehat{\theta}_T) - \left(\frac{1}{T} \sum_{t=1}^T m_G(s_t)\right) \otimes G(\widehat{\theta}_T)  \right\| > \epsilon \right) \rightarrow 0. \label{eq:apdx-conv-prob-sup-G-prob-notation}
\end{align}
Since $\kappa_T \ConvProb k$ for some constant $k$ (by Condition~(iv)), $\left(\frac{1}{T} \sum_{t=1}^T m_G(s_t)\right)$ also converges in probability to a constant matrix.
That is, $\frac{1}{T} \sum_{t=1}^T m_G(s_t) \ConvProb m^{*}_G(k)$ for some constant matrix $m^{*}_G(k)$ that only depends on $k$.
By the continuity of $G$ and the fact that $\widehat{\theta}_T \ConvProb \theta^*$ (by Condition~(i)), we have $G(\widehat{\theta}_T) \ConvProb G(\theta^*)$. Using these results with Eq.~\ref{eq:apdx-conv-prob-sup-G-prob-notation}, we get
\begin{align}
    \widehat{G}(\widehat{\theta}_T) &\ConvProb m^{*}_G(k) \otimes G(\theta) \nonumber \\
        &= G_{*}(\theta^*, k), \label{eq:apdx-convergence-of-G-theta-hat} \\
    \text{Similarly,} \,\, \widehat{G}(\Tilde{\theta}) &\xrightarrow[(a)]{p} G_{*}(\theta^*, k), \label{eq:apdx-convergence-of-G-theta-tilde}
\end{align}
where $G_{*}(\theta^*, k) = m^{*}_G(k) \otimes G(\theta^*)$ and (a) follows because $\Tilde{\theta} \ConvProb \theta^*$.

\paragraph{Convergence of the weight matrix $\widehat{W}$.}
Let $\Omega(\theta) = \E\left[ \Tilde{g}_t(\theta) \Tilde{g}_t(\theta)^\top \right]$.
By applying Lemma~\ref{lemma:apdx-uniform-convergence-dependent-data} to every element of $\widehat{\Omega}$ (using Condition~(iii)) and the union bound, we get
\begin{align}
    \sup_{\theta \in \Theta} & \left\| \widehat{\Omega}(\theta) - \left(\frac{1}{T} \sum_{t=1}^T m_\Omega(s_t)\right) \otimes \Omega(\theta) \right\| \ConvProb 0, \nonumber \\
    \therefore \,\, \forall \epsilon > 0, & \P\left( \left\| \widehat{\Omega}(\widehat{\theta}^{(\text{os})}_T) - \left(\frac{1}{T} \sum_{t=1}^T m_\Omega(s_t)\right) \otimes \Omega(\widehat{\theta}^{(\text{os})}_T) \right\| > \epsilon \right) \rightarrow 0. \label{eq:apdx-omega-conv-prob-prob-notation}
\end{align}
Since $\kappa_T \ConvProb k$ for some constant k (by Condition~(iv)), $\left(\frac{1}{T} \sum_{t=1}^T m_\Omega(s_t)\right) \ConvProb m^{*}_\Omega(k)$ for some constant matrix $m^{*}_\Omega(k)$ that only depends on $k$.
By continuity of $\Omega$ and the fact that $\widehat{\theta}^{(\text{os})}_T \ConvProb \theta^*$ (which follows by Proposition~\ref{prop:consistency}), we have $\Omega(\widehat{\theta}^{(\text{os})}_T) \ConvProb \Omega(\theta^*)$. Using these results with Eq.~\ref{eq:apdx-omega-conv-prob-prob-notation}, we get
\begin{align}
    \widehat{\Omega}(\widehat{\theta}^{(\text{os})}_T) &\ConvProb m^{*}_\Omega(k) \otimes \Omega(\theta^*) \nonumber \\
        &= \Omega_{*}(\theta^*, k), \nonumber \\
    \therefore\,\, \widehat{W} = \widehat{\Omega}(\widehat{\theta}^{(\text{os})}_T)^{-1} &\ConvProb \Omega_{*}(\theta^*, k)^{-1}, \label{eq:apdx-convergence-of-Omega}
\end{align}
where $\Omega_{*}(\theta^*, k) = m^{*}_\Omega(k) \otimes \Omega(\theta^*)$.

\paragraph{Asymptotic normality of $\frac{1}{\sqrt{T}} \sum_{t=1}^{T} g_i(\theta^*)$.}
For this part, we use the Cramer-Wold theorem and the martingale CLT in Proposition~\ref{prop:apdx-martingale-clt}. For any $v \in \R^M$ s.t. $\left\|v\right\| = 1$, $\frac{v^\top g_i(\theta^*)}{\sqrt{T}}$ is a MDS because $\E\left[ v^\top g_i(\theta^*) | H_{i-1} \right] = v^\top \E\left[ g_i(\theta^*) | H_{i-1} \right] = 0$.
We now show that the two conditions of Proposition~\ref{prop:apdx-martingale-clt} apply to this MDS.

\textit{(i) Conditional Lindeberg:} The Lyapunov condition implies the Lindeberg condition \citep[pg.~6]{bell2015lindeberg}. In our case, the Lyapunov condition is easier to check and we show that it holds. For some $\delta > 0$, we have
\begin{align*}
    \frac{1}{T^{1+\delta/2}} \sum_{i=1}^T \left| v^\top g_i(\theta^*) \right|^{2+\delta} &\overset{(a)}{\leq} \frac{1}{T^{1+\delta/2}} \sum_{i=1}^T \left\| v \right\|^{2+\delta} \left\| g_i(\theta^*) \right\|^{2+\delta} \\
    &\overset{(b)}{=} \frac{1}{T^{1+\delta/2}} \sum_{i=1}^T \left\| g_i(\theta^*) \right\|^{2+\delta} \\
\therefore\,\, \frac{1}{T^{1+\delta/2}} \sum_{i=1}^T \E\left[ \left| v^\top g_i(\theta^*) \right|^{2+\delta} \big| H_{i-1} \right] &\leq \frac{1}{T^{1+\delta/2}} \sum_{i=1}^T \E\left[ \left\| g_i(\theta^*) \right\|^{2+\delta} \big| H_{i-1} \right] \\
    &= \frac{1}{T^{1+\delta/2}} \sum_{i=1}^T \E\left[ \left\| m(s_i) \otimes \Tilde{g}_i(\theta^*) \right\|^{2+\delta} \right] \\
    &\overset{(c)}{\leq} \frac{1}{T^{1+\delta/2}} \sum_{i=1}^T \E\left[ \left\| \Tilde{g}_i(\theta^*) \right\|^{2+\delta} \right] \\
    &\overset{(d)}{\rightarrow} 0,
\end{align*}
where (a) follows by Cauchy-Schwarz, (b) because $\|v\|=1$, (c) because $m(s_i)$ is a binary vector, and (d) because $\E\left[ \left\| \Tilde{g}_i(\theta^*) \right\|^{2+\delta} \right] < \infty$ (by Condition (iii)). 

\textit{(ii) Convergence of conditional variance:} The conditional variance can be written as
\begin{align*}
    \frac{1}{T} \sum_{t=1}^T \E\left[ v^\top g_t(\theta^*) g_t(\theta^*)^\top v \big| H_{i-1} \right] &= \frac{1}{T} \sum_{t=1}^T v^\top \E\left[ g_t(\theta^*) g_t(\theta^*)^\top \big| H_{i-1} \right] v \\
    &= v^\top \left[ \left( \frac{1}{T} \sum_{t=1}^T m(s_t) m(s_t)^\top\right) \otimes \Omega(\theta^*) \right] v \\
    &\xrightarrow[p]{(a)} v^\top \left[ m^{*}_{\Omega}(k) \otimes \Omega(\theta^*) \right] v \\
    &= v^\top \left[ \Omega_{*}(\theta^*, k) \right] v,
\end{align*}
where (a) holds because $\kappa_T \ConvProb k$ (by Condition~(iv)).
Thus, using Proposition~\ref{prop:apdx-martingale-clt}, $\forall v \in \R^M$ s.t. $\|v\|=1$, we have
\begin{align*}
    \frac{1}{\sqrt{T}} \sum_{i=1}^T v^\top g_i(\theta^*) v &\xrightarrow{d} \mathcal{N}\left(0, v^\top \Omega_{*}(\theta^*, k) v \right).
\end{align*}    
Thus, by the Cramer-Wold theorem, we get
\begin{align}
    \frac{1}{\sqrt{T}} \sum_{i=1}^T g_i(\theta^*) &\xrightarrow{d} \mathcal{N}\left(0, \Omega_{*}(\theta^*, k) \right). \label{eq:apdx-convergence-in-dist-sum-g}
\end{align}

\paragraph{Asymptotic normality of $\widehat{\theta}_T$}
By Eqs.~\ref{eq:apdx-convergence-of-G-theta-hat}, \ref{eq:apdx-convergence-of-G-theta-tilde}, \ref{eq:apdx-convergence-of-Omega}, and \ref{eq:apdx-convergence-in-dist-sum-g}, and Slutsky's theorem, we get
\begin{align*}
    \sqrt{T} (\widehat{\theta}_T - \theta^*) &\overset{d}{\rightarrow} \mathcal{N}\left( 0, \Sigma(\theta^*, k) \right), \\
    \text{where} \,\, \Sigma(\theta^*, k) &= \left[ G^\top_{*}(\theta^*, k) \left(\Omega_{*}(\theta^*, k)^{-1}\right) G_{*}(\theta^*, k) \right]^{-1}.
\end{align*}
\end{proof}

\subsection{Proof of Theorem~\ref{thm:etc-regret} (Regret of OMS-ETC)}\label{sec:apdx-proof-etc-regret}

\begin{lemma}[Consistency of
$\widehat{k}_t$]\label{lemma:apdx-kappa-hat-consistency}
Suppose that Assumption~\ref{assump:kappa-star-identify} holds.
If $\widehat{\theta}_t \ConvProb \theta^*$, then $\widehat{k}_t \ConvProb \kappa^*$ where $\widehat{k}_t = \arg\min_{\kappa \in \ChoiceSimplex} V(\widehat{\theta}_{t}, \kappa)$.
\end{lemma}
\begin{proof}
By continuity of $V$, compactness of $\ChoiceSimplex$, and Assumption~\ref{assump:kappa-star-identify}, $\widehat{k}_t \overset{p}{\rightarrow} \arg\min_{\kappa \in \ChoiceSimplex} V(\theta^*, \kappa) = \kappa^*$.
\end{proof}

\begin{theorem*}[Regret of OMS-ETC]
Suppose that (i) Conditions (i)-(iii) of Proposition~\ref{prop:asymptotic-normality} hold
and
(ii) Assumption~\ref{assump:kappa-star-identify} holds.
Case (a): For a fixed $e \in (0, 1)$,
if
$\kappa^* \in \Tilde{\Delta}$, then the regret converges to zero: $R_\infty(\pi_{\text{ETC}}) = 0$.
If $\kappa^* \notin \Tilde{\Delta}$, then $\pi_{\text{ETC}}$ suffers constant regret: $R_\infty(\pi_{\text{ETC}}) = r$ for some constant $r > 0$.
Case (b): If $e$ depends on $T$ such that $e = o(1)$ 
and $Te \rightarrow \infty$ as $T \to \infty$ (e.g. $e = \frac{1}{\sqrt{T}}$),
then $\forall \theta^* \in \Theta$,
we have $R_\infty(\pi_{\text{ETC}}) = 0$.
\end{theorem*}
\begin{proof}
We first analyze Case~(a) of the theorem where $e$ is fixed. By Condition~(i), $\widehat{\theta}_{Te} \overset{p}{\rightarrow} \theta^*$. We have $\widehat{k} = \arg\min_{\kappa \in \ChoiceSimplex} V(\widehat{\theta}_{Te}, \kappa)$
and therefore
$\widehat{k} \overset{p}{\rightarrow} \kappa^*$ (by Lemma~\ref{lemma:apdx-kappa-hat-consistency}).
Thus, if $\kappa^* \in \Tilde{\Delta}$, then $\kappa_T \ConvProb \widehat{k}$ and therefore $\kappa_T \ConvProb \kappa^*$. Using Proposition~\ref{prop:asymptotic-normality}, we get
\begin{align*}
    & \sqrt{T} \left( \widehat{\beta}_T - \beta^* \right) \overset{d}{\rightarrow} \mathcal{N}\left(0, V(\theta^*, \kappa^*) \right) \\
    \therefore\,\, & R_\infty(\pi_{\text{ETC}}) = V(\theta^*, \kappa^*) - V(\theta^*, \kappa^*) = 0.
\end{align*}
If $\kappa^* \notin \Tilde{\Delta}$, then $\kappa_T \overset{p}{\rightarrow} \Bar{\kappa} \neq \kappa^*$, where $\Bar{\kappa} = \arg\min_{\kappa \in \Tilde{\Delta}} V(\theta^*, \kappa)$. Using Proposition~\ref{prop:asymptotic-normality}, we have
\begin{align*}
    & \sqrt{T} \left( \widehat{\beta}_T - \beta^* \right) \overset{d}{\rightarrow} \mathcal{N}\left(0, V(\theta^*, \Bar{\kappa}) \right) \\
    \therefore\,\, & R_\infty(\pi_{\text{ETC}}) = V(\theta^*, \Bar{\kappa}) - V(\theta^*, \kappa^*) \overset{(a)}{>} 0,
\end{align*}
where (a) follows by Condition~(ii).

Now we analyze part (b) of the theorem. When $e$ depends on $T$ such that $e = o(1)$, the feasible region converges to the entire simplex: $\Tilde{\Delta} \to \ChoiceSimplex$ as $T \to \infty$. Thus $\kappa_T - \widehat{k} \ConvProb 0$. Furthermore, since $Te \rightarrow \infty$ as $T \rightarrow \infty$, we have $\widehat{k} \ConvProb \kappa^*$ and therefore $\kappa_T \ConvProb \kappa^*$.
Using Proposition~\ref{prop:asymptotic-normality}, we get the desired result.
\end{proof}

\subsection{Proof of Lemma~\ref{lemma:gmm-conc-inequality} (GMM concentration inequality)}\label{sec:apdx-proof-lemma-concentration}

\begin{proposition}[MDS concentration inequality {\citep[Theorem~2.19]{wainwright2019high}}]\label{prop:apdx-mds-concentration}
Let $\{ ( D_k, \mathcal{F}_k ) \}_{k=1}^{\infty}$ be a MDS, and suppose that $\E\left[ \exp\left\{ \lambda D_k \right\} | \mathcal{F}_{k-1} \right] \leq \exp\left\{ \frac{ \lambda^2 \nu^2}{2} \right\}$ almost surely for any $\lambda < \frac{1}{\alpha}$. Then the sum satisfies the concentration inequality
\begin{align*}
    \P\left( \left| \frac{1}{n} \sum_{k=1}^{n} D_k \right| > \eta \right) \leq 2 \exp\left\{ -\frac{n \eta^2}{2 \nu^2} \right\} \,\, \text{if} \,\, 0 \leq \eta < \frac{\nu^2}{\alpha}. 
\end{align*}
\end{proposition}

\begin{lemma}[Uniform law for dependent data]\label{lemma:apdx-uniform-law-dependent-data}
Let $a_i(\theta) := S_i \Tilde{a}(\theta; X_i)$, where $a_i$ is a real-valued function, $S_i \in \{0, 1\}$ is $H_{i-1}$-measurable, and $X_i \overset{\text{iid}}{\sim} \P_{\theta^*}$.
Let $\Tilde{a}_{*}(\theta) = \E\left[ \Tilde{a}(\theta; X_i) \right]$.
Suppose that $\Tilde{a}(\theta)$ satisfies Property~\ref{property:concentration}.
Note that $\E\left[ a_i(\theta) | H_{i-1} \right] = S_i \Tilde{a}_*(\theta)$. Then, for some constant $\delta_0 > 0$ and $\forall \delta \in (0, \delta_0)$, 
\begin{align*}
    \P\left( \sup_{\theta \in \Theta} \left| \frac{1}{T} \sum_{i=1}^{T} \left[ a_i(\theta) - S_i \Tilde{a}_*(\theta) \right] \right| > \delta \right) < \frac{1}{\delta^D} \exp\left\{ -\OM \left( T \delta^2 \right) \right\}.
\end{align*}
\end{lemma}
\begin{proof}
Let $U = \{\theta_1, \theta_2, \hdots, \theta_N\}$ be a minimal $\delta$-cover of $\Theta$. We have $N \leq \frac{C}{\delta^D}$ for some constant $C$.
Let $q : \Theta \rightarrow U$
be a function that
returns the closest point from the cover: $q(\theta) = \arg\min_{\theta' \in U} \| \theta - \theta' \|$. We have
\begin{align*}
    & \sup_{\theta \in \Theta} \left| \frac{1}{T} \sum_{i=1}^{T} \left[ a_i(\theta) - S_i \Tilde{a}_{*}(\theta) \right] \right| \\ 
    \,\,&= \sup_{\theta \in \Theta} \left| \frac{1}{T} \sum_{i=1}^{T} \left[ a_i(\theta) - a_i(q(\theta)) + a_i(q(\theta)) - S_i \Tilde{a}_{*}(q(\theta)) + S_i \Tilde{a}_{*}(q(\theta)) - S_i \Tilde{a}(\theta) \right] \right| \\
    &\leq \sup_{\theta \in \Theta} \frac{1}{T} \sum_{i=1}^{T} \left| a_i(\theta) - a_i(q(\theta)) \right| + \max_{n \in [N]} \left| \frac{1}{T} \sum_{i=1}^{T} \left[ a_i(\theta_n) - S_i \Tilde{a}_{*}(\theta_n) \right] \right| + \sup_{\theta \in \Theta} \frac{1}{T} \sum_{i=1}^{T} S_i \left| \Tilde{a}_{*}(q(\theta)) - \Tilde{a}_{*}(\theta) \right| \\
    &= \sup_{\theta \in \Theta} \frac{1}{T} \sum_{i=1}^{T} \left| S_i \left( \Tilde{a}_i(\theta, X_i) - \Tilde{a}_i(q(\theta), X_i) \right) \right| + \max_{n \in [N]} \left| \frac{1}{T} \sum_{i=1}^{T} \left[ a_i(\theta_n) - S_i \Tilde{a}_{*}(\theta_n) \right] \right| + \sup_{\theta \in \Theta} \left| \Tilde{a}_{*}(q(\theta)) - \Tilde{a}_{*}(\theta) \right| \\
    &\leq \sup_{\theta \in \Theta} \frac{1}{T} \sum_{i=1}^{T} \left| \Tilde{a}(\theta, X_i) - \Tilde{a}_i(q(\theta), X_i) \right| + \max_{n \in [N]} \left| \frac{1}{T} \sum_{i=1}^{T} \left[ a_i(\theta_n) - S_i \Tilde{a}_{*}(\theta_n) \right] \right| + \sup_{\theta \in \Theta} \left| \Tilde{a}_{*}(q(\theta)) - \Tilde{a}_{*}(\theta) \right|.
\end{align*}
We now examine the three terms on the RHS one at a time.

\paragraph{Third term.}
By Lipschitzness of $\Tilde{a}_{*}$ (Property~\ref{property:concentration}(i)), we have:
\begin{align*}
    \sup_{\theta \in \Theta} \left| \Tilde{a}_{*}(q(\theta)) - \Tilde{a}_{*}(\theta) \right| \leq L_1 \sup_{\theta \in \Theta} \|q(\theta) - \theta\| \leq L_1 \delta.
\end{align*}

\paragraph{Second term.}
We note that it is a sum of a MDS. By Property~\ref{property:concentration}(ii) and Proposition~\ref{prop:apdx-mds-concentration}, there exists a constant $C_1 > 0$ such that for $\delta \in (0, C_1)$, we have
\begin{align*}
    \forall n \in [N], \,\, \P\left( \left| \frac{1}{T} \sum_{i=1}^{T} \left[ a_i(\theta_n) - S_i \Tilde{a}_{*}(\theta_n) \right] \right| < \delta \right) &> 1 - \exp\left\{ - \OM \left( T \delta^2 \right) \right\} \\
    \therefore\,\, \P\left( \max_{n \in [N]} \left| \frac{1}{T} \sum_{i=1}^{T} \left[ a_i(\theta_n) - S_i \Tilde{a}_{*}(\theta_n) \right] \right| < \delta \right)
    &> 1 - \P\left( \bigcup_{n \in [N]} \left| \frac{1}{T} \sum_{i=1}^{T} \left[ a_i(\theta_n) - S_i \Tilde{a}_{*}(\theta_n) \right] \right| > \delta \right) \\
    &> 1 - N \exp\left\{  -\OM \left( T \delta^2 \right) \right\} \\
        &> 1 - \frac{1}{\delta^D} \exp\left\{ - \OM \left( T \delta^2 \right) \right\}.
\end{align*}

\paragraph{First term.}
We have
\begin{align}
    u_{*}(\eta) &= \E\left[ \sup_{\substack{\theta, \theta' \in \Theta; \|\theta - \theta'\| \leq \eta}} \left| \Tilde{a}_i(\theta, X_i) - \Tilde{a}_i(\theta', X_i) \right| \right] \nonumber \\
        &\leq \E\left[ \sup_{\theta \in \Theta} \left\| A(X_i, \theta) \right\| \sup_{\substack{\theta, \theta' \in \Theta; \|\theta - \theta'\| \leq \eta}} \left\| \theta- \theta' \right\| \right] \nonumber \\
        &\leq \eta \sup_{\theta \in \Theta} \left\| A(X_i, \theta) \right\| \nonumber \\
        &\overset{(a)}{\leq} A_0 \eta. \label{eq:apdx-u-i-eta-lipchitz},
\end{align}
where (a) follows by Property~\ref{property:concentration}(iii).

Suppose that Property~\ref{property:concentration}(iv)(a) holds. Then
\begin{align*}
    \sup_{\theta \in \Theta} \frac{1}{T} \sum_{i=1}^{T} \left| \Tilde{a}_i(\theta, X_i) - \Tilde{a}_i(q(\theta), X_i) \right| &\leq \frac{1}{T} \sum_{i=1}^{T} u_i(\delta) \\
    &\leq \frac{1}{T} \sum_{i=1}^{T} u_i(\delta) - u_{*}(\delta) + u_{*}(\delta) \\
    &\overset{(a)}{\leq} \frac{1}{T} \sum_{i=1}^{T} u_i(\delta) - u_{*}(\delta) + A_0 \delta,
\end{align*}
where (a) follows by Eq.~\ref{eq:apdx-u-i-eta-lipchitz}.
By Property~\ref{property:concentration}(iv)(a), $(u_i(\delta) - u_{*}(\delta))$ is sub-Exponential. By the sub-exponential tail bound \citep[Proposition~2.9]{wainwright2019high}, for some constant $C_2 > 0$ and $\delta \in (0,  C_2)$, we have
\begin{align*}
    \P \left( \left| \frac{1}{T} \sum_{i=1}^{T} u_i(\delta) - u_{*}(\delta) \right| < \delta \right) &> 1 - \exp\left\{ - \OM\left(T \delta^2 \right) \right\} \\
    \therefore \,\, \P \left( \sup_{\theta \in \Theta} \frac{1}{T} \sum_{i=1}^{T} \left| \Tilde{a}_i(\theta, X_i) - \Tilde{a}_i(q(\theta), X_i) \right| < (A_0 + 1) \delta \right) &> 1 - \exp\left\{ - \OM\left(T \delta^2 \right) \right\} \\
    \therefore \,\, \P \left( \sup_{\theta \in \Theta} \frac{1}{T} \sum_{i=1}^{T} \left| \Tilde{a}_i(\theta, X_i) - \Tilde{a}_i(q(\theta), X_i) \right| < \delta \right) &> 1 - \exp\left\{ - \OM\left(T \delta^2 \right) \right\}.
\end{align*}
Now suppose that Property~\ref{property:concentration}(iv)(b) holds instead.
Then
\begin{align*}
    \sup_{\theta \in \Theta} \frac{1}{T} \sum_{i=1}^{T} \left| \Tilde{a}_i(\theta, X_i) - \Tilde{a}_i(q(\theta), X_i) \right| &\leq \frac{1}{T} \sum_{i=1}^{T} \sup_{\theta \in \Theta} \left\| A(X_i, \theta) \right\| \sup_{\theta \in \Theta} \left\| \theta - q(\theta) \right\| \\
        &\leq \frac{\delta}{T} \sum_{i=1}^{T} \sup_{\theta \in \Theta} \left\| A(X_i, \theta) \right\|.
\end{align*}
Since $\sup_{\theta \in \Theta} \left\| A(X_i, \theta) \right\|$ is sub-Exponential, so is $\sum_{i=1}^{T} \sup_{\theta \in \Theta} \left\| A(X_i, \theta) \right\|$. By a sub-Exponential tail bound \citep[Proposition~2.7.1(a)]{vershynin2018high}, we have for any $C_3 > 0$,
\begin{align*}
    \P\left( \frac{1}{T} \sum_{i=1}^{T} \sup_{\theta \in \Theta} \left\| A(X_i, \theta) \right\| > C_3 \right) &\leq \exp\left\{ - \OM\left( T C_3 \right) \right\} \\
    \therefore \,\, \P\left( \sup_{\theta \in \Theta} \frac{1}{T} \sum_{i=1}^{T} \left| \Tilde{a}_i(\theta, X_i) - \Tilde{a}_i(q(\theta), X_i) \right| > \delta C_3 \right) &\leq \exp\left\{ - \OM\left( T C_3 \right) \right\} \\
    \therefore \,\, \P\left( \sup_{\theta \in \Theta} \frac{1}{T} \sum_{i=1}^{T} \left| \Tilde{a}_i(\theta, X_i) - \Tilde{a}_i(q(\theta), X_i) \right| > \delta \right) &\leq \exp\left\{ - \OM\left( T \right) \right\}.
\end{align*}

Combining these results together using the union bound, we get

\begin{align*}
    \P & \left( \sup_{\theta \in \Theta} \left| \frac{1}{T} \sum_{i=1}^{T} \left[ a_i(\theta) - S_i \Tilde{a}_{*}(\theta; k) \right] \right| < (L_1 + L_2 + 2) \delta \right) \\
    &> \P\left( \max_{n \in [N]} \left| \frac{1}{T} \sum_{i=1}^{T} \left[ a_i(\theta_n) - S_i \Tilde{a}_{*}(\theta_n) \right] \right| < \delta, \,\,\, \left| \frac{1}{T} \sum_{i=1}^{T} u_i(\delta) - u_{*}(\delta) \right| < \delta \right) \\
    &> 1 - \sum_{n=1}^N \P\left( \left| \frac{1}{T} \sum_{i=1}^{T} \left[ a_i(\theta_n) - S_i \Tilde{a}_{*}(\theta_n) \right] \right| > \delta \right) - \P\left( \left| \frac{1}{T} \sum_{i=1}^{T} u_i(\delta) - u_{*}(\delta) \right| > \delta \right) \\
    &> 1 - \frac{1}{\delta^D} \exp\left\{ - \OM\left( T \delta^2 \right) \right\} \\
    \therefore \,\, \P & \left( \sup_{\theta \in \Theta} \left| \frac{1}{T} \sum_{i=1}^{T} \left[ a_i(\theta) - S_i \Tilde{a}_{*}(\theta; k)  \right] \right| < \delta \right) > 1 - \frac{1}{\delta^D} \exp\left\{ - \OM\left( T \delta^2 \right) \right\}.
\end{align*}
\end{proof}

\begin{proposition}[Boundedness and Property~\ref{property:concentration}(iv)(a)]\label{prop:apdx-boundedness-and-condition-conc}
Property~\ref{property:concentration}(iv)(a) is satisfied for bounded function classes, i.e., when $\|\Tilde{a}_i\|_\infty < A < \infty$.
\end{proposition}
\begin{proof}
We have:
\begin{align*}
    u_i(\eta) &= \sup_{\theta, \theta' \in \Theta, \|\theta-\theta'\| \leq \eta}  |\Tilde{a}(\theta, X_i) - \Tilde{a}(\theta', X_i)| \\
        &\leq 2 \sup_{\theta \in \Theta} |\Tilde{a}_i| \\
        &\leq 2 A.
\end{align*}
Thus $u_i(\eta)$ is bounded and therefore sub-Gaussian for every $\eta$.
\end{proof}

\begin{proposition}[Linearity and Property~\ref{property:concentration}(iv)(b)]\label{prop:apdx-ulln-linearity}
Suppose that
(i) $\Tilde{a}(\theta, X_i)$ is a linear function of $\theta$, i.e., $\Tilde{a}(\theta, X_i) = \theta^T \phi(X_i) + \rho(X_i)$, where $\phi$ and $\rho$ are arbitrary functions; and
(ii) $\forall d \in [D],\,\, \phi(X_i)_d$ is sub-Exponential.
Then $\Tilde{a}(\theta, X_i)$ satisfies Property~\ref{property:concentration}(iv)(b).
\end{proposition}
\begin{proof}
We have that $A(X_i, \theta) = \frac{\partial \Tilde{a}(X_i; \theta)}{\partial \theta} = \phi(X_i)$ and thus $\sup_{\theta \in \Theta} \| A(X_i, \theta) \| = \| \phi(X_i) \| \leq \sum_{d =1}^{D} |\phi(X_i)_d|$.
Therefore, for any $\eta > 0$, we have
\begin{align*}
    \P\left( \sup_{\theta \in \Theta} \| A(X_i, \theta) \| < \eta \right) &= \P\left( \| \phi(X_i) \| < \eta \right) \\
        &\geq \P\left( \sum_{d =1}^{D} |\phi(X_i)_d| < \eta \right) \\
        &\geq \P\left( \forall \, d \in [D], \,\, |\phi(X_i)_d| < \frac{\eta}{D} \right) \\
        &\overset{(a)}{\geq} 1 - \sum_{d=1}^D \P\left( |\phi(X_i)_d| > \frac{\eta}{D} \right) \\
        &\overset{(b)}{\geq} 1 - \sum_{d=1}^D \exp\left\{ -\OM(\eta) \right\} \\
        &\geq 1 - \exp\left\{ -\OM(\eta) \right\},
\end{align*}
where (a) follows by the union bound and (b) because $\phi(X_i)_d$ is sub-Exponential.
This shows that $\\ \sup_{\theta \in \Theta} \| A(X_i, \theta) \|$ is also sub-Exponential (see \citet[Definition~2.7.5]{vershynin2018high}).
\end{proof}

\begin{remark*}
\citet{rakhlin2015sequential} derive a uniform martingale LLN and develop sequential analogues of classical complexity measures used in empirical process theory. These techniques are a potential alternative for deriving the tail bound in Lemma~\ref{lemma:apdx-uniform-law-dependent-data}.
However, the conditions required for these techniques are difficult to check.
In our case, the dependent and i.i.d. components can be separated more easily.
Thus we opted for deriving a uniform concentration bound by modifying the classical uniform LLN proof.
\citet{zhan2021policy} also derive a uniform LLN without requiring boundedness of the martingale difference terms, but with structural assumptions on the summands related to their specific application.
\end{remark*}

\begin{lemma*}[GMM concentration inequality]
Let $\lambda_{*}, C_0, \eta_1, \eta_2$, and $\delta_0$ be some positive constants.
Suppose that (i) Assumption~\ref{assump:standard-gmm} holds;
(ii) $\forall j, \, \Tilde{g}_{i, j}(\theta)$ satisfies Property~\ref{property:concentration};
(iii) The spectral norm of the GMM weight matrix $\widehat{W}$ is upper bounded with high probability: $\forall \delta \in \left(0, C_0 \right), \,\, \P \left( \|\widehat{W} \| \leq \lambda_* \right) \geq 1 - \frac{1}{\delta^{D}} \exp\left\{ - \OM\left( T \delta^2 \right) \right\}$ (see Remark~\ref{remark:weight-matrix-concentration-condition});
(iv) (Local strict convexity) 
$ \forall \theta \in N_{\eta_1}(\theta^*), \,\, \P \left( \left\| \frac{\partial^2 \Bar{Q}}{\partial \theta^2} (\theta)^{-1} \right\| \leq h \right) = 1$
($\Bar{Q}(\theta)$ is defined in Assumption~\ref{assump:standard-gmm}(a));
(v) (Strict minimization)
$\forall \theta \in N_{\eta_2}(\theta^*)$, there is a unique minimizer $\kappa(\theta) = \arg\min_{\kappa} V(\theta, \kappa)$ s.t. $V(\theta, \kappa) - V(\theta, \kappa(\theta)) \leq c \delta^2 \implies \|\kappa - \kappa(\theta)\| \leq \delta$; and
(vi) $\sup_{\kappa} |V(\theta, \kappa) - V(\theta', \kappa)| \leq L \| \theta - \theta' \|$.
Then, for $\widehat{k}_t = \arg\min_{\kappa \in \ChoiceSimplex} V(\widehat{\theta}^{(\pi)}_T, \kappa)$,
any policy $\pi$, and $\forall \delta \in (0, \delta_0)$,
\begin{align*}
    \P\left( \left\| \widehat{\theta}^{(\pi)}_T - \theta^* \right\| > \delta \right) < \frac{1}{\delta^{2D}} \exp\left\{ - \OM\left( T \delta^4 \right) \right\} \,\, \text{and} \,\, \P\left( \left\| \widehat{k}_T - \kappa^* \right\| > \delta \right) < \frac{1}{\delta^{4D}} \exp\left\{ - \OM \left( T \delta^8 \right) \right\}.
\end{align*}
\end{lemma*}
\begin{proof}

Below we give the empirical and population analogues of the GMM objective for a given policy $\pi$:
\begin{align*}
    \text{Empirical objective:} \,\,\, \widehat{Q}^{(\pi)}_T(\theta) &= \left[ \frac{1}{T} \sum_{t=1}^{T} g_t(\theta) \right]^\top \widehat{W} \left[ \frac{1}{T} \sum_{t=1}^{T} g_t(\theta) \right], \\
    \text{Population objective:} \,\,\, \Bar{Q}^{(\pi)}_T(\theta) &= g^{*}_T(\theta) \widehat{W} g^{*}_T(\theta)^\top \\
    &= \left[ \frac{1}{T} \sum_{t=1}^{T} \E \left[ g_t(\theta) | H_{t-1} \right] \right]^\top \widehat{W} \left[ \frac{1}{T} \sum_{t=1}^{T} \E \left[ g_t(\theta) | H_{t-1} \right] \right] \\
        &= \left[ \left( \frac{1}{T} \sum_{t=1}^{T} m(s_t) \right) \otimes \Tilde{g}_{*}(\theta)\right]^\top \widehat{W} \left[ \left( \frac{1}{T} \sum_{t=1}^{T} m(s_t) \right) \otimes \Tilde{g}_{*}(\theta)\right],
\end{align*}
where $\Tilde{g}_{*}(\theta) = \E\left[ \Tilde{g}_t(\theta) \right]$.

To simplify notation, let $m_t = m(s_t)$. By the triangle and Cauchy-Shwartz inequalities (see \citet[Theorem~2.6]{newey1994large}),
\begin{align*}
    \Big| &  \widehat{Q}^{(\pi)}_T(\theta) - \Bar{Q}^{(\pi)}_T(\theta) \Big| \\
    &\leq \left\| \frac{1}{T} \sum_{t=1}^{T} \left[ g_t(\theta) - m_t \otimes \Tilde{g}_{*}(\theta) \right] \right\|^2 \|\widehat{W}\|^2 + 2 \left\|\Tilde{g}_{*}(\theta)\right\| \left\| \frac{1}{T} \sum_{t=1}^{T} \left[ g_t(\theta) - m_t \otimes \Tilde{g}_{*}(\theta) \right] \right\| \|\widehat{W}\| \\
    &\leq \left\| \frac{1}{T} \sum_{t=1}^{T} \left[ g_t(\theta) - m_t \otimes \Tilde{g}_{*}(\theta) \right] \right\|^2 \|\widehat{W}\|^2 + 2 C \left\| \frac{1}{T} \sum_{t=1}^{T} \left[ g_t(\theta) - m_t \otimes \Tilde{g}_{*}(\theta) \right] \right\| \|\widehat{W}\|,
\end{align*}
where $C = \sup_{\theta \in \Theta} \left\|\Tilde{g}_{*}(\theta) \right\|$.
By applying Lemma~\ref{lemma:apdx-uniform-law-dependent-data} to each element of the vector $g_i(\theta)$ and using the union bound, we get:
\begin{align}
    \P \left( \sup_{\theta \in \Theta} \left\| \frac{1}{T} \sum_{i=1}^{T} \left[ g_t(\theta) - m_t \otimes \Tilde{g}_{*}(\theta) \right] \right\| < \delta \right) &\geq \P \left( \bigcap_{j=1}^{M} \sup_{\theta \in \Theta} \left| \frac{1}{T} \sum_{t=1}^{T} \left[ g_{t,j}(\theta) - m_{t,j} \otimes (\Tilde{g}_{*})_j(\theta) \right] \right| < \frac{\delta}{M} \right) \nonumber \\
    &\geq 1 - \sum_{j=1}^{M} \P \left( \sup_{\theta \in \Theta} \left| \frac{1}{T} \sum_{t=1}^{T} \left[ g_{t,j}(\theta) - m_{t,j} \otimes (\Tilde{g}_{*})_j(\theta) \right] \right| > \frac{\delta}{M} \right) \nonumber \\
    &\geq 1 - \frac{1}{\delta^{2D}} \exp\left\{ - \OM \left( T \delta^2 \right) \right\}. \label{eq:apdx-g-mom-unif-conc}
\end{align}
This means that, for $0 < \delta < 1$,
\begin{align*}
    \left\| \frac{1}{T} \sum_{t=1}^{T} \left[ g_t(\theta) - m_t \otimes \Tilde{g}_{*}(\theta) \right] \right\| \leq \delta, \|\widehat{W}\| \leq \lambda_* &\implies \left| \widehat{Q}^{(\pi)}_T(\theta) - \Bar{Q}^{(\pi)}_T(\theta) \right| \begin{aligned}[t]
        &\leq \lambda_*^2 \delta^2 + 2 \lambda_* C \delta \\
        &= (2C + \lambda_* \delta)\lambda_* \delta \\
        &< (2C + \lambda_*)\lambda_* \delta,
    \end{aligned} \\
    \therefore \,\, \P \left( \sup_{\theta \in \Theta} \left| \widehat{Q}^{(\pi)}_T(\theta) - \Bar{Q}^{(\pi)}_T(\theta) \right| < (2C + \lambda_*)\lambda_* \delta \right)
    &\geq \P\left( \sup_{\theta \in \Theta} \left\| \frac{1}{T} \sum_{t=1}^{T} \left[ g_t(\theta) - m_t \otimes \Tilde{g}_{*}(\theta) \right] \right\| \leq \delta, \|\widehat{W}\| \leq \lambda_* \right) \\
    &\overset{(a)}{\geq} 1 - \P\left( \sup_{\theta \in \Theta} \left\| \frac{1}{T} \sum_{t=1}^{T} \left[ g_t(\theta) - m_t \otimes \Tilde{g}_{*}(\theta) \right] \right\| > \delta \right) - \P\left(\|\widehat{W}\| > \lambda_* \right) \\
    &\overset{(b)}{\geq} 1 - \frac{1}{\delta^{2D}} \exp\left\{ - \OM\left( T \delta^2 \right) \right\} \\
    \therefore \,\, \P\left( \sup_{\theta \in \Theta} \left| \widehat{Q}^{(\pi)}_T(\theta) - \Bar{Q}^{(\pi)}_T(\theta) \right| < \delta \right)
    &\geq 1 - \frac{1}{\delta^{2D}} \exp\left\{ - \OM\left( T \delta^2 \right) \right\},
\end{align*}
where (a) follows by the union bound and (b) follows by Eq.~\ref{eq:apdx-g-mom-unif-conc} and Condition~(iii).
Using this uniform concentration bound, we get
\begin{align*}
    & \P\left( \Bar{Q}^{(\pi)}_T(\widehat{\theta}_T) < \widehat{Q}^{(\pi)}_T(\widehat{\theta}_T) + \frac{\delta}{2} \right) \geq 1 - \frac{1}{\delta^{2D}} \exp\left\{ - \OM\left( T \delta^2 \right) \right\}, \\
    & \P\left( \widehat{Q}^{(\pi)}_T(\theta^*) < \Bar{Q}^{(\pi)}_T(\theta^*) + \frac{\delta}{2} \right) \geq 1 - \frac{1}{\delta^{2D}} \exp\left\{ - \OM\left( T \delta^2 \right) \right\}.
\end{align*}
Since $\widehat{\theta}_T$ minimizes $\widehat{Q}^{(\pi)}_T$ almost surely, we have
$\P\left( \widehat{Q}^{(\pi)}_T(\widehat{\theta}_T) \leq \widehat{Q}^{(\pi)}_T(\theta^*) \right) = 1$. Combining these inequalities using the union bound, we get
\begin{align*}
    \P\left( \Bar{Q}^{(\pi)}_T(\widehat{\theta}_T) < \Bar{Q}^{(\pi)}_T(\theta^*) + \delta \right)
        &\geq 1 - \frac{1}{\delta^{2D}} \exp\left\{ - \OM\left( T \delta^2 \right) \right\} \\
    \therefore \,\, \P\left( \Bar{Q}^{(\pi)}_T(\widehat{\theta}_T) < \delta \right) &\overset{(a)}{\geq} 1 - \frac{1}{\delta^{2D}} \exp\left\{ - \OM\left( T \delta^2 \right) \right\},
\end{align*}
where (a) follows because $\Bar{Q}^{(\pi)}_T(\theta^*) = 0$.

Intuitively, if $\Bar{Q}^{(\pi)}_T(\widehat{\theta}_T)$ is small, then we would expect $\widehat{\theta}_T$ to be close to $\theta^*$. To formally show this, we use the local curvature of $\Bar{Q}^{(\pi)}_T$. 
By Condition~(iv), $\Bar{Q}^{(\pi)}_T$ is locally strictly convex in the $\eta_1$-ball $N_{\eta_1}(\theta^*)$.
Therefore, there exists a closed $\gamma$-ball $N_{\gamma}(\theta^*) \subseteq N_{\eta}(\theta^*)$ such that
\begin{align*}
    \forall \theta \notin N_{\gamma}(\theta^*), \,\, \Bar{Q}^{(\pi)}_T(\theta) > \Bar{Q}_N, \,\, \text{where} \,\, \Bar{Q}_N = \sup_{\theta \in N_{\gamma}(\theta^*)} \Bar{Q}^{(\pi)}_T(\theta).
\end{align*}
This is analogous to an identification condition and ensures that $\Bar{Q}^{(\pi)}_T(\theta) \leq \Bar{Q}_N \implies \theta \in N_{\gamma}(\theta^*)$.

Let $H(\theta) = \frac{\partial^2 \Bar{Q}^{(\pi)}}{\partial \theta^2} (\theta)$. Then, by twice continuous differentiability of $g$, for $\theta \in N_{\gamma}(\theta^*)$, we have
\begin{align*}
    \Bar{Q}^{(\pi)}_T(\theta) &\overset{(a)}{=} \Bar{Q}^{(\pi)}_T(\theta^*) + (\theta - \theta^*) \left[ H(\theta') \right] (\theta - \theta^*)^\top \\
        &\overset{(b)}{=} (\theta - \theta^*) \left[ H(\theta') \right] (\theta - \theta^*)^\top, \\
    \therefore \,\, \|\theta - \theta^*\|^2 &\leq \Bar{Q}^{(\pi)}_T(\theta) \|H^{-1}(\theta')\| \\
        &\overset{(c)}{\leq} \left[ \Bar{Q}^{(\pi)}_T(\theta) \right] h,
\end{align*}
where in (a), $\theta'$ is a point on the line segment joining $\theta$; (b) follows because $\Bar{Q}^{(\pi)}_T(\theta^*) = 0$; and (c) follows by Condition~(iv).
Thus, for $\delta < \Bar{Q}_N$, we have
\begin{align*}
    & \Bar{Q}^{(\pi)}_T(\widehat{\theta}_T) < \delta \implies \|\widehat{\theta}_T - \theta^*\| < \sqrt{\delta h} \\
    \therefore\,\, & \P\left( \|\widehat{\theta}_T - \theta^*\| < \delta \right) \begin{aligned}[t]
        &\geq \P\left( \Bar{Q}^{(\pi)}_T(\widehat{\theta}_T) < \frac{\delta^2}{h} \right) \\
        &\geq 1 - \frac{1}{\delta^{2D}} \exp\left\{ - \OM\left( T \delta^4 \right) \right\}.
    \end{aligned}
\end{align*}

\textbf{Concentration inequality for $\widehat{k}_T$}

By Condition~(vi), $\sup_{\kappa \in \ChoiceSimplex} |V(\widehat{\theta}_T, \kappa) - V(\theta^*, \kappa)| \leq L \| \widehat{\theta}_T - \theta^* \|$. Therefore,
\begin{align*}
    \| \widehat{\theta}_T - \theta^* \| < \delta \implies \sup_{\kappa \in \ChoiceSimplex} |V(\widehat{\theta}_T, \kappa) - V(\theta^*, \kappa)| \leq L \delta.
\end{align*}
Furthermore, we have
\begin{align*}
    \sup_{\kappa \in \ChoiceSimplex} |V(\widehat{\theta}_T, \kappa) - V(\theta^*, \kappa)| \leq L \delta \implies & V(\theta^*, \widehat{k}_T) < V(\widehat{\theta}_T, \widehat{k}_T) + L \delta, \,\, \text{and} \\
    & V(\widehat{\theta}_T, \kappa^*) < V(\theta^*, \kappa^*) + L \delta.
\end{align*}
Since $\widehat{k}_T$ is the minimizer, we have $V(\widehat{\theta}_T, \widehat{k}_T) \leq V(\widehat{\theta}_T, \kappa^*)$.
Combining these inequalities, we get
\begin{align*}
    \| \widehat{\theta}_T - \theta^* \| < \delta \implies V(\theta^*, \widehat{k}_T) - V(\theta^*, \kappa^*) < 2L \delta.
\end{align*}
Due to Condition~(v), we have
\begin{align*}
    V(\theta^*, \widehat{k}_T) - V(\theta^*, \kappa^*) < 2L \delta &\implies \| \widehat{k}_T - \kappa^* \| < \sqrt{ \frac{2 L \delta}{c} }, \\
    \therefore \,\, \| \widehat{\theta}_T - \theta^* \| < \delta &\implies \| \widehat{k}_T - \kappa^* \| < \sqrt{ \frac{2 L \delta}{c} } \\
    \therefore \,\, \P(\| \widehat{k}_T - \kappa^* \| < \delta) &> 1 - \P\left( \|\widehat{\theta}_T - \theta^* \| < \OM\left(\delta^2 \right)  \right) \\
        &> 1 - \frac{1}{\delta^{4D}} \exp\left\{ - \OM\left( T \delta^8 \right) \right\}.
\end{align*}
\end{proof}

\begin{lemma}[Sufficient condition for $\widehat{W}$]\label{lemma:apdx-sufficient-weight-matrix-bounded-condition}
Suppose that $\forall (j, k), \, [\Tilde{g}_{i, j}(\theta) \Tilde{g}_{i, k}(\theta)]$ satisfies Property~\ref{property:concentration}.
Let $\widehat{W}_T(\widehat{\theta}^{(\text{os})}_T) = \widehat{\Omega}_T(\widehat{\theta}^{(\text{os})}_T)^{-1} = \left[ \frac{1}{T} \sum_{t=1}^{T} g_t(\widehat{\theta}^{(\text{os})}_T) g^\top_t(\widehat{\theta}^{(\text{os})}_T) \right]^{-1}$, where $\widehat{\theta}^{(\text{os})}_T$ is the one-step GMM estimator (that uses $\widehat{W} = I$).
Then $\widehat{W}_T(\widehat{\theta}^{(\text{os})}_T)$ satisfies Condition~(iii) of Lemma~\ref{lemma:gmm-conc-inequality}.
\end{lemma}
\begin{proof}
We define $W_T(\theta^*)$ as
\begin{align*}
    W_T(\theta^*) = \Omega_T(\theta^*)^{-1} &= \left[ \frac{1}{T} \sum_{t=1}^T  \E\left[ g_t(\theta^*) g^\top_t(\theta^*) \, \big| \, H_{t-1} \right] \right]^{-1} \\
    &= \left[ \left( \frac{1}{T} \sum_{t=1}^{T} m(s_t) m^\top(s_t) \right) \otimes \E\left[ \Tilde{g}_t(\theta^*) \Tilde{g}^\top_t(\theta^*) \right] \right]^{-1}.
\end{align*}
Let $\Delta = \widehat{\Omega}_T(\widehat{\theta}^{(\text{os})}_T) - \Omega_T(\widehat{\theta}^{(\text{os})}_T)$ and $\lambda_{\min}$ denote smallest eigenvalue.
Using the eigenvalue stability inequality \citep[Section~1.3.3]{tao2012topics}, we get:
\begin{align}
    & \left| \lambda_{\min}\left( \widehat{\Omega}_T(\widehat{\theta}^{(\text{os})}_T) \right) - \lambda_{\min}\left( \Omega_T(\widehat{\theta}^{(\text{os})}_T) \right) \right| \leq \left\| \Delta \right\|, \nonumber \\
    \therefore\,\, & \left\| \widehat{W}_T(\widehat{\theta}^{(\text{os})}_T) \right\| = \left\| \widehat{\Omega}_T(\widehat{\theta}^{(\text{os})}_T)^{-1} \right\| = \frac{1}{\lambda_{\min}\left( \widehat{\Omega}_T(\widehat{\theta}^{(\text{os})}_T) \right)} \leq \frac{1}{\lambda_{\min}\left( \Omega_T(\widehat{\theta}^{(\text{os})}_T) \right) - \left\| \Delta \right\|}. \label{eq:apdx-min-eigenvalue-w-hat}
\end{align}
By applying Lemma~\ref{lemma:apdx-uniform-law-dependent-data} to each term of the matrix and using the union bound, we have
\begin{align}
    \P\left( \sup_{\theta \in \Theta} \left\| \widehat{\Omega}_T(\theta) - \Omega_T(\theta) \right\| \leq \delta \right) &\overset{(a)}{\geq} \P\left( \sup_{\theta \in \Theta} \left\| \widehat{\Omega}_T(\theta) - \Omega_T(\theta) \right\|_{F} \leq \delta \right) \nonumber \\
        &\geq \P\left( \sup_{\theta \in \Theta} \sum_{i,j} \left| \widehat{\Omega}_{T,i,j}(\theta) - \Omega_{T, i,j}(\theta) \right| \leq \delta \right) \nonumber \\
        &\geq 1 - \sum_{i,j} \P\left( \sup_{\theta \in \Theta} \left| \widehat{\Omega}_{T,i,j}(\theta) - \Omega_{T, i,j}(\theta) \right| > \frac{\delta}{M^2} \right) \nonumber \\
        &= 1 - \frac{1}{\delta^{D}} \exp\left\{ - \OM\left( T \delta^2 \right) \right\} \nonumber \\
    \therefore\,\, \P\left( \left\| \Delta \right\| \leq \delta \right) = \P\left( \left\| \widehat{\Omega}_T(\Tilde{\theta}_T) - \Omega(\Tilde{\theta}_T) \right\| \leq \delta \right) &\geq 1 - \frac{1}{\delta^{D}} \exp\left\{ - \OM\left( T \delta^2 \right) \right\}, \label{eq:apdx-delta-matrix-small}
\end{align}
where in (a) $\|.\|_F$ denotes the Frobenius norm.

For some $\delta_0 > 0$, let $ \Bar{\lambda} = \inf_{\theta \in N_{\delta_0}(\theta^*), \kappa \in \ChoiceSimplex} \lambda_{\min}\left( \Omega_T(\theta) \right)$. For $\delta \leq \min\left\{ \delta_0, \frac{\Bar{\lambda}}{2} \right\}$, we have
\begin{align*}
    \left\| \Delta \right\| \leq \delta & \overset{(a)}{\implies} \left\| \widehat{W}_T(\Tilde{\theta}_T) \right\| \leq \frac{2}{\Bar{\lambda}}, \\
    \therefore \,\, \P\left( \left\| \widehat{W}_T(\Tilde{\theta}_T) \right\| \leq \frac{2}{\Bar{\lambda}} \right) &\geq \P\left( \left\| \Delta \right\| \leq \delta \right) \\
        &\overset{(b)}{\geq} 1 - \frac{1}{\delta^{D}} \exp\left\{ - \OM\left( T \delta^2 \right) \right\},
\end{align*}
where (a) follows by Eq.~\ref{eq:apdx-min-eigenvalue-w-hat} and (b) by Eq.~\ref{eq:apdx-delta-matrix-small}.
\end{proof}

In the next lemma, we present a concentration inequality for $\widehat{k}_T$ with better rates under additional restrictions on $\theta^*$. We do not require these better rates for proving zero regret for OMS-ETG. We present this lemma for the sake of completeness.
\begin{lemma}[Another concentration inequality for $\widehat{k}_T$]\label{lemma:apdx-k-hat-concentration-faster}
Let $\kappa(\theta) = \arg\min_{\kappa} V(\theta, \kappa)$, $\, \Theta_{\text{boundary}} = \{ \theta \in \Theta : \kappa(\theta) \in \text{boundary}\left( \ChoiceSimplex \right) \}$, where $\text{boundary}\left( \ChoiceSimplex \right) = \left\{ \kappa \in \ChoiceSimplex : \exists i, \,\, \text{s.t.} \,\, \kappa_i = 0 \right\}$, 
$\, \Theta_{\text{minima}} = \{ \theta \in \Theta : \frac{\partial V}{\partial \kappa}(\theta, \kappa(\theta)) = 0 \}$, and
$\Theta_{\text{restricted}} = \Theta \setminus \left( \Theta_{\text{boundary}} \bigcap \Theta_{\text{minima}} \right)$ 
Suppose that (i) the conditions of Lemma~\ref{lemma:gmm-conc-inequality} hold, and
(ii) $\theta \in \Theta_{\text{restricted}}$. 
Then
\begin{align*}
    \P\left( \left\| \widehat{k}_T - \kappa^* \right\| > \delta \right) < \frac{1}{\delta^{2D}} \exp\left\{ - \OM \left( T \delta^4 \right) \right\}.
\end{align*}
This means that if $\theta^*$ is \textit{not} such that the minimizer $\kappa(\theta) = \arg\min_{\kappa} V(\theta, \kappa)$ is on the boundary of the simplex and is also a local minimum of $V(\theta, \kappa)$ (informally, $\kappa(\theta)$ is \textit{not} ``just'' on the boundary), we can get better rates.
\end{lemma}
\begin{proof}
Now we use the tail bound for $\widehat{\theta}_T$ to derive a concentration inequality for $\widehat{k}_T$ when $\theta \in \Theta_{\text{restricted}}$.
$\widehat{k}_T$ is the solution to the following constrained optimization problem:
\begin{align*}
    \min_{\kappa \in \R^{\CardC} } V\left(\widehat{\theta}_T, \kappa \right) \,\, \text{subject to} \,\, \sum_{i=1}^{\CardC} \kappa_i = 1.
\end{align*}
The Lagrangian function is
\begin{align*}
    \mathcal{L}\left( \theta, \kappa, \lambda \right) = V\left( \theta, \kappa \right) + \lambda \left( \sum_{i=1}^{\CardC} \kappa_i - 1 \right).
\end{align*}
Let $f(\theta, \kappa, \lambda) = \frac{\partial \mathcal{L}}{\partial \kappa}\left( \theta, \kappa, \lambda \right) = \frac{\partial V}{\partial \kappa}(\theta, \kappa) + \lambda [1, 1, \hdots, 1]^\top$.
Since $\lambda [1, 1, \hdots, 1]^\top \neq 0$, there exists a Lagrange multiplier $\lambda^{*} \in \R$ such that $f(\theta^*, \kappa^*, \lambda^*) = 0$.

Condition~(ii) is required to ensure that $f(\theta, \kappa, \lambda^*)$ is continuously differentiable in $(\theta, \kappa)$
which allows us to use the implicit function theorem.
To show this, we divide the space $\Theta_{\text{restricted}}$ into two disjoint sets:
(i) $\Theta_{\text{interior}} = \Theta \setminus \Theta_{\text{boundary}}$, and
(ii) $\Theta_{\text{strict-boundary}} = \Theta_{\text{boundary}} \bigcap \Theta^{c}_{\text{minima}}$.
When $\theta \in \Theta_{\text{interior}}$, the constraint will \textit{not} be active and thus $\lambda^* = 0$. When $\theta \in \Theta_{\text{strict-boundary}}$, the constraint will be active and thus $\lambda^* > 0$. In both cases, $f(\theta, \kappa, \lambda^*)$ will be continuously differentiable in $(\theta, \kappa)$.
Note that if $\theta \in \Theta \setminus \Theta_{\text{restricted}}$, then $\lambda^* = 0$ but $f$ is not differentiable because the constraint is ``just'' inactive.

Let $Y(\theta, \kappa) = \frac{\partial f}{\partial \kappa}(\theta, \kappa) = \frac{\partial^2 V}{\partial \kappa^2}(\theta, \kappa)$, and $X(\theta, \kappa) = \frac{\partial f}{\partial \theta}(\theta, \kappa) = \frac{\partial^2 V}{\partial \theta \partial \kappa}(\theta, \kappa)$.
By the implicit function theorem, since $Y(\theta^*, \kappa^*)$ is invertible (by Condition~(v)), there exist neighbourhoods $N(\theta^*)$ and $N(\kappa^*)$ and a function $\phi : N(\theta^*) \rightarrow N(\kappa^*)$ such that $\widehat{k}_T = \phi(\widehat{\theta}_T)$ and $\frac{\partial \phi}{\partial \theta}(\theta) = - \left[ Y(\theta, \phi(\theta))^{-1} X(\theta, \phi(\theta)) \right]$. By a Taylor expansion, we get
\begin{align*}
    \widehat{k}_T = \phi(\widehat{\theta}_T) &\overset{(a)}{=} \phi(\theta^*) + \frac{\partial \phi}{\partial \theta}(\Tilde{\theta}) \left( \widehat{\theta}_T - \theta^* \right) \\
        &= \kappa^* + \frac{\partial \phi}{\partial \theta}(\Tilde{\theta}) \left( \widehat{\theta}_T - \theta^* \right) \\
    \therefore\,\, \left\| \widehat{k}_T - \kappa^* \right\| &\leq \left\| \frac{\partial \phi}{\partial \theta}(\Tilde{\theta}) \right\| \left\| \widehat{\theta}_T - \theta^* \right\| \\
        &\leq C \left\| \widehat{\theta}_T - \theta^* \right\|,
\end{align*}
where in (a) $\Tilde{\theta}$ is a point on the line segment joining $\widehat{\theta}_T$ and $\theta^*$, and $C = \sup_{\theta \in \mathcal{N}(\theta^*)} \left\| \frac{\partial \phi}{\partial \theta}(\theta) \right\|$.
Therefore, we have
\begin{align*}
    \P \left( \left\| \widehat{\kappa}_T - \kappa^* \right\| \leq \delta \right) &\geq \P \left( \left\| \widehat{\theta}_T - \theta^* \right\| \leq \frac{\delta}{C} \right)  \geq 1 - \frac{1}{\delta^{2D}} \exp\left\{ - \OM\left( t \delta^4 \right) \right\}.
\end{align*}
\end{proof}

\begin{figure}
\centering
\includegraphics[width=0.6\columnwidth]{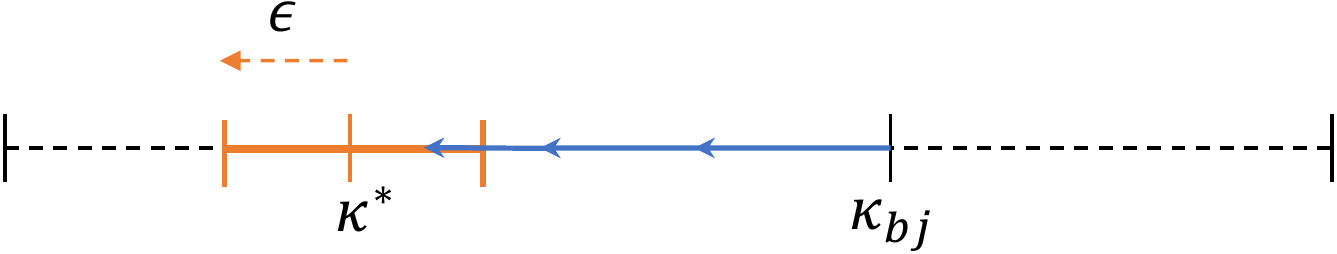}
\caption{Illustration of the proof of OMS-ETG algorithm. When the event $\mathcal{I}(\epsilon)$ occurs, (a) if the selection ratio $\kappa_{bj}$ is outside $N_\epsilon(\kappa^*)$, then then selection ratio in the next round $\kappa_{b(j+1)}$ will move closer to $N_{\epsilon}(\kappa^*)$, and (b) if $\kappa_{bj}$ is inside $N_\epsilon(\kappa^*)$, it remains inside for all future rounds.}
\label{fig:apdx-oms-etg-proof-diagram}
\end{figure}

\subsection{Proof of Theorem~\ref{thm:regret-etg} (Regret of OMS-ETG)}\label{sec:apdx-oms-etg-proof}

\begin{theorem*}[Regret of OMS-ETG]
Suppose that Conditions (i)-(iv) of Proposition~\ref{prop:asymptotic-normality} hold.
Let $\Tilde{\Delta}(s) = \{ s \kappa_b + (1-s)\kappa : \kappa \in \ChoiceSimplex \}$.
Case (a): For a fixed $s \in (0, 1)$, if the oracle selection ratio $\kappa^* \in \Tilde{\Delta}(s)$, then the regret converges to zero: $R_\infty(\pi_{\text{ETG}}) = 0$.
If $\kappa^* \notin \Tilde{\Delta}(s)$, then $R_\infty(\pi_{\text{ETG}}) = r$ for some constant r > 0.
Case (b): Now also suppose that the conditions for Lemma~\ref{lemma:gmm-conc-inequality} hold.
If $s = C T^{\eta - 1}$ for some constant $C$ and any $\eta \in [0, 1)$, then $\forall \theta^* \in \Theta$, the regret converges to zero: $R_\infty(\pi_{\text{ETG}}) = 0$.
\end{theorem*}
\begin{proof}
We prove this theorem by first showing that $\kappa_T \overset{p}{\rightarrow} \kappa^*$. Then we can apply Proposition~\ref{prop:asymptotic-normality} to get the desired result. Recall that $b = Ts$ is the batch size.

\textbf{Case 1: when $s \in (0, 1)$ is a fixed constant and $\kappa^* \in \Tilde{\Delta}(s)$.}

Let $\mathcal{I}(\epsilon)$ be the event that $\widehat{k}_{bj}$ remains inside an $\epsilon$-ball of $\kappa^*$ (denoted by $N_{\epsilon}(\kappa^*)$) for all rounds $j \in [J]$.
That is, $\mathcal{I}(\epsilon) = \left\{ \forall j \in [J], \,\, \widehat{k}_{bj} \in N_{\epsilon}(\kappa^*) \right\}$.
If $\kappa^* \in \Tilde{\Delta}(s)$, then to prove that $\kappa_T \ConvProb \kappa^*$, it is sufficient to show that $\forall \epsilon > 0, \,\, \mathcal{I}(\epsilon)$ occurs w.p.a. $1$.

This is because in OMS-ETG, after every round, we move as close to $\widehat{k}_{bj}$ as possible.
This is illustrated in Figure~\ref{fig:apdx-oms-etg-proof-diagram} for the case when $\ChoiceSimplex$ is a $1$-simplex.
When $\mathcal{I}(\epsilon)$ occurs, if the selection ratio $\kappa_{bj}$ after round $j$ is outside $N_\epsilon(\kappa^*)$, we move towards it in the subsequent round and thus $\kappa_{b(j+1)}$ will be closer to $N_\epsilon(\kappa^*)$.
Once the selection ratio enters $N_\epsilon(\kappa^*)$ (which it is guaranteed to if $\kappa^* \in \Tilde{\Delta}(s)$), it will remain inside $N_\epsilon(\kappa^*)$ for every round after that.
Thus $\mathcal{I}(\epsilon) \implies \kappa_T \in N_{\epsilon}(\kappa^*)$. Therefore, we have
\begin{align*}
    \forall \epsilon > 0, \,\, \P(\kappa_T \in N_{\epsilon}(\kappa^*)) &\geq \P(\mathcal{I}(\epsilon)) \\
        &= \P\left( \forall j \in [J], \,\, \widehat{k}_{bj} \in N_{\epsilon}(\kappa^*) \right) \\
        &= 1 - \P\left( \bigcup_{j=1}^{J} \,\, \left\| \widehat{k}_{bj} - \kappa^* \right\| > \epsilon \right) \\
        &\overset{(a)}{\geq} 1 - \sum_{j=1}^{J} \P\left( \left\| \widehat{k}_{bj} - \kappa^* \right\| > \epsilon \right) \\
        &\xrightarrow[]{(b)} 1, \\
    \therefore \,\, \kappa_T &\ConvProb \kappa^*,
\end{align*}
where (a) follows by the union bound and (b) follows because $J$ is finite and $\forall j,\,\, \widehat{k}_{bj} \ConvProb \kappa^*$ (by Lemma~\ref{lemma:apdx-kappa-hat-consistency}).

\textbf{Case 2: when $s$ depends on the horizon $T$.}

\textbf{Case 2(a): when $s \in \Omega(T^{\eta - 1})$ for any $\eta \in (0, 1)$.}

Similar to Case~1, it is sufficient to show that the event
$\mathcal{I}(\epsilon) = \left\{\forall j \in [J], \,\, \widehat{k}_{bj} \in N_{\epsilon}(\kappa^*) \right\}$
occurs w.p.a. $1$ for every $\epsilon > 0$.
However, since $J \to \infty$, consistency of $\widehat{k}_{bj}$ is no longer sufficient to prove this. Instead, we use the concentration inequality in Lemma~\ref{lemma:gmm-conc-inequality}:
\begin{align*}
    \forall \epsilon > 0, \,\, \P(\kappa_T \in N_{\epsilon}(\kappa^*)) &\geq \P(\mathcal{I}(\epsilon)) \\
    &= \P\left( \forall j \in [J], \,\, \widehat{k}_{bj} \in N_{\epsilon}(\kappa^*) \right) \\
    &= \P\left( \forall j \in [J], \,\, \left\| \widehat{k}_{bj} - \kappa^* \right\| \leq \epsilon \right) \\
    &= 1 - \P\left(\bigcup_{j=1}^{J} \,\, \left\| \widehat{k}_{bj} - \kappa^* \right\| > \epsilon \right) \\
    &\overset{(a)}{\geq} 1 - \sum_{j=1}^{J} \P\left( \left\| \widehat{k}_{bj} - \kappa^* \right\| > \epsilon \right) \\
    &\overset{(b)}{\geq} 1 - \sum_{j=1}^{J} \frac{1}{\epsilon^{4D}} \exp\left\{ - \OM\left( -Tsj \epsilon^8 \right) \right\} \\
    &\overset{(c)}{\geq} 1 - \sum_{j=1}^{J} \frac{1}{\epsilon^{4D}} \exp\left\{ - \OM\left( -Ts \epsilon^8 \right) \right\} \\
    &= 1 - \frac{J}{\epsilon^{4D}} \exp\left\{ - \OM\left( -Ts \epsilon^8 \right) \right\} \\
    &= 1 - \frac{1}{s\epsilon^{4D}} \exp\left\{ - \OM\left( -Ts \epsilon^8 \right) \right\} \\
    &\rightarrow 1 \,\, \text{if} \,\, s = C T^{\eta - 1}
\end{align*}
for any $\eta \in (0, 1)$ and some constant $C$. Here (a) follows by the union bound, (b) by Lemma~\ref{lemma:gmm-conc-inequality}, and (c) because $j \geq 1$.

\textbf{Case~2(b): when $s = \frac{C}{T}$ for some constant $C > 0$.}

We prove this similarly to Case~2(a). However, in this case, the number of rounds $J = \frac{1}{s} \in \OM(T)$.
Let $f = T^{\gamma - 1}$ for some $\gamma \in (0, 1)$ and
$\mathcal{I}(f, \epsilon) = \left\{ \forall j \in [Jf + 1, \hdots, J], \,\, \widehat{k}_{bj} \in N_{\epsilon}(\kappa^*) \right\}$ be the event that $\widehat{k}_{bj}$ remains inside $N_\epsilon(\kappa^*)$ \textit{after} the first $Jf$ rounds. 

Since $f \in o(1)$, we have $\mathcal{I}(f, \epsilon) \implies \kappa_T \in N_{\epsilon}(\kappa^*)$ for every $\epsilon > 0$. This is because the fraction $f$ is asymptotically negligible and thus we can effectively ignore the first $Jf$ rounds. Therefore we have
\begin{align*}
    \forall \epsilon > 0, \,\, \P(\kappa_T \in N_{\epsilon}(\kappa^*)) &\geq \P(\mathcal{I}(f, \epsilon)) \\
        &= \P\left( \forall j \in [Jf + 1, Jf + 2, \hdots, J], \,\, \left\| \widehat{k}_{bj} - \kappa^* \right\| \leq \epsilon \right) \\
        &= 1 - \P\left( \bigcup_{j=Jf+1}^{J} \,\, \left\| \widehat{k}_{bj} - \kappa^* \right\| > \epsilon \right) \\
        &\geq 1 - \sum_{j=Jf}^{J} \P\left( \left\| \widehat{k}_{bj} - \kappa^* \right\| > \epsilon \right) \\
        &\geq 1 - \sum_{j=Jf+1}^{J} \frac{1}{\epsilon^{4D}} \exp\left\{ -\OM\left( Tsj \epsilon^8 \right) \right\} \\
        &\overset{(a)}{\geq} 1 - \sum_{j=Jf+1}^{J} \frac{1}{\epsilon^{4D}} \exp\left\{ -\OM\left( j \epsilon^8 \right) \right\} \\
        &\overset{(b)}{\geq} 1 - \sum_{j=Jf+1}^{J} \frac{1}{\epsilon^{4D}} \exp\left\{ -\OM\left( Jf \epsilon^8 \right) \right\} \\
        &\overset{(c)}{\geq} 1 - \sum_{j=Jf+1}^{J} \frac{1}{\epsilon^{4D}} \exp\left\{ -\OM\left( T^\gamma \epsilon^8 \right) \right\} \\
        &\geq 1 - \frac{J}{\epsilon^{4D}} \exp\left\{ -\OM\left( T^\gamma \epsilon^8 \right) \right\} \\
        &\geq 1 - \frac{T}{\epsilon^{4D}} \exp\left\{ -\OM\left( T^\gamma \epsilon^8 \right) \right\} \\
        &\rightarrow 1,
\end{align*}
where (a) follows because $Ts = C$, (b) because $j \geq Jf$, and (c) because $Jf = \OM(T^\gamma)$.
We note that it is possible to unify the analysis of Case~2(a) and Case~2(b) by ignoring the first $Jf$ rounds in Case~2(a) as well. We prove the two cases separately for the sake of clarity.

\end{proof}

\section{Incorporate Cost Structure}\label{sec:appendix-cost-structure}

\begin{figure}[t]
\centering
\begin{subfigure}[b]{0.48\textwidth}
{
    \setlength{\interspacetitleruled}{0pt}%
    \setlength{\algotitleheightrule}{0pt}%
    \begin{algorithm}[H]
    \SetAlgoLined
    \KwInput{$B$, $s$, $c$}
    $\widehat{k} = \ctrSim$ \;
    $b = \frac{Bs}{c_{\max}}$ \;
    $B_l = B$\;
    $j = 0$ \;
    \While{$B_l > 0$}{
        $j \leftarrow j + 1$\;
        $\text{last\_step} = \frac{B_l}{c_{\max}} \leq b$\;
        \eIf{not last\_step}{
            Collect $b$ samples s.t. $\kappa_{bj} = \widehat{k}$\;
            $B_l \leftarrow B - bj\left( \kappa^\top_{bj} c \right)$\;
            $t = b(j+1)$\;
            $\widehat{\theta}_{t} = \text{GMM}(H_t, \widehat{W} = \widehat{W}_{\text{valid}})$\;
            $\widehat{k}_{t} = \arg\min_{\kappa \in \ChoiceSimplex} V(\widehat{\theta}_{t}, \kappa) \left( \kappa^\top c \right)$\;
            $\widehat{k} = \text{proj}(\widehat{k}_{\text{min}}, \Tilde{\Delta}_{j+1}(\kappa_t))$\;
        }{
            Collect samples s.t. $\kappa_T = \widehat{k}$\;
            $B_l \leftarrow 0$\;
        }
    }
    $\widehat{\theta}_T = \text{GMM}(H_T, \widehat{W} = \widehat{W}_{\text{efficient}})$\;
    \KwOutput{$\widehat{\theta}_T$}
    \end{algorithm}
}
\caption{OMS-ETG-FS (fixed samples per batch).}
\label{fig:policy-algorithm-etg-fs}
\end{subfigure}
\hfill
\begin{subfigure}[b]{0.48\textwidth}
{
    \setlength{\interspacetitleruled}{0pt}%
    \setlength{\algotitleheightrule}{0pt}%
    \begin{algorithm}[H]
    \SetAlgoLined
    \KwInput{$B, s, c$}
    $\widehat{k} = \ctrSim$\;
    $J = \frac{B}{s}$\;
    $t = 0$\;
    \For{$j \in [1, 2, \hdots, J]$}{
        $b = \frac{Bs}{\left(\widehat{k}^\top c \right)}$\;
        $t \leftarrow t + b$\;
        Collect $b$ samples s.t. $\kappa_{t} = \widehat{k}$\;
        $\widehat{\theta}_{t} = \text{GMM}(H_t, \widehat{W} = \widehat{W}_{\text{valid}})$\;
        $\widehat{k}_{t} = \arg\min_{\kappa \in \ChoiceSimplex} V(\widehat{\theta}_{t}, \kappa) \left( \kappa^\top c \right)$\;
        $\widehat{k} = \text{proj}(\widehat{k}_{\text{min}}, \Tilde{\Delta}_{j+1}(\kappa_t))$\;
    }
    $\widehat{\theta}_T = \text{GMM}(H_T, \widehat{W} = \widehat{W}_{\text{efficient}})$\;
    \KwOutput{$\widehat{\theta}_T$}
    \end{algorithm}
}
\caption{OMS-ETG-FB (fixed budget per batch)}
\label{fig:policy-algorithm-etg-fb}
\end{subfigure}
\caption{Algorithms for OMS-ETG-FS and OMS-ETG-FB.}
\end{figure}

\subsection{Proof of Proposition~\ref{prop:regret-etc-cost-structure} (Regret of OMS-ETC-CS)}

\begin{proposition*}[Regret of OMS-ETC-CS]
Suppose that the conditions of Theorem~\ref{thm:etc-regret} hold. 
If $e = o(1)$ such that $Be \rightarrow \infty$ as $B \to \infty$, then
$\forall \theta^* \in \Theta, \, R_{\infty}(\pi_{\text{ETC-CS}}) = 0$.
\end{proposition*}
\begin{proof}
The proof is almost exactly like that of Theorem~\ref{thm:etc-regret}. We prove that $\kappa_T \ConvProb \kappa^*$ and then apply Proposition~\ref{prop:asymptotic-normality}.
Let the number of samples used for exploration be $T_e$. Since $\kappa_{T_e} = \left[ \frac{1}{\CardC}, \frac{1}{\CardC}, \hdots, \frac{1}{\CardC} \right]$, we have
\begin{align*}
    T_e = \frac{Be}{\kappa_{T_e}^\top c}.
\end{align*}
$T_e$ is not a random variable because $\kappa_{T_e}$ is fixed.
By Lemma~\ref{lemma:apdx-kappa-hat-consistency}, we have $\widehat{k}_{T_e} \ConvProb \kappa^*$.

When $e \in o(1)$, the feasible region converges to the entire simplex, i.e., $\Tilde{\Delta} \to \ChoiceSimplex$. Thus $\kappa_T - \widehat{k}_{T_e} \ConvProb 0$.
\end{proof}

\subsection{Proof of Proposition~\ref{prop:regret-etg-fs} (Regret of OMS-ETG-FS)}

\begin{proposition*}[Regret of OMS-ETG-FS]
Suppose that the conditions of Theorem~\ref{thm:regret-etg} hold.
If $s = B^{\eta - 1}$ and any $\eta \in \left[0, 1\right)$, then
$\forall \theta^* \in \Theta, \, R_{\infty}\left(\pi_{\text{ETG-FS}}\right) = 0$.
\end{proposition*}
\begin{proof}
We can prove this similarly to Theorem~\ref{thm:regret-etg}. The key difference is that the number of rounds $J$ is now a random variable.
But we can use the fact the $J$ is bounded:
\begin{align*}
    \frac{1}{s} \leq J \leq \frac{ c_{\max}}{s c_{\min}}, \\
    \therefore \,\, J \in \OM\left( \frac{1}{s} \right).
\end{align*}
Now we can proceed like Case~2 in the proof of Theorem~\ref{thm:regret-etg}.
\end{proof}

\subsection{Proof of Proposition~\ref{prop:regret-etg-fb} (Regret of OMS-ETG-FB)}

\begin{proposition*}[Regret of OMS-ETG-FB]
Suppose that the conditions of Theorem~\ref{thm:regret-etg} hold.
If $s = B^{\eta - 1}$ and any $\eta \in \left[0, 1\right)$,
then
$\forall \theta^* \in \Theta, \, R_{\infty}\left(\pi_{\text{ETG-FB}}\right) = 0$.
\end{proposition*}
\begin{proof}
We show this similarly to Theorem~\ref{thm:regret-etg}.
In this case, the size of each batch is random but the numbers of rounds $J = \frac{1}{s}$ is not random.
Thus we can't use the concentration inequality in Lemma~\ref{lemma:gmm-conc-inequality} directly since that only holds for a fixed time step $t$.
We get around this by showing that the estimated selection ratio $\widehat{k}_t$ will remain
in an $\epsilon$-ball around $\kappa^*$
uniformly over all time steps after some asymptotically negligible fraction of the horizon $T$.

Let $T_j$ be the number of samples collected after round $j$, i.e., $T_j = \frac{Bsj}{\kappa_{T_j}^\top c}$.
Let $f = B^{\gamma - 1}$ for some $\gamma \in (0, 1)$. Like the proof of Theorem~\ref{thm:regret-etg}, we can ignore the first $Jf$ rounds since they are $f \in o(1)$ is an asymptotically negligible fraction.
And similarly to the proof of Theorem~\ref{thm:regret-etg}, in order to show that $\kappa_T \ConvProb \kappa^*$, it is sufficient to show that $\P\left( \forall j \in [Jf + 1, Jf + 2, \hdots, J], \,\, \left\| \widehat{k}_{T_j} - \kappa^* \right\| \leq \epsilon \right) \xrightarrow[]{B \to \infty} 1$.
We can show this as follows:
\begin{align}
    \P\left( \forall j \in [Jf + 1, Jf + 2, \hdots, J], \,\, \left\| \widehat{k}_{T_j} - \kappa^* \right\| \leq \epsilon \right) &\geq \P\left( \forall t \in [T_{Jf + 1}, \hdots, T_J], \,\, \left\| \widehat{k}_{t} - \kappa^* \right\| \leq \epsilon \right). \label{eq:apdx-etg-fb-intermediate}
\end{align}
The minimum and maximum batch sizes are $b_{\min} = \frac{Bs}{c_{\max}}$ and $b_{\max} = \frac{Bs}{c_{\min}}$, respectively. Therefore,
\begin{align*}
    T_{Jf + 1} &\geq Jf b_{\min} = Jf \frac{Bs}{c_{\max}}, \\
    T_{J} &\leq J b_{\max} = J \frac{Bs}{c_{\min}}.
\end{align*}
Using these facts and continuing Eq.~\ref{eq:apdx-etg-fb-intermediate}, we get:
\begin{align*}
    \P\left( \forall j \in [Jf + 1, Jf + 2, \hdots, J], \,\, \left\| \widehat{k}_{T_j} - \kappa^* \right\| \leq \epsilon \right) &\geq \P\left( \forall t \in [T_{Jf + 1}, \hdots, T_J], \,\, \left\| \widehat{k}_{t} - \kappa^* \right\| \leq \epsilon \right) \\
        &\geq \P\left( \forall t \in [Jf b_{\min}, \hdots, J b_{\max}], \,\, \left\| \widehat{k}_{t} - \kappa^* \right\| \leq \epsilon \right) \\
        &\overset{(a)}{\geq} 1 - \sum_{t=Jf b_{\min}}^{J b_{\max}} \frac{1}{\epsilon^{4D}} \exp\left\{ -\OM\left( t \epsilon^8 \right) \right\} \\
        &\overset{(b)}{\geq} 1 - \sum_{t=Jf b_{\min}}^{J b_{\max}} \frac{1}{\epsilon^{4D}} \exp\left\{ -\OM\left( Jf b_{\min} \epsilon^8 \right) \right\} \\
        &\geq 1 - \sum_{t=Jf b_{\min}}^{J b_{\max}} \frac{1}{\epsilon^{4D}} \exp\left\{ -\OM\left( Bf \epsilon^8 \right) \right\} \\
        &\geq 1 - \frac{J b_{\max}}{\epsilon^{4D}} \exp\left\{ -\OM\left( Bf \epsilon^8 \right) \right\} \\
        &\geq 1 - \frac{B}{\epsilon^{4D}} \exp\left\{ -\OM\left( Bf \epsilon^8 \right) \right\} \\
        &\geq 1 - \frac{B}{\epsilon^{4D}} \exp\left\{ -\OM\left( B^\gamma \epsilon^8 \right) \right\} \\
        &\rightarrow 1,
\end{align*}
where (a) follows by the union bound and (b) because $t \geq Jf b_{\min}$.
\end{proof}

\section{Feasible regions}\label{sec:appendix-feasibility-region}

In this section, we derive the feasibility regions for the various policies. 

\textbf{OMS-ETC}

Recall that in OMS-ETC, we first collect $Te$ samples such that $\kappa_{Te} = \ctrSim$. For the remaining $T(1-e)$ samples, the agent can query the data sources with any fraction $\kappa \in \ChoiceSimplex$. Therefore, the feasible values of $\kappa_T$ are
\begin{align*}
    \Tilde{\Delta} &= \left\{ \frac{Te \kappa_{Te} + T(1-e) \kappa}{T} : \kappa \in \ChoiceSimplex \right\} \\
        &= \left\{ e \kappa_{Te} + (1-e) \kappa : \kappa \in \ChoiceSimplex \right\}.
\end{align*}

\textbf{OMS-ETG}

After $j$ rounds, the selection ratio is denoted by $\kappa_{bj}$. In every round, we collect $b = Ts$ samples. For the batch collected in round $j+1$, the agent can query the data sources with any fraction $\kappa \in \ChoiceSimplex$. Therefore, the feasible values of $\kappa_{b(j+1)}$ are
\begin{align*}
    \Tilde{\Delta}_{j+1}(\kappa_{bj}) &= \left\{ \frac{bj \kappa_{bj} + b \kappa}{b (j+1)} : \kappa \in \ChoiceSimplex \right\} \\
        &= \left\{ \frac{Tsj \kappa_{bj} + Ts \kappa}{Ts (j+1)} : \kappa \in \ChoiceSimplex \right\} \\
        &= \left\{ \frac{j \kappa_{bj} + \kappa}{ (j+1)} : \kappa \in \ChoiceSimplex \right\}.
\end{align*}

\textbf{OMS-ETC-CS}

The agent uses $Be$ budget to uniformly query the available data sources. Let $T_e$ denote the number of samples collected after exploration. We have
\begin{align*}
    T_e = \frac{Be}{\kappa^\top_{T_e} c},
\end{align*}
where $\kappa^\top_{T_e} = \ctrSim$ and $c$ is the cost vector. With the remaining $B(1-e)$ budget, the agent can collect samples with any fraction $\kappa \in \ChoiceSimplex$. However, since the data sources can have different costs, the total number of samples $T$ depends on the choice of $\kappa$:
\begin{align*}
    T = T_e + \frac{B(1-e)}{\kappa^\top c},
\end{align*}
for $\kappa \in \ChoiceSimplex$. Therefore the feasible values of $\kappa_T$ are
\begin{align*}
    \Tilde{\Delta} &= \left\{ \frac{T_e \kappa_{T_e} + (T-T_e) \kappa}{T} : \kappa \in \ChoiceSimplex \right\} \\
    &= \left\{ \frac{\frac{Be}{\kappa^\top_{T_e} c} \kappa_{T_e} + \frac{B(1-e)}{\kappa^\top c} \kappa}{\frac{Be}{\kappa^\top_{T_e} c} + \frac{B(1-e)}{\kappa^\top c}} : \kappa \in \ChoiceSimplex \right\} \\
    &= \left\{ \frac{e \left( \kappa^\top c \right) \kappa_{T_e} + (1 - e) \left( \kappa^\top_{T_e} c \right) \kappa}{e \left( \kappa^\top c \right) + (1 - e) \left( \kappa^\top_{T_e} c \right)} : \kappa \in \ChoiceSimplex \right\}.
\end{align*}

\textbf{OMS-ETG-FS}

Since we collect a fixed number of samples in each round, the feasibility region for OMS-ETG-FS is that same as OMS-ETG:
\begin{align*}
    \Tilde{\Delta}_{j+1}(\kappa_{bj}) &= \left\{ \frac{j \kappa_{bj} + \kappa}{ (j+1)} : \kappa \in \ChoiceSimplex \right\}.
\end{align*}

\textbf{OMS-ETG-FB}

Let the selection ratio after $j$ rounds be $\kappa_{T_j}$ where $T_j$ number of samples collected after round $j$: $T_j = \frac{Bsj}{\kappa^\top_{bj} c}$. For the batch collected in round $j+1$, the agent can query the data sources with any fraction $\kappa \in \ChoiceSimplex$. However, the number of samples collected in round $j+1$ would depend on the choice $\kappa$ due to the cost structure. Therefore the number of samples collected after round $j+1$ is
\begin{align*}
    T_{j+1} = T_j + \frac{Bsj}{\kappa^\top c},
\end{align*}
for $\kappa \in \ChoiceSimplex$. Hence, the feasible values of $\kappa_{T_{j+1}}$ are
\begin{align*}
    \Tilde{\Delta}_{j+1}(\kappa_{T_j}) &= \left\{ \frac{T_j \kappa_{T_j} +  (T_{j+1} - T_{j}) \kappa}{T_{j+1}} : \kappa \in \ChoiceSimplex \right\} \\
        &= \left\{ \frac{\frac{Bsj}{\kappa^\top_{bj} c} \kappa_{T_j} +  \frac{Bsj}{\kappa^\top c} \kappa}{\frac{Bsj}{\kappa^\top_{bj} c} + \frac{Bsj}{\kappa^\top c}} : \kappa \in \ChoiceSimplex \right\} \\
        &= \left\{ \frac{j \left( \kappa^\top c \right) \kappa_{T_j} + \left( \kappa^\top_{T_j} c \right) \kappa}{j \left( \kappa^\top c \right) + \left( \kappa^\top_{T_j} c \right)} : \kappa \in \ChoiceSimplex \right\}.
\end{align*}

\section{Experiments}\label{sec:appendix-experiments}

\subsection{Linear IV graph}

Data from the linear IV graph (Figure~\ref{fig:disjoint-iv-graph}) is simulated as follows:
\begin{align*}
    Z &\sim \mathcal{N}\left(0, \sigma^2_z \right), \\
    U &\sim \mathcal{N}\left(0, \sigma^2_u \right), \\
    X &:= \alpha Z + \gamma U + \epsilon_x, \,\, \epsilon_x \sim \mathcal{N}\left(0, \sigma^2_x \right), \\
    Y &:= \beta X + \phi U + \epsilon_y, \,\, \epsilon_y \sim \mathcal{N}\left(0, \sigma^2_y \right),
\end{align*}
where $\epsilon_x$ and $\epsilon_y$ are exogenous noise terms independent of other variables and each other and $U$ is an unobserved confounder.
Here $\{ \beta, \alpha, \gamma, \phi, \sigma^2_z, \sigma^2_u, \sigma^2_x, \sigma^2_y \}$ are parameters that we set for simulating the data. 
For the experiment in Section~\ref{sec:experiments-synthetic-data}, we used $\beta=1, \alpha=1, \gamma=1, \phi=1, \sigma_z = 1, \sigma_u = 1, \sigma_x = 1, \sigma_y = 1$.

The moment conditions used for estimation are
\begin{align*}
    g_t(\theta) = \underbrace{\begin{bmatrix} 
        s_{t, 1} \\
        s_{t, 2}
    \end{bmatrix}}_{=m(s_t)} \otimes \underbrace{\begin{bmatrix}
        Z_t (X_t - \alpha Z_t) \\
        Z_t (Y_t - \alpha \beta Z_t)
    \end{bmatrix}}_{=\Tilde{g}_t(\theta)}.
\end{align*}
The parameter we estimate is $\theta = [\beta, \alpha]^\top$ and $\beta = f_{\text{tar}}(\theta) = \theta_{0}$.

\subsection{Two IVs graph}

Data from the two IVs graph (Figure~\ref{fig:multiple-iv-graph}) is simulated as follows:
\begin{align*}
    Z_1 &\sim \mathcal{N}\left(0, \sigma^2_{z_1} \right), \\
    Z_2 &\sim \mathcal{N}\left(0, \sigma^2_{z_2} \right), \\
    U &\sim \mathcal{N}\left(0, \sigma^2_u \right), \\
    X &:= \alpha_1 Z_1 + \alpha_2 Z_2 + \gamma U + \epsilon_x, \,\, \epsilon_x \sim \mathcal{N}\left(0, \sigma^2_x \right), \\
    Y &:= \beta X + \phi U + \epsilon_y, \,\, \epsilon_y \sim \mathcal{N}\left(0, \sigma^2_y \right),
\end{align*}
where $\epsilon_x$ and $\epsilon_y$ are exogenous noise terms independent of other variables and each other and $U$ is an unobserved confounder.
For the experiment in Section~\ref{sec:experiments-synthetic-data}, we used $\beta=1, \alpha=1, \gamma=1, \phi=1, \sigma_z = 1, \sigma_u = 1, \sigma_x = 1, \sigma_y = 1$.

The moment conditions used for estimation are
\begin{align*}
    g_t(\theta) = \underbrace{\begin{bmatrix} 
        s_{t, 1} \\
        s_{t, 2}
    \end{bmatrix}}_{=m(s_t)} \otimes \underbrace{\begin{bmatrix}
        (Z_1)_t (Y_t - \beta X_t) \\
        (Z_2)_t (Y_t - \beta X_t)
    \end{bmatrix}}_{=\Tilde{g}_t(\theta)}.
\end{align*}
The parameter we estimate is $\theta = [\beta]$ and $\beta = f_{\text{tar}}(\theta) = \theta_0$.

\subsection{Confounder-mediator graph}
Data from the confounder-mediator graph (Figure~\ref{fig:multiple-iv-graph}) is simulated as follows:
\begin{align*}
    W &\sim \mathcal{N}\left(0, \sigma^2_{w} \right), \\
    X &:= d W + \epsilon_x, \,\, \epsilon_x \sim \mathcal{N}\left(0, \sigma^2_x \right), \\
    M &:= \frac{\beta}{a} X + \epsilon_m, \,\, \epsilon_m \sim \mathcal{N}\left(0, \sigma^2_m \right), \\
    Y &:= a M + b W + \epsilon_y, \,\, \epsilon_y \sim \mathcal{N}\left(0, \sigma^2_y \right),
\end{align*}
where $\epsilon_x, \epsilon_m$, and $\epsilon_y$ are exogenous noise terms independent of other variables and each other.
For the experiment in Section~\ref{sec:experiments-synthetic-data}, we used $\beta=-0.32, a=0.33, b=-0.34, d=0.45, \sigma_w = 1, \sigma_x = 1, \sigma_m = 1, \sigma_y = 1$.

The moment conditions used for estimation are
\begin{align*}
    g_t(\theta) = \underbrace{\begin{bmatrix} 
        s_{t, 1} \\
        s_{t, 1} \\
        s_{t, 2} \\
        s_{t, 2} \\
        s_{t, 2} \\
        s_{t, 1} \\
        s_{t, 1} \\
        s_{t, 1} \\
        1
    \end{bmatrix}}_{=m(s_t)} \otimes \underbrace{\begin{bmatrix}
        X_t (Y_t - b W_t - \beta X_t) \\
        W_t (Y_t - b W_t - \beta X_t) \\
        X_t (M_t - \frac{\beta}{a} X_t) \\
        M_t \left(Y_t - a M_t - \frac{b d \sigma^2_w}{d^2 \sigma^2_w + \sigma^2_x} X_t \right) \\
        X_t \left(Y_t - a M_t - \frac{b d \sigma^2_w}{d^2 \sigma^2_w + \sigma^2_x} X_t\right) \\
        W_t^2 - \sigma^2_w \\
        W_t (X_t - d W) \\
        X_t^2 - (d^2 \sigma^2_w + \sigma^2_x)
    \end{bmatrix}}_{=\Tilde{g}_t(\theta)}.
\end{align*}
The parameter we estimate is $\theta = [\beta, a, b, d, \sigma^2_w, \sigma^2_x]^\top$ and $\beta = f_{\text{tar}}(\theta) = \theta_0$.

\subsection{IHDP dataset}

To generate semi-synthetic IHDP dataset, we use two covariates: birth weight (denoted by $W_1$) and whether the mother smoked (denoted by $W_2$).
The binary treatment is denoted by $X$ and the outcome is denoted by $Y$.
The corresponding causal graph is shown in Figure~\ref{fig:ihdp-graph}.
For every sample of the semi-synthetic dataset, $W_1, W_2$, and $X$ are sampled uniformly at random from the real data. The outcome $Y$ is simulated as follows:
\begin{align*}
    Y := \beta X + \alpha_1 W_1 + \alpha_2 W_2 + \epsilon_y, \,\, \epsilon_y \sim \mathcal{N}\left(0, \sigma^2_y \right),
\end{align*}
where $\epsilon_y$ is an independent exogenous noise term.
For the experiment in Section~\ref{sec:experiments-semi-synthetic-data}, we used $\beta=1,\alpha_1=1, \alpha_2=0.1, \sigma_y=1$.

The moment conditions used for estimation are
\begin{align*}
    g_t(\theta) = \underbrace{\begin{bmatrix} 
        1-s_{t, 2} \\
        1-s_{t, 2} \\
        1-s_{t, 1} \\
        1-s_{t, 1} \\
        s_{t, 3} \\
        1-s_{t, 2} \\
        1-s_{t, 1} \\
        1-s_{t, 2} \\
        1-s_{t, 1} \\
        1-s_{t, 2} \\
        1-s_{t, 1}
    \end{bmatrix}}_{=m(s_t)} \otimes \underbrace{\begin{bmatrix}
        (W_1)_t \left( \left( Y_t - \alpha_1 (W_1)_t - \beta X_t \right) - \alpha_2 d \right) \\
        X_t \left( \left( Y_t - \alpha_1 (W_1)_t - \beta X_t \right) - \alpha_2 \tau_2 \right) \\
        (W_2)_t \left( \left( Y_t - \alpha_2 (W_2)_t - \beta X_t \right) - \alpha_1 d \right) \\
        X_t \left( \left( Y_t - \alpha_2 (W_2)_t - \beta X_t \right) - \alpha_1 \tau_1 \right) \\
        (W_1)_t (W_2)_t - d \\
        X (W_1)_t - \tau_1 \\
        X (W_2)_t - \tau_2 \\
        (W_1)^2_t - \sigma^2_{w_1} \\
        (W_2)^2_t - \sigma^2_{w_2} \\
        \left( Y_t - \alpha_1 (W_1)_t - \beta X \right)^2 - \alpha_2^2 \sigma^2_{w_2} - \sigma^2_y \\
        \left( Y_t - \alpha_2 (W_2)_t - \beta X \right)^2 - \alpha_1^2 \sigma^2_{w_1} - \sigma^2_y \\
    \end{bmatrix}}_{=\Tilde{g}_t(\theta)}.
\end{align*}
The parameter we estimate is $\theta = [\beta, \alpha_1, \alpha_2, d, \tau_1, \tau_2, \sigma^2_w, \sigma^2_y]^\top$ and $\beta = f_{\text{tar}}(\theta) = \theta_0$.

\subsection{The Vietnam draft and future earnings dataset}\label{sec:apdx-vietnam-earnings-experiment}

The causal graph for this dataset corresponds to Figure~\ref{fig:disjoint-iv-graph} with a binary IV $Z$, binary treatment $X$ and continuous outcome $Y$.
In this dataset, $\{Z, X\}$ and $\{Z, Y\}$ are collected from different data sources and thus $\{ Z, X, Y \}$ are not observed simultaneously. For our experiment, we only use data from the $1951$ cohort.

In the semi-synthetic dataset, we sample $Z$ uniformly at random from the real dataset. The treatment $X$ is generated similarly to a probit model.
We first generate an intermediate variable $X^{*}$ and then use that to generate $X$ as follows:
\begin{align*}
    X^{*} &:= \alpha Z + c^{*} + \epsilon_x, \,\, \epsilon_x \sim \mathcal{N}(0, 1), \\
    X &:= \mathbf{1}(X^{*} > 0),
\end{align*}
where $\mathbf{1}$ is the indicator function.
To reduce clutter, let $\mu_z = \widehat{\P}(Z = 1) = 0.3425$, $\mu^{(1)}_x = \P(X=1|Z=1)$ and $\mu^{(0)}_x = \P(X=1|Z=0)$.
We set the parameters $\alpha$ and $c^{*}$ such that $\mu^{(1)}_x = 0.2831$ and $\mu^{(0)}_x = 0.1468$ (these values have been taken from \citep[Table~2]{angrist1990lifetime} to match the empirical distribution):
\begin{align*}
    \mu^{(0)}_x &= \P( \mathbf{1}(X^{*} > 0) |Z=0) \\
        &= \P( c^{*} + \epsilon_x > 0) ) \\
        &= \P(  \epsilon_x > -c^{*}) ) \\
        &= \P(  \epsilon_x < c^{*}) ) \\
        &= \Phi(c^{*}), \\
    \therefore \, c^{*} &= \Phi^{-1}(\mu^{(0)}_x) \\
        &= \Phi^{-1}(0.1468) \\
        &= -1.050, \\
    \mu^{(1)}_x &= \P( \mathbf{1}(X^{*} > 0) |Z=1) \\
        &= \P( \alpha + c^{*} + \epsilon_x > 0) ) \\
        &= \Phi(\alpha + c^{*}) \\
    \therefore \, \alpha &= \Phi^{-1}(\mu^{(1)}_x) - c^{*} \\
        &= \Phi^{-1}(\mu^{(1)}_x) - \Phi^{-1}(\mu^{(0)}_x) \\
        &= \Phi^{-1}(0.2831) - \Phi^{-1}(0.1468) \\
        &= 0.4766,
\end{align*}
where $\Phi$ is the cumulative distribution function of the standard normal distribution.

In the real data, we standardize the outcome $Y$ by subtracting its mean and dividing by its standard deviation and thus $\widehat{\E}[Y] = 0$ and $\widehat{\Var}(Y) = 1$.
To generate the simulated outcome $Y$, we use $Y := \beta X + \gamma + c_0 \epsilon_{x} + \epsilon_y$, where $\epsilon_y \sim \mathcal{N}(0, \sigma^2_{\epsilon_y})$. When $c_0 \neq 0$, the noise term $(c_0 \epsilon_{x} + \epsilon_y) \notindep X$. Thus $c_0$ determines the extent of the confounding between $X$ and $Y$.

We now describe how we set $\beta$ and $\gamma$. Since $E[Y] = 0$, we have
\begin{align*}
    \gamma &= -\beta \E[X] \\
        &= -\beta \left( \mu^{(0)}_x (1-\mu_z) +  \mu^{(1)}_x \mu_z \right) \\
        &= -0.1934 \beta.
\end{align*}
Using the covariance of $Y$ and $Z$, we have
\begin{align*}
    \Cov(Y, Z) &= \E[Y Z] \\
        &= \beta \E[Z X] + \gamma \E[Z] \\
        &= \beta \left( \E[Z X] - \E[X] \E[Z] \right) \\
        &= \beta \left( \E[Z \mathbf{1}(\alpha Z + c^{*} + \epsilon_x > 0)] - \E[X] \E[Z] \right) \\
        &= \beta \left( \E[ Z \E[ \mathbf{1}(\alpha Z + c^{*} + \epsilon_x > 0)|Z]] - \E[X] \E[Z] \right) \\
        &= \beta \left( \E[ Z \E[ \mathbf{1}( \epsilon_x > -(\alpha Z + c^{*}))|Z]] - \E[X] \E[Z] \right) \\
        &= \beta \left( \E[ Z \Phi(\alpha Z + c^{*})] - \E[X] \E[Z] \right) \\
        &= \beta \left(  \Phi(\alpha Z + c^{*}) \mu_z - \E[X] \E[Z] \right) \\
        &= \beta \mu_z \left(  \mu^{(1)}_x - \E[X] \right).
\end{align*}
Therefore, we set $\beta$ and $\gamma$ as
\begin{align*}
    \beta &= \frac{ \widehat{E}[YZ] }{ \mu_z \left( \mu^{(1)}_x  - \E[X] \right) } = -0.4313, \\
    \gamma &= -0.1934 \beta = 0.0834.
\end{align*}
Now we describe how we set $c_0$ and $\sigma^2_{\epsilon_y}$. For this, we use the variance of $Y$:
\begin{align}\label{eq:apdx-angrist-var-y-expression}
    \Var(Y) = 1 = \beta^2 \Var(X) + c_0^2 \sigma^2_{\epsilon_y} + 2 \beta c_0 \E[X \epsilon_x].
\end{align}
We have 
\begin{align*}
    \Var(X) &= \Var[\E(X|Z)] + \E[\Var(X|Z)] \\
        &= \Var\left(Z \mu^{(1)}_x + (1-Z)\mu^{(0)}_x \right) + \mu_z \mu^{(1)}_x (1 - \mu^{(1)}_x) + (1 - \mu_z) \mu^{(0)}_x (1 - \mu^{(0)}_x) \\
        &= \mu_z (1-\mu_z) (\mu^{(1)}_x - \mu^{(0)}_x)^2 + \mu_z \mu^{(1)}_x (1 - \mu^{(1)}_x) + (1 - \mu_z) \mu^{(0)}_x (1 - \mu^{(0)}_x) \\
        &= 0.1560,\\
    \E[X \epsilon_x] &= \E[\E[ \mathbf{1}(Z \alpha + c^* + \epsilon_x > 0) \epsilon_x | Z ]] \\
        &= \E[\E[ \mathbf{1}( \epsilon_x > -(Z \alpha + c^*)) \epsilon_x | Z ]] \\
        &\overset{(a)}{=} \E_Z\left[ \int_{-(Z \alpha + c^*)}^{\infty} x f(x) dx \right] \\
        &= \E\left[ \frac{1}{\sqrt{2 \pi}} \exp\left\{ \frac{-(Z \alpha + c^*)^2}{2} \right\} \right] \\
        &= \frac{1}{\sqrt{2 \pi}} \left[ \exp\left\{ \frac{-(c^*)^2}{2} (1-\mu_z) + \exp\left\{ \frac{-(\alpha + c^*)^2}{2} \right\} \mu_z \right\} \right] \\
        &= 0.2670,
\end{align*}
where in (a), $f(x)$ is the probability density function of the standard normal distribution.
We set $c_0 = 0.5$ and using Eq.~\ref{eq:apdx-angrist-var-y-expression}, we get $\sigma^2_{\epsilon_y} = 0.6058$.

To summarize, the data is generated as follows:
\begin{align*}
    Z &\sim \text{Bernoulli}(\mu_z), \\
    X^* &:= \alpha Z + c^* + \epsilon_x, \, \epsilon_x \sim \mathcal{N}(0, 1), \\
    X &:= \mathbf{1}(X^* > 0), \\
    Y &:= \beta X + \gamma + c_0 \epsilon_x + \epsilon_y, \, \epsilon_y \sim \mathcal{N}(0, \sigma^2_{\epsilon_y}),
\end{align*}
where $\mu_z = 0.3424, \alpha = 0.4766, c^* = -1.0502, \beta = -0.4313, \gamma = 0.0834$, and
$\sigma^2_{\epsilon_y} = 0.6058$.

The moment conditions used for estimation are
\begin{align*}
    g_t(\theta) = \underbrace{\begin{bmatrix} 
        s_{t, 1} \\
        s_{t, 1} \\
        s_{t, 2} \\
        s_{t, 2}
    \end{bmatrix}}_{=m(s_t)} \otimes \underbrace{\begin{bmatrix}
        Z_t (Y_t - \mu_1) \\
        (1 - Z_t) (Y_t - \mu_0) \\
        Z_t (X_t - \tau_1) \\
        (1 - Z_t) (X_t - \tau_0)
    \end{bmatrix}}_{=\Tilde{g}_t(\theta)}.
\end{align*}
The parameter we estimate is $\theta = [\mu_1, \mu_0, \tau_1, \tau_0]$ and the target parameter is $\beta = f_{\text{tar}}(\theta) = \frac{\mu_1 - \mu_0}{\tau_1 - \tau_0}$. 

\end{document}